%% file: main_arxiv.tex
\documentclass{article} % For LaTeX2e
\usepackage{times}

\input{math_commands.tex}

\usepackage[preprint]{arxiv}

\usepackage[citecolor=blue,colorlinks=true]{hyperref}       % hyperlinks
\usepackage{url}
\usepackage{booktabs}       % professional-quality tables
\usepackage{nicefrac}       % compact symbols for 1/2, etc.
\usepackage{amsmath,amsfonts,amssymb,multirow,amsthm}
\usepackage{xspace}
\usepackage{bm}
\usepackage{xcolor}
\usepackage{enumitem}
\usepackage{graphicx}
\usepackage{subcaption}

\theoremstyle{definition}

\makeatletter
\newcommand{\newreptheorem}[2]{\newtheorem*{rep@#1}{\rep@title}\newenvironment{rep#1}[1]{\def\rep@title{#2 \ref*{##1}}\begin{rep@#1}}{\end{rep@#1}}}
\makeatother

\theoremstyle{plain}
\newtheorem{theorem}{Theorem}[section]

\newtheorem{lemma}[theorem]{Lemma}
\newreptheorem{theorem}{Theorem}
\newreptheorem{lemma}{Lemma}

\theoremstyle{remark}
\newtheorem*{remark}{Remark}

\newcommand{\simiid}{\overset{iid}{\sim}}
\newcommand{\RNum}[1]{\uppercase\expandafter{\romannumeral #1\relax}}

\makeatletter
\DeclareRobustCommand\onedot{\futurelet\@let@token\@onedot}
\def\@onedot{\ifx\@let@token.\else.\null\fi\xspace}

\def\eg{\emph{e.g}\onedot} 
\def\ie{\emph{i.e}\onedot} 
\def\cf{\emph{c.f}\onedot} 
\def\st{\emph{s.t}\onedot} 
\def\aka{a.k.a\onedot}

\def\wrt{w.r.t\onedot}

\makeatother

\title{A PAC-Bayesian Approach to Generalization Bounds for Graph Neural Networks}

\author{Renjie Liao$^{1,2,3}$, Raquel Urtasun$^{1,2,3,4}$, Richard  Zemel$^{1,3,4}$ \\
University of Toronto$^1$, Uber ATG$^2$, Vector Institute$^3$, Canadian Institute for Advanced Research$^4$ \\
\texttt{\{rjliao, urtasun, zemel\}@cs.toronto.edu}
}

\begin{document}

\maketitle

\begin{abstract}
In this paper, we derive generalization bounds for the two primary classes of graph neural networks (GNNs), namely graph convolutional networks (GCNs) and message passing GNNs (MPGNNs), via a PAC-Bayesian approach. 
% RZ Specifically, 
Our result reveals that the maximum node degree and spectral norm of the weights govern the generalization bounds of both models.
%Importantly, 
We also show that our bound for GCNs is a natural generalization of the results developed in \citep{neyshabur2017pac} for fully-connected and convolutional neural networks.
For message passing GNNs, our PAC-Bayes bound improves over the 
Rademacher complexity based bound in \citep{garg2020generalization}, 
showing a tighter dependency on the maximum node degree and the maximum hidden dimension. 
The key ingredients of our proofs are a perturbation analysis of GNNs and the generalization of PAC-Bayes analysis to non-homogeneous GNNs. 
We perform an empirical study on several real-world graph datasets and verify that our PAC-Bayes bound is tighter than others. 
% \raquel{this is a bit missplaced here, and not clear what you are trying to say}
% \raquel{say something of what you observe in this experiments.}
% \renjie{After the experiments, we may say our overall bound is sharper empirically.}
\end{abstract}

\input{section/intro}

\input{section/background}

\input{section/main_results}

\input{section/experiments}

\input{section/conclusions}


\subsubsection*{Acknowledgments}
We would like to thank Vikas K. Garg and Stefanie Jegelka for the valuable feedback and discussion.

\clearpage
\bibliography{reference}
\bibliographystyle{iclr2021_conference}

\clearpage
\appendix
\input{section/appendix}

\end{document}

%% file: math_commands.tex
\usepackage{amsmath,amsfonts,bm}

\def\eqref#1{equation~\ref{#1}}
\def\1{\bm{1}}

\DeclareMathAlphabet{\mathsfit}{\encodingdefault}{\sfdefault}{m}{sl}
\SetMathAlphabet{\mathsfit}{bold}{\encodingdefault}{\sfdefault}{bx}{n}

\newcommand{\KL}{D_{\mathrm{KL}}}

\DeclareMathOperator*{\argmax}{arg\,max}

%% file: section/intro.tex
\section{Introduction}

Graph neural networks (GNNs)
% \RZ{Need to change citation style, from Garg (2017) to (Garg, 2017)}
\citep{gori2005new,scarselli2008graph,bronstein2017geometric,battaglia2018relational} have become very popular
%widespread in the past few years 
recently due to their ability to learn powerful representations from graph-structured data, and have achieved state-of-the-art results in a variety of application domains such as social networks \citep{hamilton2017inductive}, quantum chemistry \citep{gilmer2017neural}, computer vision \citep{monti2017geometric}, reinforcement learning \citep{sanchez2018graph}, robotics \citep{casas2019spatially}, and physics \citep{henrion2017neural}. 
Given a graph along with node/edge features, GNNs learn  node/edge representations by propagating information on the graph via local computations shared across the nodes/edges. 
% \RZ{The next couple sentences can be shifted down below, when the two classes are brought up again.} 
Based on the specific form of local computation employed, GNNs can be divided into two categories: graph convolution based GNNs \citep{bruna2013spectral,duvenaud2015convolutional,kipf2016semi} and message passing based GNNs \citep{li2015gated,dai2016discriminative,gilmer2017neural}.
% The former generalizes the convolution operator from regular graphs (\eg, grids) to ones with arbitrary topology, whereas the latter mimics message passing algorithms (\eg, belief propagation) and parameterizes the shared functions (\eg, node state update functions) via neural networks. 
The former generalizes the convolution operator from regular graphs (\eg, grids) to ones with arbitrary topology, whereas the latter mimics message passing algorithms and parameterizes the shared functions via neural networks. 

% \raquel{maybe talk about some of them? a bit odd otherwise}
% \raquel{maybe specify what the shared functions are in this case}

Due to the tremendous empirical success of GNNs, there is increasing interest in understanding their theoretical properties.
For example, some recent works study their expressiveness \citep{maron2018invariant,xu2018powerful,chen2019equivalence}, that is,  what class of functions can be represented by GNNs. 
However, only few works investigate why GNNs generalize so well to unseen graphs.
They are either
%existing literature is either 
restricted to a specific model variant \citep{verma2019stability,du2019graph,garg2020generalization} or have loose dependencies on graph statistics \citep{scarselli2018vapnik}. 
% \raquel{too loose to be useful is a very strong sentence that might piss off people. Smooth it}

% \renjie{Say existing theoretical work on GNNs and discuss their limitations, \eg, not covering all main streams variants of GNN.} 
% \raquel{yes, add this here expaining the limitations of current results.}

On the other hand, GNNs have close ties to standard feedforward neural networks, \eg, multi-layer perceptrons (MLPs) and convolutional neural networks (CNNs).
% \RZ are closely relevant to GNNs.
In particular, if each i.i.d. sample is viewed as a node, then the whole dataset becomes a graph without edges. 
Therefore, GNNs can be seen as generalizations of MLPs/CNNs since they model not only the regularities within a
sample but also the dependencies among samples as defined in the graph. 
%the additional dependencies among samples in the form of graphs. 
It is therefore natural to ask if we can generalize the recent advancements on generalization bounds for MLPs/CNNs \citep{harvey2017nearly,neyshabur2017pac,bartlett2017spectrally,dziugaite2017computing,arora2018stronger,arora2019fine} to GNNs, and how would graph structures affect the generalization bounds? 

% \raquel{what is the feature? this is odd}
% \raquel{why do you say that is the structure between samples? this confuses me}
% \raquel{odd to call it graph data}
% \raquel{rewrite the last 2 sentneces, english is not great}

%As a consequence MLPs and CNNs are special cases of GNNs-- GNNs belong to a broader class of models that can model the dependencies (in the form of graph structure) between samples.
%Given the above relationship and the recent theoretical advances in the study of the generalization ability of MLPs and CNNs \cite{harvey2017nearly,neyshabur2017pac,bartlett2017spectrally,dziugaite2017computing,arora2018stronger,arora2019fine}, 

In this paper, we answer the above questions by proving generalization bounds for the two primary 
%representative 
classes of GNNs, \ie, graph convolutional networks (GCNs) \citep{kipf2016semi} and message-passing GNNs  (MPGNNs) 
% \RZ{should use this abbreviation throughout} 
\citep{dai2016discriminative,jin2018learning}. 

% \raquel{you wrote before as if it was the full family, but just an instance now. This is a contradiction}

% \renjie{Talk about margin-based in the following part and how it generalizes to other situations. Summarize previous work and simplify the following two paragraphs.}

Our generalization bound for GCNs shows an intimate relationship with the bounds for MLPs/CNNs with ReLU activations \citep{neyshabur2017pac,bartlett2017spectrally}.
% \RZ{I'd replace "ReLU-activated MLPs/CNNs" with "ReLU-activated MLPs/CNNs with ReLU activations".}
% \raquel{dependency? what do you mean by statistic here?} 
In particular, they share the same term, \ie, the product of the spectral norms of the learned weights at each layer multiplied by a factor that is additive across layers.
% In contrast to the plain vector-valued data for MLPs/CNNs, GCNs handle the more flexible and challenging graph-structured data 
% and the graph-structured data is degenerated to the plain vector valued data
The bound for GCNs has an additional multiplicative factor $d^{{(l-1)}/{2}}$ where $d - 1$ is the maximum node degree and $l$ is the network depth.
Since MLPs/CNNs are special GNNs operating on graphs without edges (\ie, $d - 1 = 0$), the bound for GCNs coincides with the ones for MLPs/CNNs with ReLU activations \citep{neyshabur2017pac} on such degenerated graphs.
Therefore, our result is a natural generalization of the existing results for MLPs/CNNs.

Our generalization bound for message passing GNNs reveals that the governing terms of the bound are similar to the ones of GCNs, \ie, the geometric series of the learned weights and the multiplicative factor $d^{l-1}$.
The geometric series appears due to the weight sharing across message passing steps, thus corresponding to the product term across layers in GCNs.
The term $d^{l-1}$ encodes the key graph statistics. 
% \raquel{it is not clear what you mean here by "graph structure"}.
Our bound improves the dependency on the maximum node degree and the maximum hidden dimension compared to the recent Rademacher complexity based bound \citep{garg2020generalization}.
% \renjie{Say something about empirical support here!}
Moreover, we compute the bound values on four real-world graph datasets (\eg, social networks and protein structures) and verify that our bounds are tighter.
% Moreover, our PAC-Bayes analysis is more succinct compared to the complicated Rademacher complexity based one using the matrix covering number bound and Dudley’s entropy integral via chaining.

% \raquel{this paragraph is out of context as you have talked in previous sentences about results. I'll flip paragraphs}
In terms of the proof techniques, our analysis follows the PAC-Bayes framework in the seminal work of \citep{neyshabur2017pac} for MLPs/CNNs with ReLU activations.
However, we make two distinctive contributions which are customized for GNNs.
First, a naive adaptation of the perturbation analysis in \citep{neyshabur2017pac} does not work for GNNs since ReLU is not 1-Lipschitz under the spectral norm, \ie, $\Vert \text{ReLU}(X) \Vert_2 \le \Vert X \Vert_2$ does not hold for some real matrix $X$.
Instead, we construct the recursion on certain node representations of GNNs like the one with maximum $\ell_2$ norm, so that we can perform perturbation analysis with vector 2-norm.
% thus bypassing the challenge of deriving a tight Lipschitz constant for ReLU under spectral norm. 
% \raquel{explain why you want this, or what is good about this}
Second, in contrast to \citep{neyshabur2017pac} which only handles the homogeneous networks, \ie, $f(ax) = a f(x)$ when $a \ge 0$, we properly construct a quantity of the learned weights which 1) provides a way to satisfy the constraints of the previous perturbation analysis and 2) induces a finite covering on the range of the quantity so that the PAC-Bayes bound holds for all possible weights.
This generalizes the analysis to non-homogeneous GNNs like typical MPGNNs.

The rest of the paper is organized as follows.
% We discuss the related work in section \ref{sect:related_work}.
In Section \ref{sect:background}, we introduce background material necessary for our analysis.
We then present our generalization bounds and the comparison to existing results in Section \ref{sect:main_results}.
We also provide an empirical study to support our theoretical arguments in Section \ref{sect:exp}.
At last, we discuss the extensions, limitations and some open problems.

%% file: section/background.tex
\vspace{-0.1cm}
\section{Background}\label{sect:background}
\vspace{-0.1cm}

In this section, we first explain our analysis setup including notation and assumptions.
We then describe the two representative GNN models in detail.
Finally, we review the PAC-Bayes analysis.

\input{section/problem_setup}

\input{section/gnn_model}

\input{section/pac_bayes_background}

%% file: section/problem_setup.tex
\subsection{Analysis Setup}

In the following analysis, we consider the $K$-class graph classification problem which is common in the GNN literature, where given a graph sample $z$, we would like to classify it into one of the predefined $K$ classes.
We will discuss extensions to other problems like graph regression in Section \ref{sect:discussion}.
Each graph sample $z$ is a triplet of an adjacency matrix $A$, node features $X \in \mathbb{R}^{n \times h_0}$ and output label $y \in \mathbb{R}^{1 \times K}$, \ie $z = (A, X, y)$, where $n$ is the number of nodes and $h_0$ is the input feature dimension.
We start our discussion by defining our notations. 
% \raquel{I always find pretty cumbersome papers that start with so much notation, vs defining it when necessary. Its a  style difference though}
Let $\mathbb{N}_{k}^{+}$ be the first $k$ positive integers, \ie, $\mathbb{N}_{k}^{+} = \{1, 2, \dots, k\}$, 
$\vert \cdot \vert_p$  the vector $p$-norm
and $\Vert \cdot \Vert_p$  the operator norm induced by the vector $p$-norm.
Further, $\Vert \cdot \Vert_F$ denotes the Frobenius norm of a matrix, 
$e$  the base of the natural logarithm function $\log$,  
$A[i, j]$  the $(i, j)$-th element of matrix $A$ and $A[i, :]$  the $i$-th row.
We use parenthesis to avoid the ambiguity, \eg, $(AB)[i, j]$ means the $(i, j)$-th element of the product matrix $AB$.
We then introduce some terminologies from statistical learning theory and define the sample space as $\mathcal{Z}$, $z = (A, X, y) \in \mathcal{Z}$ where $X \in \mathcal{X}$ (node feature space) and $A \in \mathcal{G}$ (graph space), data distribution $\mathcal{D}$, $z \simiid \mathcal{D}$, hypothesis (or model) $f_w$ where $f_w \in \mathcal{H}$ (hypothesis class), and training set $S$ with size $m$, $S = \{z_1, \dots, z_m\}$.
% Loss $\ell$, $\ell: \mathcal{H} \times \mathcal{Z} \rightarrow \mathbb{N}_{k}^{+}$
We make the following assumptions which also appear in the literature: 
% \raquel{you need to say something about whether this assumptions are limited or standard or something}
\begin{enumerate}[label=A\arabic*]
    \item Data, \ie, triplets $(A, X, y)$, are i.i.d. samples drawn from some unknown distribution $\mathcal{D}$. \label{asm:iid_data}
    \item The maximum hidden dimension across all layers is $h$. \label{asm:max_node_feat_dim}
    \item Node feature of any graph is contained in a $\ell_2$-ball with radius $B$. Specifically, we have $\forall i \in \mathbb{N}^{+}_{n}$, the $i$-th node feature $X[i, :] \in \mathcal{X}_{B,h_0} = \{ x \in \mathbb{R}^{h_0} \vert \sum_{j=1}^{h_0} x_j^2 \le B^2 \}$. \label{asm:bound_input}
    \item We only consider simple graphs (\ie, undirected, no loops\footnote{Here loop means an edge that connects a vertex to itself, \aka, self-loop.}, and no multi-edges) with maximum node degree as $d-1$. 
    % (\ie, each node in the computation graphs of GNNs has a maximum node degree of $d$). 
    \label{asm:max_node_degree}
\end{enumerate}
Note that it is straightforward to estimate $B$ and $d$ empirically on real-world graph data.

%% file: section/gnn_model.tex
\subsection{Graph Neural Networks (GNNs)}
% \renjie{Add more motivations of choosing these two models, description, main ideas}

% \raquel{you should explain what this subsection is going to focuss on, e.g. its not so obvious why in this section you will talk about muticlass margin loss}
In this part, we describe the details of the GNN models and the loss function we used for the graph classification problem. 
The essential idea of GNNs is to propagate information over the graph so that the learned representations capture the dependencies among nodes/edges.
We now review two classes of GNNs, GCNs and MPGNNs, which have different mechanisms for propagating information.
We choose them since they are the most popular variants and represent two common types of neural networks, \ie, feedforward (GCNs) and recurrent (MPGNNs) neural networks. 
% \raquel{you mean that in a special case they are RNNs and feed forward?}
% In particular, GCNs are feedforward neural networks whereas the Message Passing GNNs are recurrent neural networks.
We discuss the extension of our analysis to other GNN variants in Section \ref{sect:discussion}.
For ease of notation, we define the model to be $f_w \in \mathcal{H}: \mathcal{X} \times \mathcal{G} \rightarrow \mathbb{R}^{K}$ where $w$ is the vectorization of all model  parameters. 

% \RZ{Should mention somewhere (maybe here) that the presentation here and the analysis focuses on the graph classification problem. This is a canonical problem but the analysis is not specific to it, as it can readily apply to node or link classification (I think?)}.

\paragraph{GCNs:}
Graph convolutional networks (GCNs) \citep{kipf2016semi} for the $K$-class graph classification problem can be  defined as follows,
\begin{align}
    % H_0 & = X \quad && (\text{Input Node Feature}) \nonumber \\
    H_{k} & = \sigma_{k} \left( \tilde{L} H_{k-1} W_{k} \right) \quad && (k \text{-th Graph Convolution Layer}) \nonumber \\
    H_{l} & = \frac{1}{n} \bm{1}_{n} H_{l-1} W_{l} \quad && (\text{Readout Layer}), \label{eq:gcn}
\end{align}
where $k \in \mathbb{N}^{+}_{l-1}$, $H_{k} \in \mathbb{R}^{n \times h_{k}}$ are the node representations/states, $\bm{1}_{n} \in \mathbb{R}^{1 \times n}$ is a all-one vector,  
$l$ is the number of layers.\footnote{We count the readout function as a layer to be consistent with the existing analysis of MLPs/CNNs.}
and $W_{j}$ is the weight matrix of the $j$-th layer.
The initial node state is the observed node feature $H_0 = X$.
For both GCNs and MPGNNs, we consider $l > 1$ since otherwise the model degenerates to a linear transformation which does not leverage the graph and is trivial to analyze. 
Due to assumption \ref{asm:max_node_feat_dim}, $W_{j}$ is of size at most $h \times h$, \ie, $h_k \le h$, $\forall k \in \mathbb{N}^{+}_{l-1}$.
The graph Laplacian $\tilde{L}$ is  defined as, $\tilde{A} = I + A$, $\tilde{L} = D^{-\frac{1}{2}} \tilde{A} D^{-\frac{1}{2}}$ where $D$ is the degree matrix of $\tilde{A}$.
Note that the maximum eigenvalue of $\tilde{L}$ is $1$ in this case.
We absorb the bias into the weight by appending constant $1$ to the node feature.
Typically, GCNs use ReLU as the non-linearity, \ie, 
$\sigma_i(x) = \max(0, x), \forall i=1, \cdots, l-1$. 
%$\sigma_1(x) = \cdots = \sigma_{l-1}(x) = \max(0, x)$.
We use the common mean-readout to obtain the graph representation where $H_{l-1} \in \mathbb{R}^{n \times h_{l-1}}$, $W_l \in \mathbb{R}^{h_{l-1} \times K}$, and $H_l \in \mathbb{R}^{1 \times K}$. 
% \raquel{this is confusing what you mean by average here}

\paragraph{MPGNNs:}

% \renjie{Briefly comment on how generalizable of analysis to other variants of message passing GNNs, extentions to other problems like node classifications, loss function.}

There are multiple variants of message passing GNNs, \eg, \citep{li2015gated,dai2016discriminative,gilmer2017neural}, which share the same algorithmic framework but instantiate a few components differently, \eg, the node state update function.
We choose the same class of models as in \citep{garg2020generalization} which are popular in the literature \citep{dai2016discriminative,jin2018learning} in order to fairly compare bounds.
This MPGNN model can be written in matrix forms as follows,
\begin{align}
    % H_{0} & = \bm{0}  \quad && (\text{Initial Node State}) \nonumber \\
    M_{k} & = g( C_{\operatorname{out}}^{\top} H_{k-1} ) \quad && (k \text{-th step Message Computation}) \nonumber \\
    \bar{M}_{k} & = C_{\operatorname{in}} M_{k} \quad && (k \text{-th step Message Aggregation}) \nonumber \\
    H_{k} & = \phi\left( X W_1 + \rho \left( \bar{M}_{k} \right) W_2 \right) \quad && (k \text{-th step Node State Update}) \nonumber \\
    H_{l} & = \frac{1}{n} \bm{1}_{n} H_{l-1} W_{l} \quad && (\text{Readout Layer}), \label{eq:mpgnn}
\end{align}
where $k \in \mathbb{N}^{+}_{l-1}$, $H_{k} \in \mathbb{R}^{n \times h_{k}}$ are node representations/states and $H_l \in \mathbb{R}^{1 \times K}$ is the output representation.
Here we initialize $H_{0} = \bm{0}$.
% Note that the node states $H_k$ have the same shape at each messaging passing step and the messages $M_k$ could have different 2nd-dimension compared to $H_k$.
W.l.o.g., we assume $\forall k \in \mathbb{N}^{+}_{l-1}$, $H_k \in \mathbb{R}^{n \times h}$ and $M_k \in \mathbb{R}^{n \times h}$ since $h$ is the maximum hidden dimension.
$C_{\operatorname{in}} \in \mathbb{R}^{n \times c}$ and $C_{\operatorname{out}} \in \mathbb{R}^{n \times c}$ ($c$ is the number of edges) are the incidence matrices corresponding to incoming and outgoing nodes\footnote{
For undirected graphs, we convert each edge into two directed edges. 
% Also, the maximum incoming node degree equals the maximum outgoing node degree of which the value is $d-1$ following the assumption \ref{asm:max_node_degree}.
} respectively.
Specifically, rows and columns of $C_{\operatorname{in}}$ and $C_{\operatorname{out}}$ correspond to nodes and edges respectively. 
$C_{\operatorname{in}}[i,j] = 1$ indicates that the incoming node of the $j$-th edge is the $i$-th node.
Similarly, $C_{\operatorname{out}}[i,j] = 1$ indicates that the outgoing node of the $j$-th edge is the $i$-th node.
$g, \phi, \rho$ are nonlinear mappings, \eg, ReLU and Tanh.
Technically speaking, $g: \mathbb{R}^{h} \rightarrow \mathbb{R}^{h}$, $\phi: \mathbb{R}^{h} \rightarrow \mathbb{R}^{h}$, and $\rho: \mathbb{R}^{h} \rightarrow \mathbb{R}^{h}$ operate on vector-states of individual node/edge. 
However, since we share these functions across nodes/edges, we can naturally generalize them to matrix-states, \eg, $\tilde{\phi}: \mathbb{R}^{n \times h} \rightarrow \mathbb{R}^{n \times h}$ where $\tilde{\phi}(X)[i, :] = \phi(X[i, :])$.
By doing so, the same function could be applied to matrices with varying size of the first dimension.
% \eg, varying number of nodes $n$.
% \raquel{odd sentence the varying size of first dimension} 
For simplicity, we use $g, \phi, \rho$ to denote such generalization to matrices.
We denote the Lipschitz constants of $g, \phi, \rho$ under the vector $2$-norm as $C_{g}, C_{\phi}, C_{\rho}$ respectively.
We also assume $g(\bm{0}) = \bm{0}$, $\phi(\bm{0}) = \bm{0}$, and $\rho(\bm{0}) = \bm{0}$ and define the \emph{percolation complexity} as $\mathcal{C} = C_{g} C_{\phi} C_{\rho} \Vert W_2 \Vert_2$ following \citep{garg2020generalization}. 
% denote $\mathcal{C} = C_{g} C_{\phi} C_{\rho}$ as the \emph{percolation complexity} following \cite{garg2020generalization}.

% We initialize the node states $H_0$ via the node feature for simplicity. 
% One can also initialize the node states via a subnetwork which takes the node feature as the input.
% One could again use more complicated readout function.
% We left the generalization of our analysis with these variants as the future work.

\paragraph{Multiclass Margin Loss:}

% \raquel{why is this model definition on the margin loss section? the number of layers should be explain at the beginning that you dont consider the linear case, or?}
We use the multi-class $\gamma$-margin loss following \citep{bartlett2017spectrally,neyshabur2017pac}. 
% \raquel{this is a very weak reason, can you motivate it better? or just talk about following somethign...}
% \raquel{Also, it is not that is ontop, but that this is the loss that you will analyze}
% \begin{align}
%     \ell_{\text{ramp}}(r, \gamma) & = 
%         \begin{cases}
%             0 & \qquad r < -\gamma, \\
%             1 + r/\gamma & \qquad r \in [-\gamma, 0], \\
%             1 & \qquad r > 0\\
%         \end{cases} \nonumber \\
%     \ell_{\gamma}(H_l, y) & = \ell_{\text{ramp}}(\max_{i \neq y} H_l[i] - H_l[y]), \gamma)  
% \end{align}
% where $\gamma > 0$.
The generalization error is defined as,
\begin{align}\label{eq:max_margin_loss}
    L_{\mathcal{D}, \gamma}(f_w) = \underset{z \sim \mathcal{D}}{\mathbb{P}} \left( f_w(X, A)[y] \le \gamma + \max_{j \neq y}f_w(X, A)[j] \right), 
\end{align}
where $\gamma > 0$ and $f_w(X, A)$ is the $l$-th layer representations, \ie, $H_l = f_w(X, A)$.  
% \raquel{is notation ${\mathbb{P}}$ common?}
% In our context $H_l = f_w(X, A)$.
Accordingly, we can define the empirical error as,
\begin{align}\label{eq:empirical_max_margin_loss}
    L_{S, \gamma}(f_w) = \frac{1}{m} \sum_{z_i \in S} \bm{1} \left( f_w(X, A)[y] \le \gamma + \max_{j \neq y}f_w(X, A)[j] \right). 
\end{align}
% We discuss the extension to other loss functions in Section \ref{sect:discussion}.

%% file: section/pac_bayes_background.tex
\subsection{Background of PAC-Bayes Analysis}

%We now introduce the background of PAC-Bayes analysis.
PAC-Bayes \citep{mcallester1999pac,mcallester2003simplified,langford2003pac} takes a Bayesian view of the probably approximately correct (PAC) learning theory \citep{valiant1984theory}.
In particular, it assumes that we have a prior distribution $P$ over the hypothesis class $\mathcal{H}$ and obtain a posterior distribution $Q$ over the same support through the learning process on the training set.
Therefore, instead of having a deterministic model/hypothesis as in common learning formulations, we have a distribution of models.
Under this Bayesian view, we define the generalization error and the empirical error as,
% \raquel{what do you mean in terms of this bayesian view? rephrase}
\begin{align}
    L_{S, \gamma}(Q) = \mathbb{E}_{w \sim Q}[ L_{S, \gamma}(f_w) ], \nonumber
    \qquad \qquad
    L_{\mathcal{D}, \gamma}(Q) = \mathbb{E}_{w \sim Q}[ L_{\mathcal{D}, \gamma}(f_w) ]. \nonumber
\end{align}

Since many interesting models like neural networks are deterministic and the exact form of the posterior $Q$ induced by the learning process and the prior $P$ is typically unknown, %hardly known,
it is unclear how one can perform PAC-Bayes analysis.
Fortunately, we can exploit the following result from the PAC-Bayes theory. 
\begin{theorem}\citep{mcallester2003simplified}\label{thm:pac_bayes}
    (Two-sided)
    Let $P$ be a prior distribution over $\mathcal{H}$ and let $\delta \in (0, 1)$.
    Then, with probability $1 - \delta$ over the choice of an i.i.d. size-$m$ training set $S$ according to $\mathcal{D}$, for all distributions $Q$ over $\mathcal{H}$ and any $\gamma > 0$, we have 
    \begin{align}
    	L_{\mathcal{D}, \gamma}(Q) \le L_{S, \gamma}(Q) + \sqrt{\frac{\KL(Q \Vert P) + \ln \frac{2m}{\delta} }{2(m-1)}}. \nonumber
    \end{align}
\end{theorem}
Here $\KL$ is the KL-divergence.
The nice thing about this result is that the inequality holds for all possible prior $P$ and posterior $Q$ distributions. 
Hence, we have the freedom to construct specific priors and posteriors so that we can work out the bound.
Moreover, \cite{mcallester2003simplified,neyshabur2017pac} provide a general recipe to construct the posterior such that for a large class of models, including %those 
deterministic ones, the PAC-Bayes bound  can be computed.
Taking a neural network as an example, we can choose a prior distribution with some known density, \eg, a fixed Gaussian, over the initial weights.
After the learning process, we can add random perturbations to the learned weights from another known distribution as long as the KL-divergence permits an analytical form.
This converts the deterministic model into a distribution of models while still obtaining a tractable KL divergence.
Leveraging Theorem \ref{thm:pac_bayes} and the above recipe, \cite{neyshabur2017pac} obtained the following result which holds for a large class of deterministic models.
\begin{lemma}\citep{neyshabur2017pac}\label{lemma:pac_bayes_deterministic}\footnote{The constants slightly differ from the original paper since we use a two-sided version of Theorem \ref{thm:pac_bayes}.}
    Let $f_w(x): \mathcal{X} \rightarrow \mathbb{R}^K $ be any model with parameters $w$, and let $P$ be any distribution on the parameters that is independent of the training data.
    For any $w$, we construct a posterior $Q(w + u)$ by adding any random perturbation $u$ to $w$, \st, $\mathbb{P}(\max_{x \in \mathcal{X}} \vert f_{w + u}(x) - f_w(x) \vert_{\infty} < \frac{\gamma}{4} ) > \frac{1}{2}$.
    Then, for any $\gamma, \delta > 0$, with probability at least $1 - \delta$ over an i.i.d. size-$m$ training set $S$ according to $\mathcal{D}$, for any $w$, we have:
    \begin{align}
        L_{\mathcal{D}, 0}(f_w) \le L_{S, \gamma}(f_w) + \sqrt{\frac{2\KL(Q(w + u) \Vert P) + \log \frac{8m}{\delta} }{2(m-1)}}. \nonumber
    \end{align}
\end{lemma}
This lemma guarantees that, as long as the change of the output brought by the perturbations is small with a large probability, one can obtain the corresponding generalization bound.

%% file: section/main_results.tex
\section{Generalization Bounds}\label{sect:main_results}
% \raquel{use a more informed title}

In this section, we present the main results: generalization bounds of GCNs and MPGNNs using a PAC-Bayesian approach.
We then relate them to existing generalization bounds of GNNs and draw connections to the bounds of MLPs/CNNs.
We summarize the key ideas of the proof in the main text and defer the details  to the appendix.

\input{section/gcn_results}

\input{section/mpgnn_results}

\input{section/comparison_bounds}

%% file: section/gcn_results.tex
\subsection{PAC-Bayes Bounds of GCNs}

As discussed above, %before, 
in order to apply Lemma \ref{lemma:pac_bayes_deterministic}, we must ensure that the change of the output brought by the weight perturbations  is small with a large probability.
In the following lemma, we bound this change using the product of the spectral norms of learned weights at each layer and a term depending on some statistics of the graph.
\begin{lemma}\label{lemma:gcn_perturbation}
    (GCN Perturbation Bound) For any $B > 0, l > 1$, let $f_w \in \mathcal{H}: \mathcal{X} \times \mathcal{G} \rightarrow \mathbb{R}^{K}$ be a $l$-layer GCN. Then for any $w$, and $x \in \mathcal{X}_{B,h_0}$, and any perturbation $u = \text{vec}( \{U_i\}_{i=1}^{l})$ such that $\forall i \in \mathbb{N}^{+}_{l}$, $\Vert U_i \Vert_{2} \le \frac{1}{l} \Vert W_i \Vert_{2}$, the change in the output of GCN is bounded as,
    \begin{align}
        \left\vert f_{w+u}(X, A) - f_w(X, A) \right\vert_{2} \le eB d^{\frac{l-1}{2}} \left( \prod_{i=1}^{l} \Vert W_{i} \Vert_{2} \right) \sum_{k=1}^{l} \frac{\Vert U_{k} \Vert_{2}}{\Vert W_{k} \Vert_{2}}. \nonumber
    \end{align}
\end{lemma}
The key idea of the proof is to decompose the change of the network output into two terms which depend on two quantities of GNNs respectively: the maximum change of node representations $\max_{i} \left\vert H_{l-1}^{\prime}[i, :] - H_{l-1}[i, :] \right\vert_{2}$ and the maximum node representation $\max_{i} \left\vert H_{l-1}[i, :] \right\vert_{2}$.
% and derive their bounds separately. 
% \raquel{is there some intuition that you can say about what these two terms are? as otherwise is a bit arbitrary here without seeing the proof}
Here the superscript prime denotes the perturbed model.
These two terms can be bounded by an induction on the layer index.
From this lemma, we can see that the most important graph statistic for the stability of GCNs is the maximum node degree, \ie, $d - 1$.
Armed with Lemma \ref{lemma:gcn_perturbation} and Lemma \ref{lemma:pac_bayes_deterministic}, we now present the PAC-Bayes generalization bound of GCNs as Theorem \ref{thm:gcn_generalization_bound}.

\begin{theorem}\label{thm:gcn_generalization_bound}
    (GCN Generalization Bound) For any $B > 0, l > 1$, let $f_w \in \mathcal{H}: \mathcal{X} \times \mathcal{G} \rightarrow \mathbb{R}^{K}$ be a $l$ layer GCN. Then for any $\delta, \gamma > 0$, with probability at least $1 - \delta$ over the choice of an i.i.d. size-$m$ training set $S$ according to $\mathcal{D}$, for any $w$, we have,
    \begin{align}
    	L_{\mathcal{D}, 0}(f_w) \le L_{S, \gamma}(f_w) + \mathcal{O} \left( \sqrt{\frac{ B^2 d^{l-1} l^2 h \log(lh) \prod\limits_{i=1}^{l} \Vert W_{i} \Vert_{2}^2 \sum\limits_{i=1}^{l} ({\Vert W_{i} \Vert_F^2}/{\Vert W_{i} \Vert_{2}^2}) + \log \frac{ml}{\delta} }{\gamma^2 m}} \right). \nonumber
    \end{align}
\end{theorem}

Since it is easy to show GCNs are homogeneous, the proof of Theorem \ref{thm:gcn_generalization_bound} follows the one for MLPs/CNNs with ReLU activations in \citep{neyshabur2017pac}.
In particular, we choose the prior distribution $P$ and the perturbation distribution to be zero-mean Gaussians with the same diagonal variance $\sigma$.
The key steps of the proof are: 1) constructing a quantity of learned weights $\beta = ( \prod_{i=1}^{l} \Vert W_i \Vert_{2} )^{1/l}$; 2) fixing any $\tilde{\beta}$, considering all $\beta$ that are in the range $\vert \beta - \tilde{\beta} \vert \le \beta / l$ and choosing $\sigma$ which depends on $\tilde{\beta}$ so that one can apply Lemma \ref{lemma:gcn_perturbation} and \ref{lemma:pac_bayes_deterministic} to obtain the PAC-Bayes bound;
3) taking a union bound of the result in the 2nd step by considering multiple choices of $\tilde{\beta}$ so that all possible values of $\beta$ (corresponding to all possible weight $w$) are covered.
Although Lemma \ref{lemma:pac_bayes_deterministic} and \ref{lemma:gcn_perturbation} have their own constraints on the random perturbation, above steps provide a way to set the variance $\sigma$ which satisfies these constraints and the independence \wrt learned weights. 
The latter is important since $\sigma$ is also the variance of the prior $P$ which should not depend on data.
% Note that in the last step we only need to consider a certain range of $\beta$ since values outside this range make the bound hold trivially. \raquel{this last sentence is hard to see without the proof... not sure if its worth highlighting here}

%% file: section/mpgnn_results.tex
\subsection{PAC-Bayes Bounds of MPGNNs}

For MPGNNs, we again need to perform a perturbation analysis to make sure that the change of the network output brought by the perturbations on weights is small with a large probability. 
Following the same strategy adopted in proving Lemma \ref{lemma:gcn_perturbation}, we prove the following Lemma. 
% \raquel{do you really mean recursively?}

\begin{lemma}\label{lemma:mpgnn_perturbation}
    (MPGNN Perturbation Bound) For any $B > 0, l > 1$, let $f_w \in \mathcal{H}: \mathcal{X} \times \mathcal{G} \rightarrow \mathbb{R}^{K}$ be a $l$-step MPGNN.
    % Let $\lambda = C_{\phi} B \Vert W_1 \Vert_{2}$ and $\tau = d \mathcal{C} \Vert W_2 \Vert_{2}$.
    Then for any $w$, and $x \in \mathcal{X}_{B,h_0}$, and any perturbation $u = \text{vec}(\{U_1, U_2, U_l\})$ such that $\eta = \max \left( \frac{\Vert U_1 \Vert_{2}}{\Vert W_1 \Vert_{2}}, \frac{\Vert U_2 \Vert_{2}}{\Vert W_2 \Vert_{2}}, \frac{\Vert U_l \Vert_{2}}{\Vert W_l \Vert_{2}} \right) \le \frac{1}{l}$, the change in the output of MPGNN is bounded as,
    \begin{align}
        \vert f_{w+u}(& X, A) - f_w(X, A) \vert_{2} \le e B l \eta \Vert W_1 \Vert_{2} \Vert W_l \Vert_{2} C_{\phi} \frac{\left( d \mathcal{C} \right)^{l-1} - 1}{d \mathcal{C} - 1},
        \nonumber
    \end{align}
    where $\mathcal{C} = C_{\phi} C_{\rho} C_{g} \Vert W_2 \Vert_{2}$.
\end{lemma}

The proof again involves decomposing the change into two terms which depend on two quantities respectively: the maximum change of node representations $\max_{i} \left\vert H_{l-1}^{\prime}[i, :] - H_{l-1}[i, :] \right\vert_{2}$ and the maximum node representation $\max_{i} \left\vert H_{l-1}[i, :] \right\vert_{2}$.
Then we perform an induction on the layer index to obtain their bounds individually.
Due to the weight sharing across steps, we have a form of geometric series $(( d \mathcal{C} )^{l-1} - 1) / (d \mathcal{C} - 1)$ rather than the product of spectral norms of each layer as in GCNs.
Technically speaking, the above lemma only works with $d \mathcal{C} \neq 1$. 
We refer the reader to the appendix for the special case of $d \mathcal{C} = 1$. 
%We need to also consider the case $d \mathcal{C} = 1$ for the completeness.
%However, since the latter happens rarely in practice, to save space we include it in the appendix.
We now provide the generalization bound for MPGNNs.
\begin{theorem}\label{thm:mpgnn_generalization_bound}
    (MPGNN Generalization Bound) For any $B > 0, l > 1$, let $f_w \in \mathcal{H}: \mathcal{X} \times \mathcal{G} \rightarrow \mathbb{R}^{K}$ be a $l$-step MPGNN. 
    Then for any $\delta, \gamma > 0$, with probability at least $1 - \delta$ over the choice of an i.i.d. size-$m$ training set $S$ according to $\mathcal{D}$, for any $w$, we have,
    \begin{align}
	    L_{\mathcal{D}, 0}(f_w) \le L_{S, \gamma}(f_w) + \mathcal{O} \left( \sqrt{\frac{ B^2 \left( \max\left(\zeta^{-(l+1)}, (\lambda \xi)^{(l+1)/l} \right) \right)^{2} l^2 h \log(lh) \vert w \vert_2^2 + \log \frac{m(l+1)}{\delta} }{\gamma^2 m}} \right), \nonumber
    \end{align}
    where 
    $\zeta = \min \left( \Vert W_1 \Vert_{2}, \Vert W_2 \Vert_{2}, \Vert W_l \Vert_{2} \right)$,
    $\vert w \vert_2^2 = \Vert W_1 \Vert_F^2 + \Vert W_2 \Vert_F^2 + \Vert W_l \Vert_F^2$,
    $\mathcal{C} = C_{\phi} C_{\rho} C_{g} \Vert W_2 \Vert_{2}$,
    $\lambda = \Vert W_1 \Vert_{2} \Vert W_l \Vert_{2}$, 
    and $\xi = C_{\phi} \frac{\left( d \mathcal{C} \right)^{l-1} - 1}{d \mathcal{C} - 1}$.
\end{theorem}

The proof also contains three steps:
% \raquel{use the same notation. Before the steps were (1), (2), etc. Do the same here}
1) since MPGNNs are typically non-homogeneous, \eg, when any of $\phi$, $\rho$, and $g$ is a bounded non-linearity like Sigmoid or Tanh, we design a special quantity of learned weights $\beta = \max( \zeta^{-1}, (\lambda \xi)^{{1}/{l}})$. 
% Note that $\zeta^{-1} = \max \left( 1/\Vert W_1 \Vert_{2}, 1/\Vert W_2 \Vert_{2}, 1/\Vert W_l \Vert_{2} \right)$ appears since we need to satisfy the assumption $\eta \le 1/l$ in Lemma \ref{lemma:mpgnn_perturbation}.
2) fixing any $\tilde{\beta}$, considering all $\beta$ that are in the range $\vert \beta - \tilde{\beta} \vert \le \beta/(l+1)$ and choosing $\sigma$ which depends on $\tilde{\beta}$ so that one can apply Lemma \ref{lemma:mpgnn_perturbation} and \ref{lemma:pac_bayes_deterministic} to work out the PAC-Bayes bound;
3) taking a union bound of the previous result by considering multiple choices of $\tilde{\beta}$ so that all possible values of $\beta$ are covered.
The case with $d \mathcal{C} = 1$ is again included in the appendix.
The first step is non-trivial since we do not have the nice construction as in the homogeneous case, \ie, normalizing the weights so that the spectral norms of weights across layers are the same while the network output is unchanged. 
Moreover, the quantity is vital to the whole proof framework since it determines whether one can 1) satisfy the constraints on the random perturbation (so that Lemma \ref{lemma:pac_bayes_deterministic} and \ref{lemma:mpgnn_perturbation} are applicable) and 2) simultaneously induce a finite covering on its range (so that the bound holds for any $w$).
Since it highly depends on the form of the perturbation bound and the network architecture, there seems to be no general recipe on how to construct such a quantity.

%% file: section/comparison_bounds.tex
\subsection{Comparison with Other Bounds}

\begin{table}[t]
\begin{center}
\resizebox{\textwidth}{!}
{
\begin{tabular}{c|cccc}
    \hline
    \toprule
    Statistics & \begin{tabular}{@{}c@{}}Max Node Degree \\ $d-1$ \end{tabular} & \begin{tabular}{@{}c@{}}Max Hidden Dim \\ $h$ \end{tabular} & \begin{tabular}{@{}c@{}}Spectral Norm of \\ Learned Weights \end{tabular} \\
    \midrule
    \midrule
    \begin{tabular}{@{}c@{}}VC-Dimension \\ \citep{scarselli2018vapnik}\end{tabular}
      & - & $\mathcal{O}\left( h^4 \right)$ & - \\ 
    \begin{tabular}{@{}c@{}}Rademacher \\ Complexity \\ \citep{garg2020generalization} \end{tabular} & $\mathcal{O}\left( d^{l-1} \sqrt{\log(d^{2l-3})} \right)$ & $\mathcal{O}\left( h \sqrt{\log h} \right)$ & $\mathcal{O}\left( \lambda \mathcal{C} \xi \sqrt{ \log \left( \Vert W_{2} \Vert_2 \lambda \xi^{2} \right)} \right)$ \\
    Ours & $\mathcal{O}\left( d^{l-1} \right)$ & $\mathcal{O}\left( \sqrt{h \log h} \right)$ & $\mathcal{O}\left( \lambda^{1 + \frac{1}{l}} \xi^{1 + \frac{1}{l}} \sqrt{ \Vert W_1 \Vert_F^2 + \Vert W_2 \Vert_F^2 + \Vert W_l \Vert_F^2 } \right)$ \\
    \bottomrule
\end{tabular}
}
\end{center}
\vspace{-0.1cm}
\caption{Comparison of generalization bounds for GNNs. ``-" means inapplicable.
$l$ is the network depth.
Here $\mathcal{C} = C_{\phi} C_{\rho} C_{g} \Vert W_2 \Vert_2$, 
$\xi = C_{\phi} \frac{\left( d \mathcal{C} \right)^{l-1} - 1}{d \mathcal{C} - 1}$, $\zeta = \min \left( \Vert W_1 \Vert_{2}, \Vert W_2 \Vert_{2}, \Vert W_l \Vert_{2} \right)$, 
and $\lambda = \Vert W_1 \Vert_{2} \Vert W_l \Vert_{2}$.
More details about the comparison can be found in Appendix \ref{sect:appendix_bound_comparison}.} 
\vspace{-0.5cm}
\label{table:comparison_bound}
\end{table}

In this section, we compare our generalization bounds with the ones in the GNN literature and draw connections with existing MLPs/CNNs bounds.

\subsubsection{Comparison with Existing GNN Generalization Bounds}

We compare against the VC-dimension based bound in \citep{scarselli2018vapnik} and the most recent Rademacher complexity based bound in \citep{garg2020generalization}.
Our results are not directly comparable to \citep{du2019graph} since they consider a ``infinite-wide'' class of GNNs constructed based on the neural tangent kernel \citep{jacot2018neural}, whereas we focus on commonly-used GNNs.
%It is also problematic to compare with \cite{verma2019stability} 
Comparisons to \citep{verma2019stability} are also difficult
since: 1) they only show the bound for one graph convolutional layer, \ie, it does not depend on the network depth $l$; and 2) their bound scales as $\mathcal{O}\left( \lambda_{\max}^{2T}/m \right)$, where $T$ is the number of SGD steps and $\lambda_{\max}$ is the maximum absolute eigenvalue of Laplacian $L = D - A$.
Therefore, for certain graphs\footnote{Since $\lambda_{\max} = \max_{v \neq \bm{0}} ({v^{\top}(D - A)v})/({v^{\top}v})$, we have $\lambda_{\max} \ge (D - A)[i,i]$ by choosing $v = \bm{e}_i$, \ie, the $i$-th standard basis. We can pick any node $i$ which has more than 1 neighbor to make $\lambda_{\max} > 1$.}, the generalization gap is monotonically increasing with $T$, which cannot explain the generalization phenomenon. 
We compare different bounds by examining their dependency on three terms:
%, \ie, 
the maximum node degree, the spectral norm of the learned weights, and the maximum hidden dimension.
We summarize the overall comparison in Table \ref{table:comparison_bound} and leave the details such as
%like
how we convert bounds into our context to Appendix \ref{sect:appendix_bound_comparison}.

% Note that since \cite{garg2020generalization} introduce extra assumptions like the uniform upper bound of the node update function $\phi$ and their final bound is quite involved, we leave the details of how to translate their results into our context to the appendix \ref{sect:appendix_bound_comparison}.

\paragraph{Max Node Degree $(d - 1)$:}
The Rademacher complexity bound scales as $\mathcal{O}\left( d^{l-1} \sqrt{\log(d^{2l-3})} \right)$ whereas ours scales as $\mathcal{O}(d^{l-1})$\footnote{Our bound actually scales as $\mathcal{O}\left( d^{{(l+1)(l-2)}/{l}} \right)$ which is upper bounded by $\mathcal{O}\left( d^{l-1} \right)$.}. 
% Note that the VC-dimension based bound \cite{scarselli2018vapnik} scale as $\mathcal{O}(N^2)$ where $N$ is the maximum number of nodes.
% \RZ{doesn't have to be better twice! - leave this one out}which is clearly better.
Many real-world graphs such as social networks tend to have large hubs \citep{barabasi2016network}, which lead to very large node degrees.
Thus, our bound would be significantly better in these scenarios.

% Moreover, many real-world networks/graphs like social networks are reported to be scale-free \cite{barabasi1999emergence}, \ie, the degree distribution follows a power law.
% The name of scale-free comes from the fact that the first moment (average degree) of the degree distribution is finite whereas the second one (variance) is infinite.
% In other words, if we randomly pick a node within the graph, its degree could deviate from the average degree arbitrarily.
% Therefore, the maximum node degree in such graphs tends to be very large, \eg, the largest hub in social networks \cite{barabasi2016network}, which means our bound would be significantly better in these situations.

\paragraph{Max Hidden Dimension $h$:}
Our bound scales as $\mathcal{O}(\sqrt{h \log h})$ which is tighter than the Rademacher complexity bound $\mathcal{O}\left( h \sqrt{\log h} \right)$ and the VC-dimension bound $\mathcal{O}(h^4)$. 

\paragraph{Spectral Norm of Learned Weights:}
As shown in Table \ref{table:comparison_bound}, we cannot compare the dependencies on the spectral norm of learned weights without knowing the actual values of the learned weights. 
Therefore, we perform an empirical study in Section \ref{sect:exp}.

\subsubsection{Connections with Existing Bounds of MLPs/CNNs}

As described above, MLPs/CNNs can be viewed as special cases of GNNs.
In particular, we have two ways to show the inclusion relationship.
First, we can treat each i.i.d. sample as a node and the whole dataset as a graph without edges.
Then conventional tasks (\eg, classification) become node-level tasks (\eg, node classification) on this graph.
Second, we can treat each i.i.d. sample as a single-node graph.
Then conventional tasks (\eg, classification) becomes graph-level tasks (\eg, graph classification).
Since we focus on the graph classification, we adopt the second view.
% This transformation allows us to convert the original task (\eg, classification) to node-level tasks (\eg, node classification).
% Since we consider graph classification problems, another transformation to achieve the same goal is treating each i.i.d. sample as a single-node graph.
% \raquel{I found the previous sentence not clear}
In particular, MLPs/CNNs with ReLU activations are equivalent to GCNs with the graph Laplacian $\tilde{L} = I$ (hence $d = 1$).
We leave the details of this conversion to Appendix \ref{sect:appendix_connections}.
We restate the PAC-Bayes bound for MLPs/CNNs with ReLU activations in \citep{neyshabur2017pac} as follows,
\begin{center}
\mbox{\small$\displaystyle
	L_{\mathcal{D}, 0}(f_w) \le L_{S, \gamma}(f_w) + \mathcal{O} \left( \sqrt{ \left( { B^2 l^2 h \log(lh) \prod\limits_{i=1}^{l} \Vert W_{i} \Vert_{2}^2 \sum\limits_{i=1}^{l} ( {\Vert W_{i} \Vert_F^2} / {\Vert W_{i} \Vert_{2}^2} ) + \log \frac{ml}{\delta} } \right) /{\gamma^2 m}} \right).
$}    
\end{center}
Comparing it with our bound for GCNs in Theorem \ref{thm:gcn_generalization_bound}, it is clear that we only add a factor $d^{l-1}$ to the first term inside the square root which is due to the underlying graph structure of the data.
If we apply GCNs to single-node graphs, the two bounds coincide since $d = 1$.
Therefore, our Theorem \ref{thm:gcn_generalization_bound} directly generalizes the result in \citep{neyshabur2017pac} to GCNs, which is a strictly larger class of models than MLPs/CNNs with ReLU activations.

%% file: section/experiments.tex
\section{Experiments}\label{sect:exp}

\begin{figure}[t]
\centering
\begin{subfigure}{.5\textwidth}
\vspace{-1cm}
\centering
\includegraphics[width=.95\linewidth]{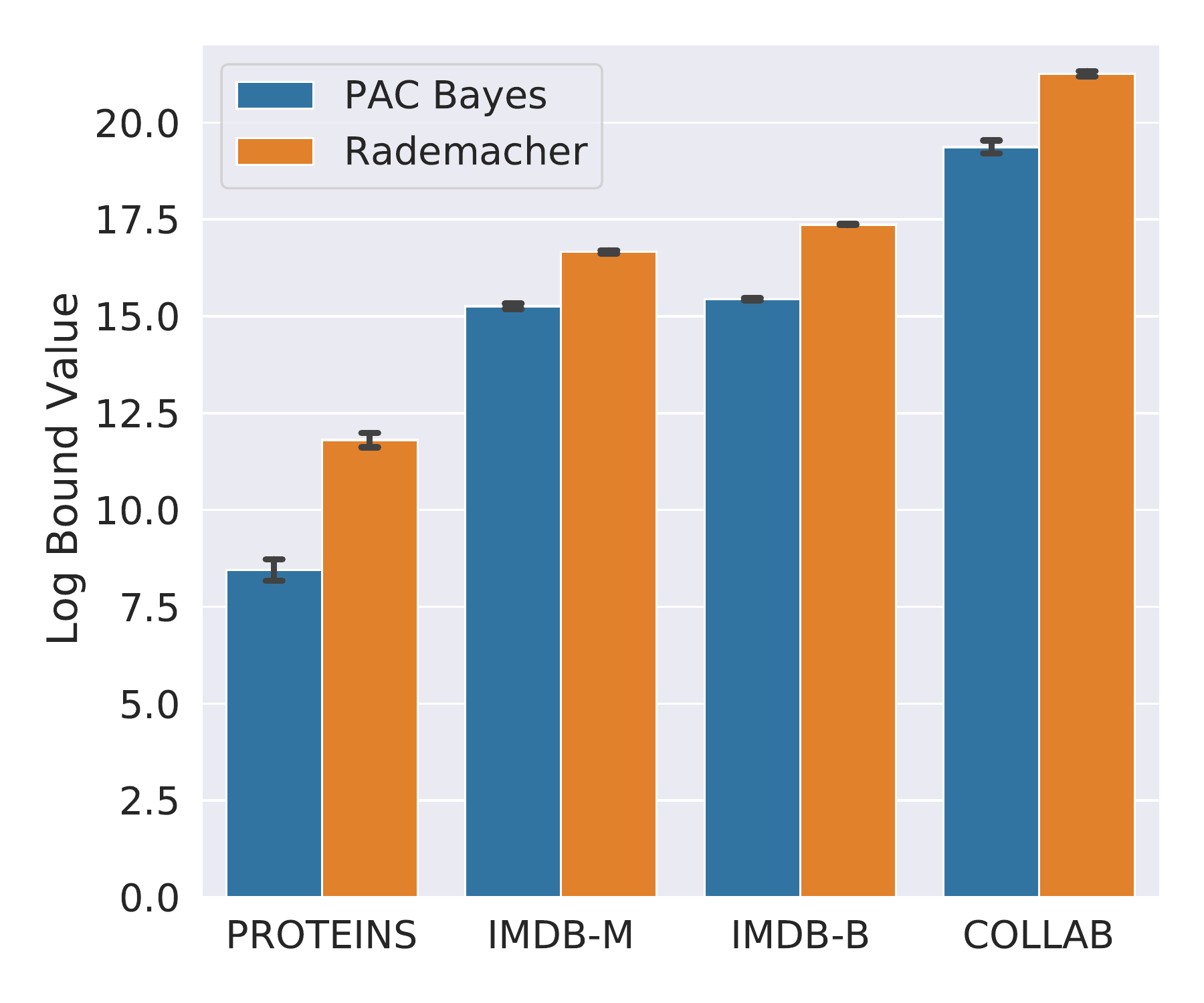}
\vspace{-0.3cm}
\caption{MPGNNs with $l = 2$.}
\label{fig:bound_l_2}
\end{subfigure}%
\begin{subfigure}{.5\textwidth}
\vspace{-1cm}
\centering
\includegraphics[width=.95\linewidth]{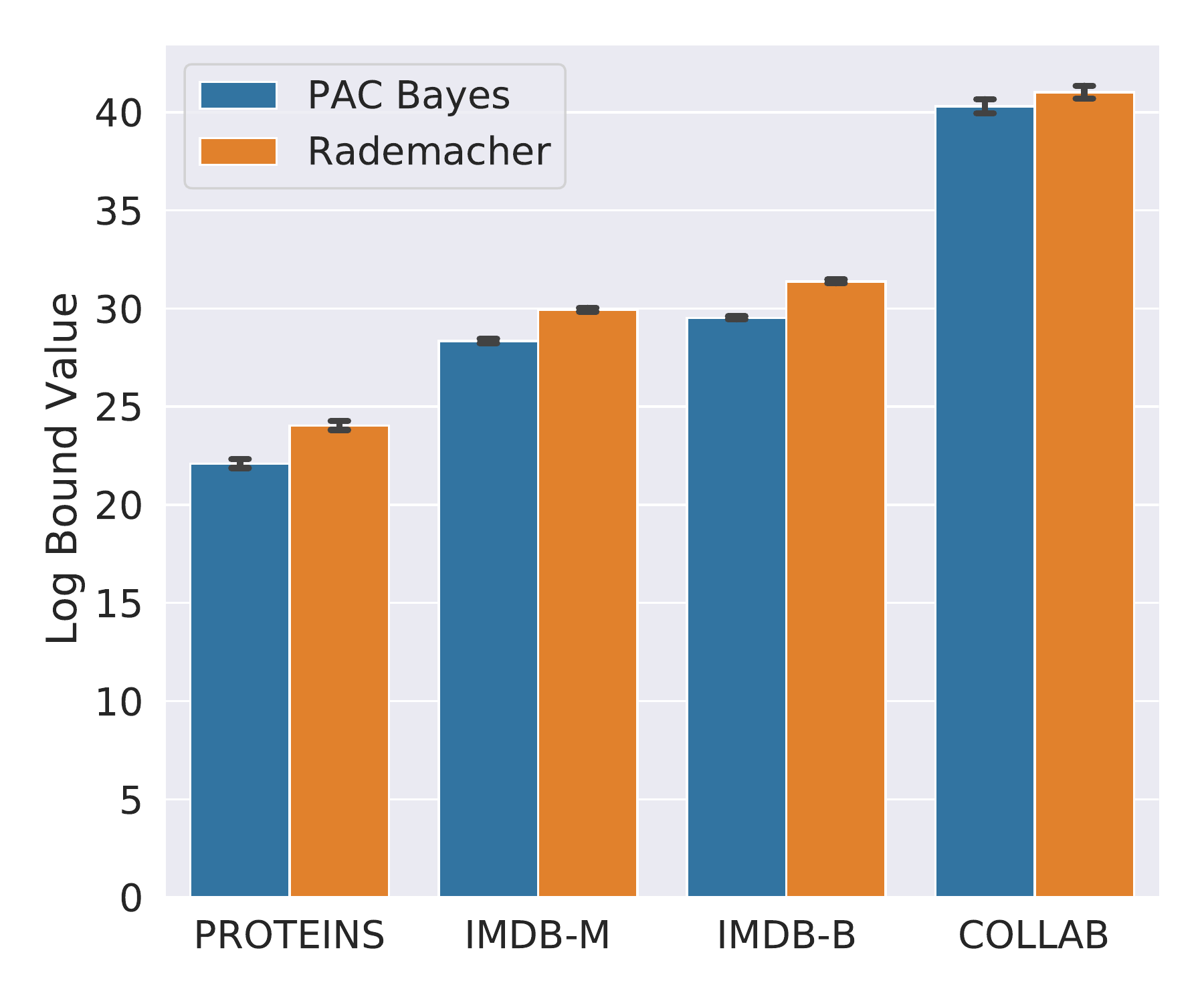}
\vspace{-0.3cm}
\caption{MPGNNs with $l = 4$.}
\label{fig:bound_l_4}
\end{subfigure}
\vspace{-0.3cm}
\caption{Bound evaluations on real-world datasets.
The maximum node degrees (\ie, $d-1$) of four datasets from left to right are: $25$ (PROTEINS), $88$ (IMDB-M), $135$ (IMDB-B), and $491$ (COLLAB).}
\vspace{-0.5cm}
\label{fig:bound_comparison}
\end{figure}

In this section, we perform an empirical comparison between our bound and the Rademacher complexity bound for MPGNNs.
We experiment on 6 synthetic datasets of random graphs (corresponding to 6 random graph models), 3 social network datasets (COLLAB, IMDB-BINARY, IMDB-MULTI), and a bioinformatics dataset PROTEINS from \citep{yanardag2015deep}.
In particular, we create synthetic datasets by generating random graphs from the Erdős–Rényi model and the stochastic block model with different settings (\ie, number of blocks and edge probabilities).
% All datesets focus on graph classifications with the maximum node degree ranging from $5$ to $87$.
All datesets focus on graph classifications.
% For all synthetic datasets, we generate random Gaussian node feature (normalized to have unit $\ell_2$ norm) as input and uniformly random binary class labels.
We repeat all experiments 3 times with different random initializations and report the means and the standard deviations.
Constants are considered in the bound computation.
More details of the experimental setup, dataset statistics, and the bound computation are provided in Appendix \ref{sect:appendix_exp}.

As shown in Fig. \ref{fig:bound_comparison} and Fig. \ref{fig:bound_comparison_syn}, our bound is mostly tighter than the Rademacher complexity bound with varying message passing steps $l$ on both synthetic and real-world datasets.
Generally, the larger the maximum node degree is, the more our bound improves\footnote{Note that it may not be obvious from the figure as the y axis is in log domain. Please refer to the appendix where the actual bound values are listed in the table.} over the Rademacher complexity bound (\cf, PROTEINS vs. COLLAB).
This could be attributed to the better dependency on $d$ of our bound.
For graphs with large node degrees (\eg, social networks like Twitter have influential users with lots of followers), the gap could be more significant.
% \raquel{you mean outperforms? as the bound is the same...}
Moreover, with the number of steps/layers increasing, our bound also improves more in most cases. 
It may not be clear to read from the figures since the y-axis is in the log domain and its range differ from figure to figure.
We also provide the numerical values of the bound evaluations in the appendix for an exact comparison.
The number of steps is chosen to be no larger than 10 as GNNs are generally shown to perform well with just a few steps/layers \citep{kipf2016semi,jin2018learning}. 
We found $d\mathcal{C} > 1$ and the geometric series $(( d \mathcal{C} )^{l-1} - 1) / (d \mathcal{C} - 1) \gg 1$ on all datasets which imply learned GNNs are not contraction mappings (\ie, $d\mathcal{C} < 1$).
This also explains why both bounds becomes larger with more steps.
At last, we can see that bound values are much larger than $1$ which indicates both bounds are still vacuous, similarly to the cases for regular neural networks in \citep{bartlett2017spectrally,neyshabur2017pac}. 
% \RZ{I think you have a bit more space where you can say more about the comparison, such as that this improved dependency on d will become even more important as the problem size scales up.}

\begin{figure}[t]
\centering
\begin{subfigure}{.5\textwidth}
\vspace{-1cm}
\centering
\includegraphics[width=.95\linewidth]{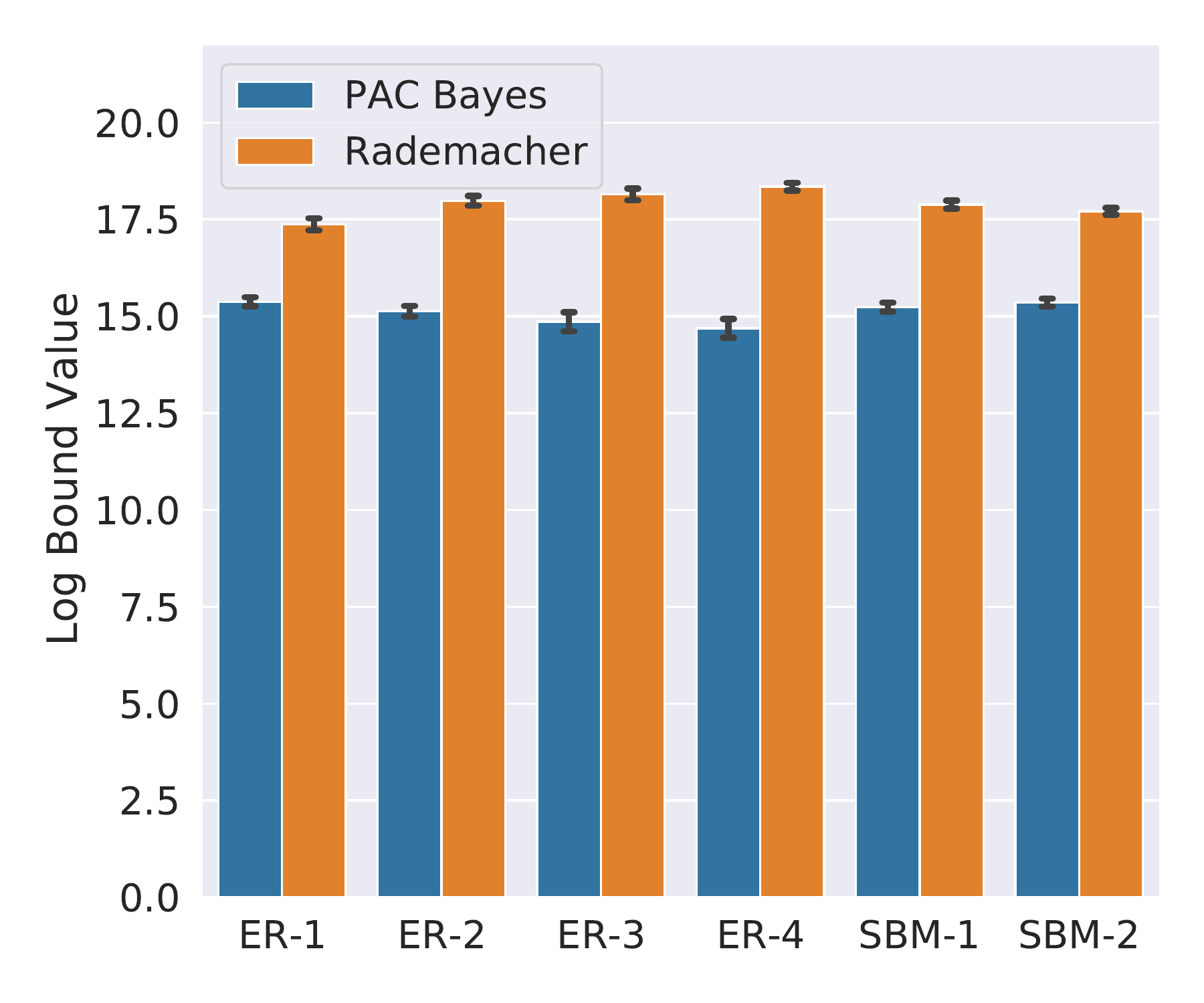}
\vspace{-0.4cm}
\caption{MPGNNs with $l = 2$.}
\label{fig:bound_l_2_syn}
\end{subfigure}%
\begin{subfigure}{.5\textwidth}
\vspace{-1cm}
\centering
\includegraphics[width=.95\linewidth]{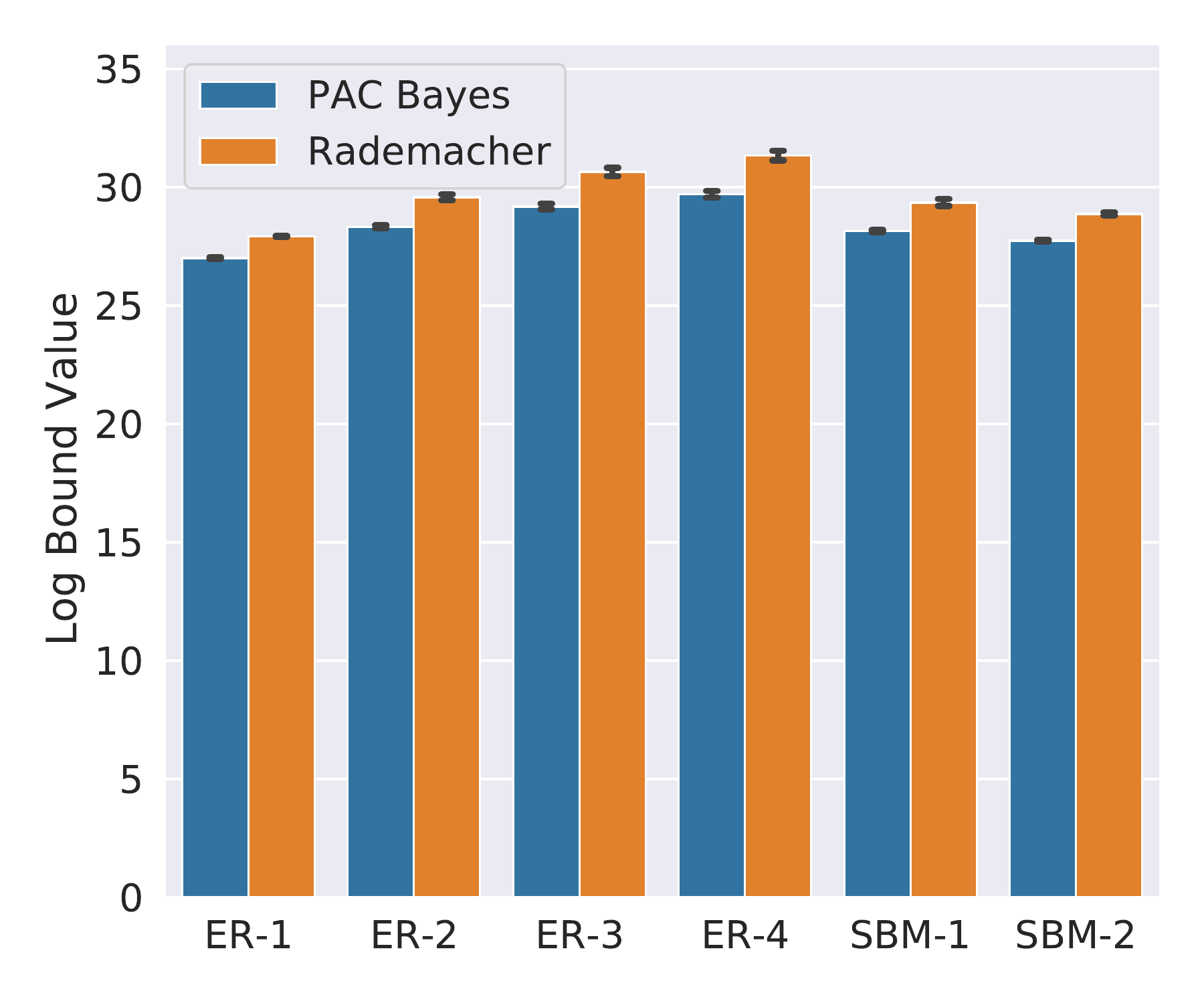}
\vspace{-0.4cm}
\caption{MPGNNs with $l = 4$.}
\label{fig:bound_l_4_syn}
\end{subfigure}
\begin{subfigure}{.5\textwidth}
\centering
\includegraphics[width=.95\linewidth]{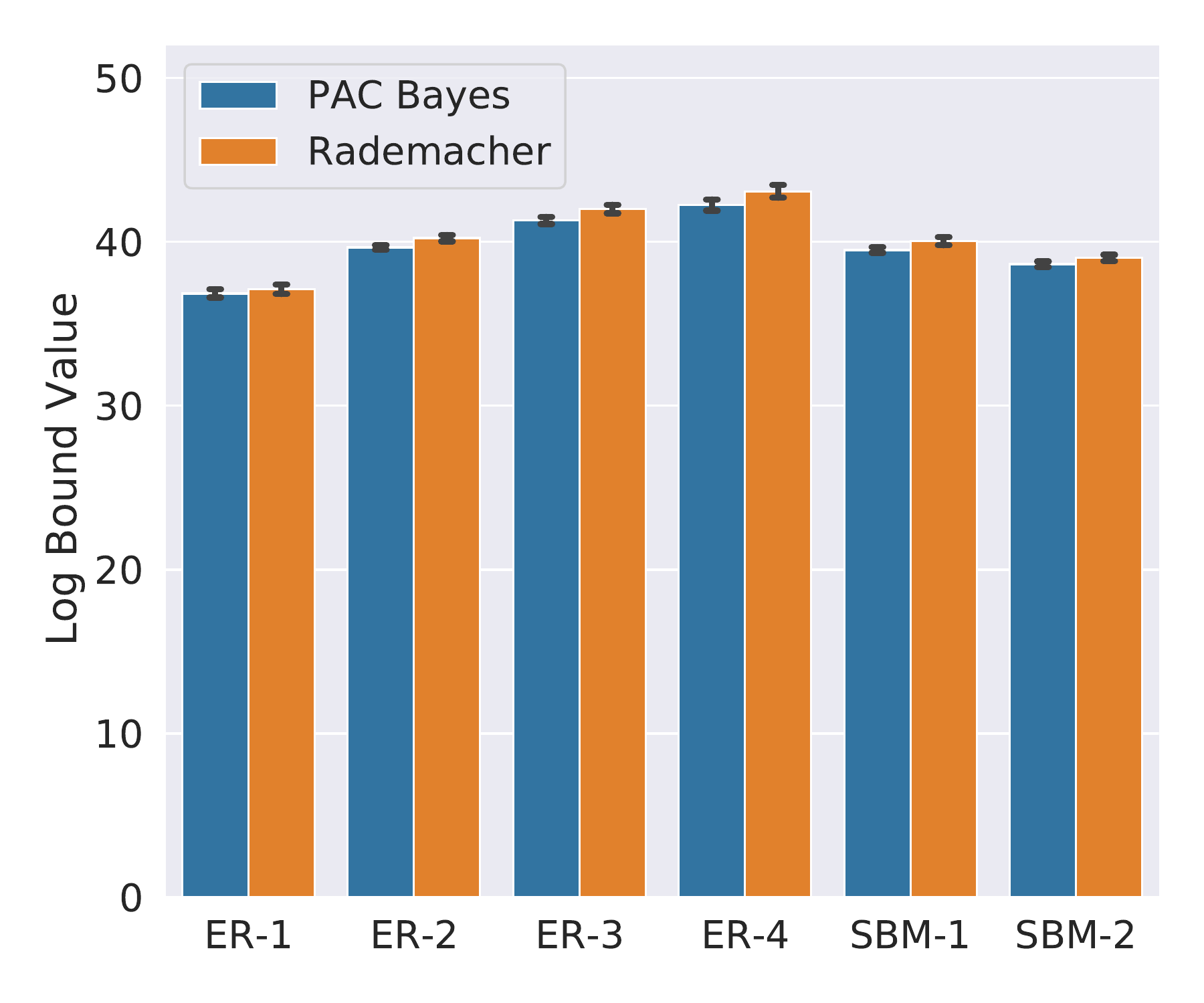}
\vspace{-0.3cm}
\caption{MPGNNs with $l = 6$.}
\label{fig:bound_l_6_syn}
\end{subfigure}%
\begin{subfigure}{.5\textwidth}
\centering
\includegraphics[width=.95\linewidth]{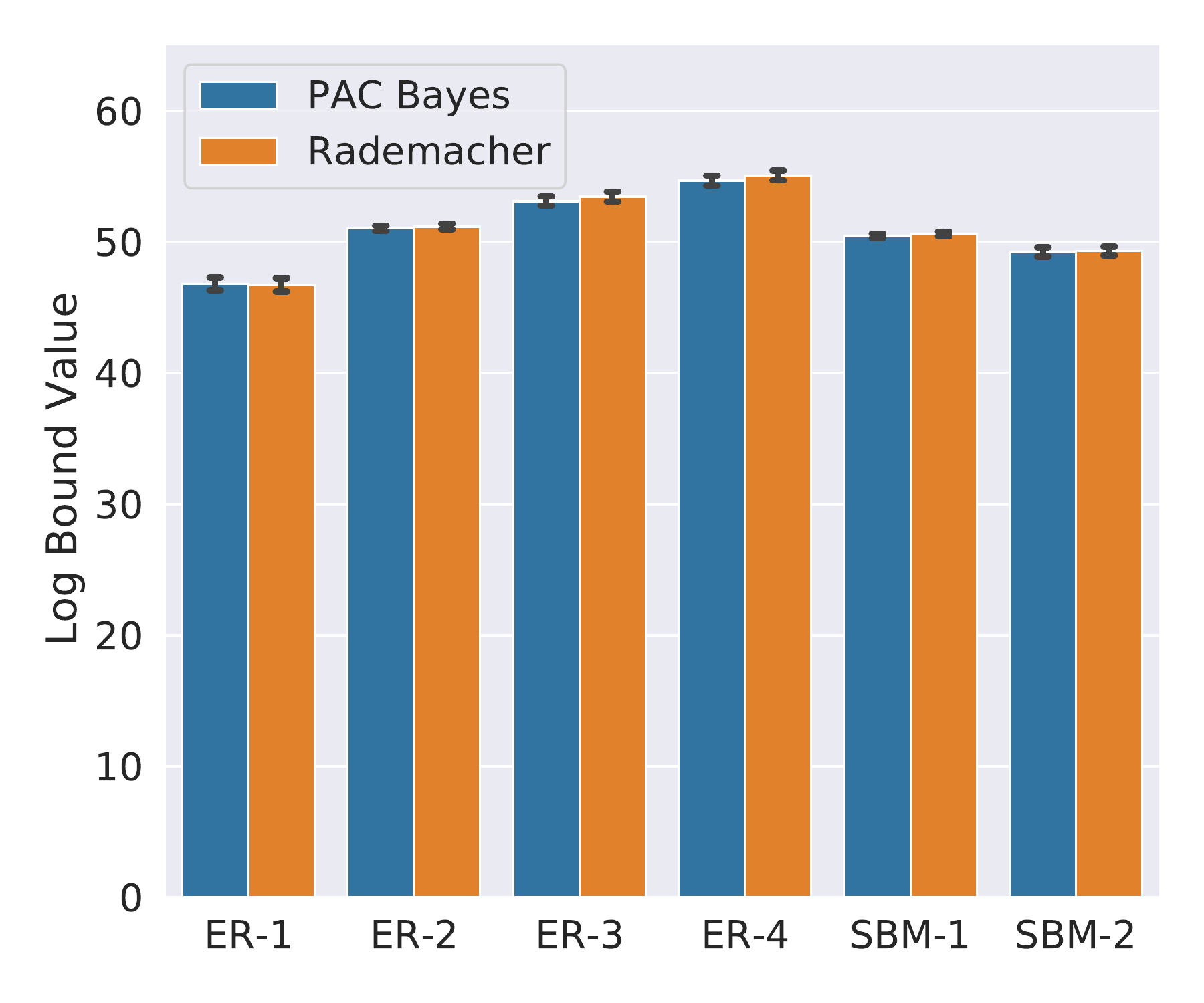}
\vspace{-0.3cm}
\caption{MPGNNs with $l = 8$.}
\label{fig:bound_l_8_syn}
\end{subfigure}
\vspace{-0.3cm}
\caption{Bound evaluations on synthetic datasets.
The maximum node degrees (\ie, $d-1$) of datasets from left to right are: $25$ (ER-1), $ 48$ (ER-2), $69$ (ER-3), $87$ (ER-4), $25$ (SBM-1), and $36$ (SBM-2).
`ER-X' and `SBM-X' denote the Erdős–Rényi model and the stochastic block model with the `X'-th setting respectively. Please refer to the appendix for more details.
}
\vspace{-0.5cm}
\label{fig:bound_comparison_syn}
\end{figure}

%% file: section/conclusions.tex
\vspace{-0.1cm}
\section{Discussion}\label{sect:discussion}
\vspace{-0.1cm}

In this paper, we present generalization bounds for two primary classes of GNNs, \ie, GCNs and MPGNNs.
We show that the maximum node degree and the spectral norms of learned weights govern the bound for both models.
Our results for GCNs generalize the bounds for MLPs/CNNs in \citep{neyshabur2017pac}, while our results for MPGNNs improve over the state-of-the-art Rademacher complexity bound in \citep{garg2020generalization}.
Our PAC-Bayes analysis can be generalized to other graph problems such as node classification and link prediction since our perturbation analysis bounds the maximum change of any node representation. 
Other loss functions (\eg, ones for regression) could also work in our analysis as long as they are bounded.
% Our analysis framework can be used to derive bounds for other MPGNN variants without substantial blocks since they only differ slightly in terms of the model.

However, we are far from being able to explain the practical behaviors of GNNs.
Our bound values are still vacuous as shown in the experiments. 
% \raquel{you didnt point this out in the experiments. You should}
Our perturbation analysis is in the worst-case sense which may be loose for most cases.
We introduce Gaussian posterior in the PAC-Bayes framework to obtain an analytical form of the KL divergence. 
Nevertheless, the actual posterior induced by the prior and the learning process may likely to be non-Gaussian.
We also do not explicitly consider the optimization algorithm in the analysis which clearly has an impact on the learned weights.
% \RZ{Other limitations? other graph problems, and loss functions?}

This work leads to a few interesting open problems for future work:
(1) Is the maximum node degree the only graph statistic that has an impact on the generalization ability of GNNs? Investigating other graph statistics may provide more insights on the behavior of GNNs and inspire the development of novel models and algorithms.
(2) Would the analysis still work for other interesting GNN architectures, such as those with attention \citep{velivckovic2017graph} and learnable spectral filters \citep{liao2019lanczos}?
(3) Can recent advancements for MLPs/CNNs, \eg, the compression technique in \citep{arora2018stronger} and data-dependent prior of \citep{parrado2012pac}, help further improve the bounds for GNNs?
(4) What is the impact of the optimization algorithms like SGD on the generalization ability of GNNs? Would graph structures play a role in the analysis of optimization?

\iffalse

\begin{itemize}
    \item Is the maximum node degree the only important graph statistic that has an impact on the generalization ability of GNNs? Investigating other graph statistics may provide more insights on the behavior of GNNs and inspire the development of novel models/algorithms.
    \item Would the analysis still work for other interesting GNN architectures, \eg, those with attention \cite{velivckovic2017graph} and learnable spectral filters \cite{liao2019lanczos}?
    \item Can recent results for MLPs/CNNs, \eg, the compression technique \cite{arora2018stronger} and data-dependent prior \cite{dziugaite2017computing}, help further improve the bounds for GNNs?
    \item What is the impact of the optimization algorithm like SGD on the generalization ability of GNNs? Would graph structures play a role so that the analysis of SGD is substantially different from the one in MLPs/CNNs?
\end{itemize}

\fi

%% file: section/appendix.tex
\section{Appendix}

We summarize the notations used throughout the paper in Table \ref{table:symbol_notation}.
In the following sections, we provide proofs of all results in the main text and the additional details.

\begin{table}[!htbp]
\begin{center}
% \resizebox{\textwidth}{!}
{
\begin{tabular}{c|c}
    \hline
    \toprule
    Symbol & Meaning \\
    \midrule
    \midrule
    $\mathbb{R}$ & the set of real numbers \\ 
    $\mathbb{R}^{m \times n}$ & the set of $m \times n$ real matrices \\ 
    $\mathbb{N}^{+}_{k}$ & the set of first $k$ positive numbers \\
    $\vert \cdot \vert_p$ & the vector p-norm \\
    $\Vert \cdot \Vert_p$ & the operator norm induced by the vector p-norm \\
    $\Vert \cdot \Vert_F$ & the Frobenius norm \\
    $X[i,j]$ & the ($i, j$)-th element of matrix $X$ \\
    $X[i,:]$ & the $i$-th row of matrix $X$ \\
    $X[:, i]$ & the $i$-th column of matrix $X$ \\
    $\bm{1}_n$ & a all-one vector with size $n$ \\
    $A$ & an adjacency matrix \\
    $\tilde{A}$ & an adjacency matrix plus the identity matrix \\
    $B$ & the radius of the $\ell_2$-ball where an input node feature lies \\
    $C_{\text{in}}$ & an incidence matrix of incoming nodes \\
    $C_{\text{out}}$ & an incidence matrix of outgoing nodes \\
    $\phi$, $\rho$, $g$ & non-linearities in MPGNN \\
    $C_{\phi}$, $C_{\rho}$, $C_{g}$ & Lipschitz constants of $\phi$, $\rho$, $g$ under the vector 2-norm \\
    $\mathcal{C}$ & the percolation complexity \\
    $D$ & the degree matrix \\
    $\mathcal{D}$ & the unknown data distribution \\
    $d$ & the maximum node degree plus one \\
    $e$ & the Euler's number \\
    $f_w$ & a model parameterized by vector $w$ \\
    $\mathcal{G}$ & the space of graph \\
    $h$ & the maximum hidden dimension \\
    $H$ & a node representation matrix \\
    $\mathcal{H}$ & the hypothesis/model class \\
    $I$ & the identity matrix \\
    $l$ & the number of graph convolution layers / message passing steps \\
    $L$ & the loss function \\
    $\tilde{L}$ & the Laplacian matrix \\
    $m$ & the number of training samples \\
    $\mathbb{P}$, $\mathbb{E}$ & probability and expectation of a random variable \\
    $P$ & the prior distribution over hypothesis class \\
    $Q$ & the posterior distribution over hypothesis class \\
    $S$ & a set of training samples \\
    $W$ & a weight matrix \\
    $X$ & a node feature matrix where each row corresponds to a node \\
    $\mathcal{X}$ & the space of node feature \\
    $y$ & the graph class label \\
    $\gamma$ & the margin parameter \\
    $z$ & a data triplet $(A, X, y)$ \\
    $\mathcal{Z}$ & the space of data triplet \\
    $\log$ & the natural logarithm \\
    \bottomrule
\end{tabular}
}
\end{center}
\caption{Summary of important notations.} 
\label{table:symbol_notation}
\end{table}

\input{proof/pac_bayes_proof}

\input{proof/graph_proof}

\input{proof/gcn_proof}

\input{proof/mpgnn_proof}

\input{proof/bound_comparison}

\input{proof/mlp_cnn_gnn_connection}

\input{proof/exp_detail}

%% file: proof/pac_bayes_proof.tex
\subsection{PAC Bayes Results}

For completeness, we provide the proofs of the standard PAC-Bayes results as below.

\begin{lemma}\label{lemma:expect_nonneg_RV}
    For non-negative continuous random variables $X$, we have
    \begin{align}
        \mathbb{E} [ X ] = \int_{0}^{\infty} \mathbb{P} (X \ge \nu) \mathrm{d}\nu. \nonumber
    \end{align}
\end{lemma}

\begin{proof}
    \begin{align}
        \mathbb{E} [ X ] & = \int_{0}^{\infty} X \mathbb{P} (X) \mathrm{d} X \nonumber \\
                         & = \int_{0}^{\infty} \int_{0}^{X} \bm{1} \mathrm{d} \nu \mathbb{P} (X) \mathrm{d} X \nonumber \\
                         & = \int_{0}^{\infty} \int_{0}^{X} \mathbb{P} (X) \mathrm{d} \nu \mathrm{d} X \nonumber \\
                         & = \int_{0}^{\infty} \int_{\nu}^{\infty} \mathbb{P} (X) \mathrm{d} X \mathrm{d} \nu \qquad (\text{region of the integral is the same}) \nonumber \\
                         & = \int_{0}^{\infty} \mathbb{P} (X \ge \nu) \mathrm{d} \nu \nonumber
    \end{align}
\end{proof}

\begin{lemma}\label{lemma:two_side_square_expect_bound}
    [2-side] Let $X$ be a random variable satisfying $\mathbb{P} (X \ge \epsilon) \le e^{-2m\epsilon^2}$ and $\mathbb{P} (X \le -\epsilon) \le e^{-2m\epsilon^2}$ where $m \ge 1$ and $\epsilon > 0$, we have
    \begin{align}
        \mathbb{E} [e^{2(m-1)X^2}] \le 2m. \nonumber
    \end{align}
\end{lemma}

\begin{proof}
    If $m = 1$, the inequality holds trivially.
    Let us now consider $m > 1$.
    \begin{align}\label{eq:expectation}
        \mathbb{E} [e^{2(m-1)X^2}] & = \int_{0}^{\infty} \mathbb{P} \left( e^{2(m-1)X^2} \ge \nu \right) \mathrm{d}\nu \qquad (\text{Lemma \ref{lemma:expect_nonneg_RV}}) \nonumber \\
        & = \int_{0}^{\infty} \mathbb{P} \left( X^2 \ge \frac{\log \nu}{2(m-1)} \right) \mathrm{d}\nu \qquad \nonumber \\
        & = \int_{0}^{\infty}  \mathbb{P} \left( X \ge \sqrt{ \frac{\log \nu}{2(m-1)} } \right) \mathrm{d}\nu + \int_{0}^{\infty}  \mathbb{P} \left( X \le - \sqrt{ \frac{\log \nu}{2(m-1)} } \right) \mathrm{d}\nu 
    \end{align}
    
    \begin{align}\label{eq:one_side_integral_bound}
        \int_{0}^{\infty} \mathbb{P} \left( X \ge \sqrt{ \frac{\log \nu}{2(m-1)} } \right) \mathrm{d}\nu 
        & = \int_{0}^{1}  \mathbb{P} \left( X \ge \sqrt{ \frac{\log \nu}{2(m-1)} } \right) \mathrm{d}\nu + \int_{1}^{\infty}  \mathbb{P} \left( X \ge \sqrt{ \frac{\log \nu}{2(m-1)} } \right) \mathrm{d}\nu \nonumber \\
        & \le 1 + \int_{1}^{\infty}  \mathbb{P} \left( X \ge \sqrt{ \frac{\log \nu}{2(m-1)} } \right) \mathrm{d}\nu \nonumber \\
        & \le 1 + \int_{1}^{\infty} e^{-2m \frac{\log \nu}{2(m-1)} } \mathrm{d}\nu \nonumber \\
        & = 1 + \left( \left. -(m-1) \nu^{- \frac{1}{m-1}} \right\rvert_{1}^{\infty} \right) \nonumber \\
        & = m 
    \end{align}
    
    Similarly, we can show that 
    \begin{align}
        \int_{0}^{\infty} \mathbb{P} \left( X \le - \sqrt{ \frac{\log \nu}{2(m-1)} } \right) \mathrm{d}\nu \le m
    \end{align}
    
    Combining Eq. (\ref{eq:expectation}) and Eq. (\ref{eq:one_side_integral_bound}), we finish the proof.
\end{proof}

\begin{reptheorem}{thm:pac_bayes}
    (Two-side) Let $P$ be a prior distribution over $\mathcal{H}$ and let $\delta \in (0, 1)$.
    Then, with probability $1 - \delta$ over the choice of an i.i.d. training set $S$ according to $\mathcal{D}$, for all distributions $Q$ over $\mathcal{H}$ and any $\gamma > 0$, we have 
    \begin{align}
    	L_{\mathcal{D}, \gamma}(Q) \le L_{S, \gamma}(Q) + \sqrt{\frac{\KL(Q \Vert P) + \log \frac{2m}{\delta} }{2(m-1)}} \nonumber
    \end{align}
\end{reptheorem}

\begin{proof}
    Let $\Delta(h) = L_{\mathcal{D}, \gamma}(h) - L_{S, \gamma}(h)$.
    For any function $f(h)$, we have
    \begin{align}\label{eq:expectation_h_upper_bound}
        \mathbb{E}_{h \sim Q} [f(h)] & = \mathbb{E}_{h \sim Q} [ \log e^{f(h)} ] \nonumber \\
        & = \mathbb{E}_{h \sim Q} [ \log e^{f(h)} + \log \frac{Q}{P} + \log \frac{P}{Q} ] \nonumber \\
        & = \KL(Q \Vert P) + \mathbb{E}_{h \sim Q} \left[ \log \left( \frac{P}{Q} e^{f(h)} \right) \right] \nonumber \\
        & \le \KL(Q \Vert P) + \log \mathbb{E}_{h \sim Q} \left[ \frac{P}{Q} e^{f(h)} \right] \qquad (\text{Jensen's inequality}) \nonumber \\
        & = \KL(Q \Vert P) + \log \mathbb{E}_{h \sim P} \left[ e^{f(h)} \right].
    \end{align}
    Let $f(h) = 2(m-1) \Delta(h)^2$.
    We have
    \begin{align}\label{eq:pac_bayes_kl_binary_binomial_tmp_3}
        2(m-1) \mathbb{E}_{h \sim Q} [\Delta(h)]^2 & \le 2(m-1) \mathbb{E}_{h \sim Q} [\Delta(h)^2] \qquad (\text{Jensen's inequality}) \nonumber \\
        & \le \KL(Q \Vert P) + \log \mathbb{E}_{h \sim P} \left[ e^{2(m-1) \Delta(h)^2} \right].
    \end{align}    
    Since $L_\mathcal{D}(h) \in [0, 1]$, based on Hoeffding's inequality, for any $\epsilon > 0$, we have 
    \begin{align}
        \mathbb{P} ( \Delta(h) \ge \epsilon ) \le e^{-2m\epsilon^2} \nonumber \\
        \mathbb{P} ( \Delta(h) \le -\epsilon ) \le e^{-2m\epsilon^2} \nonumber
    \end{align}
    Hence, based on Lemma \ref{lemma:two_side_square_expect_bound}, we have
    \begin{align}
        \mathbb{E}_{S} \left[ e^{2(m-1) \Delta(h)^2} \right] \le 2m ~ & \Rightarrow ~ \mathbb{E}_{h \sim P} \left[ \mathbb{E}_{S} \left[ e^{2(m-1) \Delta(h)^2} \right] \right] \le 2m \nonumber \\ 
        & \Leftrightarrow ~ \mathbb{E}_{S} \left[ \mathbb{E}_{h \sim P} \left[ e^{2(m-1) \Delta(h)^2} \right] \right] \le 2m \nonumber
    \end{align}
    
    Based on Markov's inequality, we have
    \begin{align}\label{eq:pac_bayes_kl_binary_binomial_tmp_4}
        \mathbb{P}\left( \mathbb{E}_{h \sim P} \left[ e^{2(m-1) \Delta(h)^2} \right] \ge \frac{2m}{\delta} \right) \le \frac{ \delta \mathbb{E}_{S} \left[ \mathbb{E}_{h \sim P} \left[ e^{2(m-1) \Delta(h)^2} \right] \right]}{2m} \le \delta.
    \end{align}
    
    Combining Eq. (\ref{eq:pac_bayes_kl_binary_binomial_tmp_3}) and Eq. (\ref{eq:pac_bayes_kl_binary_binomial_tmp_4}), with probability $1 - \delta$, we have
    \begin{align}
        \mathbb{E}_{h \sim Q} [\Delta(h)]^2 \le \frac{\KL(Q \Vert P) + \log \left( \frac{2m}{\delta} \right)}{2(m-1)}
    \end{align}
    which proves the theorem.
\end{proof}

\begin{replemma}{lemma:pac_bayes_deterministic}
    Let $f_w(x): \mathcal{X} \rightarrow \mathbb{R}^k $ be any model with parameters $w$, and $P$ be any distribution on the parameters that is independent of the training data.
    For any $w$, we construct a posterior $Q(w + u)$ by adding any random perturbation $u$ to $w$, \st, $\mathbb{P}(\max_{x \in \mathcal{X}} \vert f_{w + u}(x) - f_w(x) \vert_{\infty} < \frac{\gamma}{4} ) > \frac{1}{2}$.
    Then, for any $\gamma, \delta > 0$, with probability at least $1 - \delta$ over an i.i.d. size-$m$ training set $S$ according to $\mathcal{D}$, for any $w$, we have:
    \begin{align}
        L_{\mathcal{D}, 0}(f_w) \le L_{S, \gamma}(f_w) + \sqrt{\frac{2\KL(Q(w+u) \Vert P) + \log \frac{8m}{\delta} }{2(m-1)}} \nonumber
    \end{align}
\end{replemma}

\begin{proof}
    Let $\tilde{w} = w + u$.
    Let $\mathcal{C}$ be the set of perturbation with the following property,
    \begin{align}
        \mathcal{C} = \left\{ w^{\prime} \middle \vert \max_{x \in \mathcal{X}} \vert f_{w^{\prime}}(x) - f_w(x) \vert_{\infty} < \frac{\gamma}{4} \right\}.
    \end{align}
    
    $\tilde{w} = w + u$ ($w$ is deterministic and $u$ is stochastic) is distributed according to $Q(\tilde{w})$. 
    We now construct a new posterior $\tilde{Q}$ as follows, 
    \begin{align}
        \tilde{Q}(\tilde{w}) = 
        \begin{cases}
            \frac{1}{Z}Q(\tilde{w}) & \tilde{w} \in \mathcal{C} \\
            0 & \tilde{w} \in \bar{\mathcal{C}}.
        \end{cases} 
    \end{align}
    Here $Z = \int_{\tilde{w} \in \mathcal{C}} \mathrm{d} Q(\tilde{w}) = \underset{\tilde{w} \sim Q}{\mathbb{P}}(\tilde{w} \in \mathcal{C})$ and $\bar{\mathcal{C}}$ is the complement set of $\mathcal{C}$.
    We know from the assumption that $Z > \frac{1}{2}$.
    Therefore, for any $\tilde{w} \sim \tilde{Q}$, we have
    \begin{align}\label{eq:pac_bayes_deterministic_tmp_1}
        & \left. \max_{i \in \mathbb{N}_{k}^{+}, j \in \mathbb{N}_{k}^{+}, x \in \mathcal{X}} \middle\vert \left\vert f_{\tilde{w}}(x)[i] - f_{\tilde{w}}(x)[j] \right\vert -  \vert f_{w}(x)[i] - f_{w}(x)[j] \vert \right\vert \nonumber \\
        \le & \left. \max_{i \in \mathbb{N}_{k}^{+}, j \in \mathbb{N}_{k}^{+}, x \in \mathcal{X}} \middle\vert f_{\tilde{w}}(x)[i] - f_{\tilde{w}}(x)[j] - f_{w}(x)[i] + f_{w}(x)[j] \right\vert \nonumber \\
        \le & \left. \max_{i \in \mathbb{N}_{k}^{+}, j \in \mathbb{N}_{k}^{+}, x \in \mathcal{X}} \middle\vert f_{\tilde{w}}(x)[i] - f_{w}(x)[i] \middle\vert + \middle\vert f_{\tilde{w}}(x)[j] - f_{w}(x)[j] \right\vert \nonumber \\
        \le & \left. \max_{i \in \mathbb{N}_{k}^{+}, x \in \mathcal{X}} \middle\vert f_{\tilde{w}}(x)[i] - f_{w}(x)[i] \middle\vert + \max_{j \in \mathbb{N}_{k}^{+}, x \in \mathcal{X}} \middle\vert f_{\tilde{w}}(x)[j] - f_{w}(x)[j] \right\vert \nonumber \\
        & < \frac{\gamma}{4} + \frac{\gamma}{4} = \frac{\gamma}{2}
    \end{align}
    Recall that 
    \begin{align}
        L_{\mathcal{D}}(f_w, 0) & = \underset{z \sim \mathcal{D}}{\mathbb{P}} \left( f_w(x)[y] \le \max_{j \neq y}f_w(x)[j] \right) \nonumber \\
        L_{\mathcal{D}}(f_{\tilde{w}}, \frac{\gamma}{2}) & = \underset{z \sim \mathcal{D}}{\mathbb{P}} \left( f_{\tilde{w}}(x)[y] \le \frac{\gamma}{2} + \max_{j \neq y}f_{\tilde{w}}(x)[j] \right). \nonumber
    \end{align}
    Denoting $j_1^{\ast} = \argmax_{j \neq y}f_{\tilde{w}}(x)[j]$ and $j_2^{\ast} = \argmax_{j \neq y}f_w(x)[j]$, from Eq. (\ref{eq:pac_bayes_deterministic_tmp_1}), we have
    \begin{align}\label{eq:pac_bayes_deterministic_tmp_2}
        & \left\vert f_{\tilde{w}}(x)[y] - f_{\tilde{w}}(x)[j_2^{\ast}] - f_w(x)[y] + f_w(x)[j_2^{\ast}] \middle\vert < \frac{\gamma}{2} \right. \nonumber \\
        \Rightarrow ~ & f_{\tilde{w}}(x)[y] - f_{\tilde{w}}(x)[j_2^{\ast}] < f_w(x)[y] - f_w(x)[j_2^{\ast}] + \frac{\gamma}{2} 
    \end{align}
    Note that since $f_{\tilde{w}}(x)[j_1^{\ast}] \ge f_{\tilde{w}}(x)[j_2^{\ast}]$, we have
    \begin{align}
        f_{\tilde{w}}(x)[y] - f_{\tilde{w}}(x)[j_1^{\ast}] & \le f_{\tilde{w}}(x)[y] - f_{\tilde{w}}(x)[j_2^{\ast}] \nonumber \\
        & \le f_w(x)[y] - f_w(x)[j_2^{\ast}] + \frac{\gamma}{2} \qquad (\text{Eq. (\ref{eq:pac_bayes_deterministic_tmp_2})}) \nonumber
    \end{align}
    Therefore, we have
    \begin{align}
        f_w(x)[y] - f_w(x)[j_2^{\ast}] \le 0 ~ \Rightarrow ~ f_{\tilde{w}}(x)[y] - f_{\tilde{w}}(x)[j_1^{\ast}] \le \frac{\gamma}{2}, \nonumber
    \end{align}
    which indicates $\underset{z \sim \mathcal{D}}{\mathbb{P}} \left( f_w(x)[y] \le f_w(x)[j_2^{\ast}] \right) \le \underset{z \sim \mathcal{D}}{\mathbb{P}} \left( f_{\tilde{w}}(x)[y] \le f_{\tilde{w}}(x)[j_1^{\ast}] + \frac{\gamma}{2} \right)$, or equivalently
    \begin{align}\label{eq:pac_bayes_deterministic_tmp_3}
        L_{\mathcal{D}, 0}(f_w) \le L_{\mathcal{D}, \frac{\gamma}{2}}(f_{\tilde{w}}).
    \end{align}
    Note that this holds for any perturbation $\tilde{w} \sim \tilde{Q}$.
    
    Again, recall that 
    \begin{align}
        L_{\mathcal{D}, \frac{\gamma}{2}}(f_{\tilde{w}}) & = \underset{z \sim \mathcal{D}}{\mathbb{P}} \left( f_{\tilde{w}}(x)[y] \le \frac{\gamma}{2} + \max_{j \neq y}f_{\tilde{w}}(x)[j] \right) \nonumber \\
        L_{\mathcal{D}, \gamma}(f_w) & = \underset{z \sim \mathcal{D}}{\mathbb{P}} \left( f_w(x)[y] \le \gamma + \max_{j \neq y}f_w(x)[j] \right) \nonumber
    \end{align}
    From Eq. (\ref{eq:pac_bayes_deterministic_tmp_1}), we have
    \begin{align}\label{eq:pac_bayes_deterministic_tmp_4}
        & \left\vert f_{\tilde{w}}(x)[y] - f_{\tilde{w}}(x)[j_1^{\ast}] - f_w(x)[y] + f_w(x)[j_1^{\ast}] \middle\vert < \frac{\gamma}{2} \right. \nonumber \\
        \Rightarrow ~ & f_w(x)[y] - f_w(x)[j_1^{\ast}] < f_{\tilde{w}}(x)[y] - f_{\tilde{w}}(x)[j_1^{\ast}] + \frac{\gamma}{2} 
    \end{align}
    Note that since $f_{{w}}(x)[j_2^{\ast}] \ge f_{{w}}(x)[j_1^{\ast}]$, we have
    \begin{align}
        f_{{w}}(x)[y] - f_{{w}}(x)[j_2^{\ast}] & \le f_{{w}}(x)[y] - f_{{w}}(x)[j_1^{\ast}] \nonumber \\
        & \le f_{\tilde{w}}(x)[y] - f_{\tilde{w}}(x)[j_1^{\ast}] + \frac{\gamma}{2} \qquad (\text{Eq. (\ref{eq:pac_bayes_deterministic_tmp_4})}) \nonumber
    \end{align}
    Therefore, we have
    \begin{align}
        f_{\tilde{w}}(x)[y] - f_{\tilde{w}}(x)[j_1^{\ast}] \le \frac{\gamma}{2} ~ \Rightarrow ~ f_{{w}}(x)[y] - f_{{w}}(x)[j_2^{\ast}] \le \gamma, \nonumber
    \end{align}
    which indicates $L_{\mathcal{D}, \frac{\gamma}{2}}(f_{\tilde{w}}) \le L_{\mathcal{D}, \gamma}(f_{{w}})$. 
    Therefore, from the perspective of the empirical estimation of the probability, for any $\tilde{w} \sim \tilde{Q}$, we almost surely have
    \begin{align}\label{eq:pac_bayes_deterministic_tmp_5}
        L_{S, \frac{\gamma}{2}}(f_{\tilde{w}}) \le L_{S, \gamma}(f_{{w}}).    
    \end{align}
    Now with probability at least $1 - \delta$, we have
    \begin{align}\label{eq:pac_bayes_deterministic_tmp_6}
        L_{\mathcal{D}, 0}(f_w) & \le \mathbb{E}_{\tilde{w} \sim \tilde{Q}} \left[ L_{\mathcal{D}, \frac{\gamma}{2}}(f_{\tilde{w}}) \right] \qquad (\text{Eq. (\ref{eq:pac_bayes_deterministic_tmp_3})}) \nonumber \\
        & \le \mathbb{E}_{\tilde{w} \sim \tilde{Q}} \left[ L_{S, \frac{\gamma}{2}}(f_{\tilde{w}}) \right] + \sqrt{\frac{\KL(\tilde{Q} \Vert P) + \log \frac{2m}{\delta} }{2(m-1)}} \qquad (\text{Theorem \ref{thm:pac_bayes}}) \nonumber \\
        & \le L_{S, \gamma}(f_{w}) + \sqrt{\frac{\KL(\tilde{Q} \Vert P) + \log \frac{2m}{\delta} }{2(m-1)}} \qquad (\text{Eq. (\ref{eq:pac_bayes_deterministic_tmp_5})})
    \end{align}
    Note that 
    {
    \begin{align}\label{eq:pac_bayes_deterministic_tmp_7}
        \KL(Q \Vert P) & = \underset{{\tilde{w} \in \mathcal{C}}}{\int} Q \log \frac{Q}{P} \mathrm{d}\tilde{w} + \underset{\tilde{w} \in \bar{\mathcal{C}}}{\int} Q \log \frac{Q}{P} \mathrm{d}\tilde{w} \nonumber \\
        & = \underset{\tilde{w} \in \mathcal{C}}{\int} \frac{QZ}{Z} \log \frac{Q}{ZP} \mathrm{d}\tilde{w} + \underset{\tilde{w} \in \mathcal{C}}{\int} Q \log Z \mathrm{d}\tilde{w} \nonumber \\ 
        & \quad + \underset{\tilde{w} \in \bar{\mathcal{C}}}{\int} \frac{Q(1 - Z)}{1 - Z} \log \frac{Q}{(1 - Z)P} \mathrm{d}\tilde{w} + \underset{\tilde{w} \in \bar{\mathcal{C}}}{\int} Q \log (1-Z) \mathrm{d}\tilde{w} \nonumber \\
        & = Z\KL(\tilde{Q} \Vert P) + (1-Z)\KL(\bar{Q} \Vert P) - H(Z),
    \end{align}}%
    where $\bar{Q}$ denotes the normalized density of $Q$ restricted to $\bar{\mathcal{C}}$.
    $H(Z)$ is the entropy of a Bernoulli random variable with parameter $Z$.
    Since we know $\frac{1}{2} \le Z \le 1$ from the beginning, $0 \le H(Z) \le \log2$, and $\KL$ is nonnegative, from Eq. (\ref{eq:pac_bayes_deterministic_tmp_7}), we have
    \begin{align}\label{eq:pac_bayes_deterministic_tmp_8}
        \KL(\tilde{Q} \Vert P) & = \frac{1}{Z} \left[ \KL(Q \Vert P) + H(Z) - (1-Z)\KL(\bar{Q} \Vert P) \right] \nonumber \\
        & \le \frac{1}{Z} \left[ \KL(Q \Vert P) + H(Z) \right] \nonumber \\
        & \le 2 \KL(Q \Vert P) + 2\log2.
    \end{align}
    Combining Eq. (\ref{eq:pac_bayes_deterministic_tmp_6}) and Eq. (\ref{eq:pac_bayes_deterministic_tmp_8}), we have 
    \begin{align}
        L_{\mathcal{D}, 0}(f_w) \le L_{S, \gamma}(f_{w}) + \sqrt{\frac{\KL(Q \Vert P) + \frac{1}{2} \log \frac{8m}{\delta} }{m-1}},
    \end{align}
    which finishes the proof.
\end{proof}

%% file: proof/graph_proof.tex
\subsection{Graph Results}

In this part, we provide a result on the graph Laplacian used by GCNs along with the proof.
It is used in the perturbation analysis of GCNs.

\begin{lemma}\label{lemma:symmetric_laplacian_norm}
Let $A$ be the binary adjacency matrix of an arbitrary simple graph $G = (V, E)$ and $\tilde{A} = A + I$.
We define the graph Laplacian $L = D^{-\frac{1}{2}} \tilde{A} D^{-\frac{1}{2}}$ where $D$ is the degree matrix of $\tilde{A}$. Then we have $\Vert L \Vert_{1} = \Vert L \Vert_{\infty} \le \sqrt{d}$, $\Vert L \Vert_{2} \le 1$, and $\Vert L \Vert_{F} \le \sqrt{r}$ where $r$ is the rank of $L$ and $d-1$ is the maximum node degree of $G$.
\end{lemma}
\begin{proof}
    First, $\tilde{A}$ is symmetric and element-wise nonnegative. 
    Denoting $n = \vert V \vert$, we have $\tilde{A} \in \mathbb{R}^{n \times n}$, $D_{i} = \sum_{j=1}^{n} \tilde{A}[i,j]$, and $1 \le D_{i} \le d, \forall i \in \mathbb{N}^{+}_{n}$.
    It is easy to show that $L[i,j] = {\tilde{A}[i,j]}/{\sqrt{D_{i}D_{j}}}$.
    
    For the infinity norm and 1-norm, we have $\Vert L \Vert_{1} = \Vert L^{\top} \Vert_{\infty} = \Vert L \Vert_{\infty}$.
    Moreover,
    \begin{align}
        \Vert L \Vert_{\infty} & = \max_{i \in \mathbb{N}^{+}_{n}} \sum_{j=1}^{n} \left\vert L[i,j] \right\vert \nonumber \\
        & = \max_{i \in \mathbb{N}^{+}_{n}} \sum_{j=1}^{n} \frac{\tilde{A}[i,j]}{\sqrt{D_{i}D_{j}}} \nonumber \\
        & \le \max_{i \in \mathbb{N}^{+}_{n}} \frac{1}{\sqrt{D_{i}}} \sum_{j=1}^{n} \tilde{A}[i,j] \nonumber \\
        & = \max_{i \in \mathbb{N}^{+}_{n}} \sqrt{D_{i}} \nonumber \\
        & \le \sqrt{d} \label{eq:symmetric_laplacian_tmp_2}
    \end{align}

    % \begin{align}
    %     \Vert L \Vert_{\infty} & = \max_{i \in \mathbb{N}^{+}_{n}} \sum_{j=1}^{n} \left\vert L[i,j] \right\vert = \max_{i \in \mathbb{N}^{+}_{n}} \sum_{j=1}^{n} {\tilde{A}[i,j]}/{\sqrt{D_{i}D_{j}}} \nonumber \\
    %     & = \max_{i \in \mathbb{N}^{+}_{n}} \frac{1}{D_{i}} + \sum_{j \in \mathcal{N}_i, j \neq i} \frac{1}{\sqrt{D_{i}D_{j}}} \nonumber \\
    %     & \le \max_{i \in \mathbb{N}^{+}_{n}} \frac{1}{D_{i}} + \frac{1}{\sqrt{D_{i}}} \frac{D_{i}-1}{\sqrt{2}}  \qquad (D_{j} \ge 2 \text{ since node $j$ at least connects $i$ and itself}) \nonumber \\
    %     & \le \frac{1}{d} + \frac{d-1}{\sqrt{2d}} \label{eq:symmetric_laplacian_tmp_1} \\
    %     & \le \sqrt{d} \label{eq:symmetric_laplacian_tmp_2}
    % \end{align}
    % where $\mathcal{N}_i$ is the set of neighboring nodes (including itself) of node $i$.
    % In the second last inequality, we use the fact that $h(x) = \frac{1}{x} + \frac{x-1}{\sqrt{2x}}$ is increasing when $x \ge 2$ and $h(1) = h(2) = 1$.
    % In the last inequality, we use the fact that $h(x) = \frac{1}{x} + \frac{x-1}{\sqrt{2x}} \le \frac{1}{x} + \frac{\sqrt{x}}{\sqrt{2}} \le \sqrt{x}$ when $x \ge 2$ and $h(1) = 1 \le \sqrt{1}$.
    % Although one can use both Eq. (\ref{eq:symmetric_laplacian_tmp_1}) and Eq. (\ref{eq:symmetric_laplacian_tmp_2}) as upper bounds, the latter is simpler.
    
    For the spectral norm, based on the definition, we have
    \begin{align}
        \Vert L \Vert_{2} = \sup_{x \neq 0} \frac{ \vert Lx \vert_2 }{ \vert x \vert_2} = \sigma_{\max},
    \end{align}
    where $\sigma_{\max}$ is the maximum singular value of $L$.
    Since $L$ is symmetric, we have $\sigma_{i} = \vert \lambda_{i} \vert$ where $\lambda_{i}$ is the $i$-th eigenvalue of $L$.
    Hence, $\sigma_{\max} = \max_{i} \vert \lambda_{i} \vert$.
    From Raylaigh quotient and Courant–Fischer minimax theorem, we have 
    \begin{align}
        \Vert L \Vert_{2} & = \max_{i} \vert \lambda_{i} \vert = \max_{x \neq 0} \left\vert \frac{ x^{\top} L x }{ x^{\top} x} \right\vert \nonumber \\
        & = \max_{x \neq 0} \left\vert \frac{ \sum_{i=1}^{n} \sum_{j=1}^{n} L[i,j] x_i x_j }{ \sum_{i=1}^{n} x_i^2 } \right\vert \nonumber \\
        & = \max_{x \neq 0} \left\vert \frac{ \sum_{i=1}^{n} \sum_{j=1}^{n} {\tilde{A}[i,j] x_i x_j}/{\sqrt{D_{i}D_{j}}} }{ \sum_{i=1}^{n} x_i^2 } \right\vert \nonumber \\
        & = \max_{x \neq 0} \left\vert \frac{ \sum_{(i, j) \in \tilde{E}} {x_i x_j}/{\sqrt{D_{i}D_{j}}} }{ \sum_{i=1}^{n} x_i^2 } \right\vert \nonumber \\
        & \le \max_{x \neq 0} \left\vert \frac{ \frac{1}{2} \sum_{(i, j) \in \tilde{E}} \left( {x_i^2}/{D_{i}} + {x_j^2}/{D_{j}} \right) }{ \sum_{i=1}^{n} x_i^2 } \right\vert \nonumber \\
        & = \max_{x \neq 0} \left\vert \frac{ \sum_{(i, j) \in \tilde{E}} {x_i^2}/{D_{i}} }{ \sum_{i=1}^{n} x_i^2 } \right\vert 
        = \max_{x \neq 0} \left\vert \frac{ \sum_{i=1}^{n} {x_i^2}}{ \sum_{i=1}^{n} x_i^2 } \right\vert = 1,
    \end{align}    
    where $\tilde{E}$ is the union of the set of edges $E$ in the original graph and the set of self-loops.
    For Frobenius norm, we have $\Vert L \Vert_{F} \le \sqrt{r} \Vert L \Vert_2 \le \sqrt{r}$ where $r$ is the rank of $L$.
\end{proof}

%% file: proof/gcn_proof.tex
\subsection{GCN Results}

In this part, we provide the proofs of the main results regarding GCNs.

\begin{replemma}{lemma:gcn_perturbation}
    (GCN Perturbation Bound) For any $B > 0, l > 1$, let $f_w \in \mathcal{H}: \mathcal{X} \times \mathcal{G} \rightarrow \mathbb{R}^{K}$ be a $l$-layer GCN. Then for any $w$, and $x \in \mathcal{X}_{B,h_0}$, and any perturbation $u = \text{vec}( \{U_i\}_{i=1}^{l})$ such that $\forall i \in \mathbb{N}^{+}_{l}$, $\Vert U_i \Vert_{2} \le \frac{1}{l} \Vert W_i \Vert_{2}$, the change in the output of GCN is bounded as,
    \begin{align}
        \left\vert f_{w+u}(X, A) - f_w(X, A) \right\vert_{2} \le eB d^{\frac{l-1}{2}} \left( \prod_{i=1}^{l} \Vert W_{i} \Vert_{2} \right) \sum_{k=1}^{l} \frac{\Vert U_{k} \Vert_{2}}{\Vert W_{k} \Vert_{2}} \nonumber
    \end{align}
\end{replemma}

\begin{proof}
    We first perform the recursive perturbation analysis on node representations of all layers except the last one, \ie, the readout layer.
    Then we derive the bound for the graph representation of the last readout layer.
    
    \paragraph{Perturbation Analysis on Node Representations.}
    In GCN, for any layer $j < l$ besides the last readout one, the node representations are,
    \begin{align}
        f_{w}^{j}(X, A) = H_{j} = \sigma_{j} \left( \tilde{L} H_{j-1} W_{j} \right).
    \end{align}
    We add perturbation $u$ to the weights $w$, \ie, for the $j$-th layer, the perturbed weights are $W_{j} + U_{j}$.
    For the ease of notation, we use the superscript of prime to denote the perturbed node representations, \eg, $H_{j}^{\prime} = f_{w + u}^{j}(X, A)$.
    Let $\Delta_j = f_{w + u}^{j}(X, A) - f_{w}^{j}(X, A) = H_{j}^{\prime} - H_{j}$. 
    Note that $\Delta_j \in \mathbb{R}^{n \times h_j}$.
    Let $\Psi_{j} = \max\limits_{i} \left\vert \Delta_j[i, :] \right\vert_{2} = \max\limits_{i} \left\vert H_j^{\prime}[i, :] - H_j[i, :] \right\vert_{2}$ and $\Phi_j = \max\limits_{i} \left\vert H_j[i, :] \right\vert_{2}$.
    We denote the $u_j^{\ast} = \argmax\limits_{i} \left\vert \Delta_j[i, :] \right\vert_{2}$ and $v_j^{\ast} = \argmax\limits_{i} \left\vert H_j[i, :] \right\vert_{2}$.

    \paragraph{Upper Bound on the Max Node Representation}
    For any layer $j < l$, we can derive an upper bound on the maximum (\wrt $\ell_2$ norm) node representation as follows,
    \begin{align}\label{eq:gcn_output_bound}
        \Phi_j & = \max_{i} \left\vert H_j[i, :] \right\vert_{2} = \left\vert \left( \sigma_j \left( \tilde{L} H_{j-1} W_{j} \right) \right) [v_{j}^{\ast}, :] \right\vert_{2} = \left\vert \sigma_j \left( \left( \tilde{L} H_{j-1} W_{j} \right) [v_{j}^{\ast}, :] \right) \right\vert_{2} \nonumber \\
        & \le \left\vert \left( \tilde{L} H_{j-1} W_{j} \right) [v_{j}^{\ast}, :] \right\vert_{2} \qquad (\text{Lipschitz property of ReLU under vector 2-norm}) \nonumber \\
        & = \left\vert \left( \tilde{L} H_{j-1} \right) [v_{j}^{\ast}, :] W_{j} \right\vert_{2} \nonumber \\
        & \le \left\vert \left( \tilde{L} H_{j-1} \right) [v_{j}^{\ast}, :] \right\vert_{2} \left\Vert W_{j} \right\Vert_{2} = \left\vert \sum\nolimits_{k \in \mathcal{N}_{v_{j}^{\ast}}} \tilde{L}[v_{j}^{\ast}, k] H_{j-1}[k , :] \right\vert_{2} \left\Vert W_{j} \right\Vert_{2} \nonumber \\
        & \le \sum\nolimits_{k \in \mathcal{N}_{v_{j}^{\ast}}} \tilde{L}[v_{j}^{\ast}, k] \left\vert H_{j-1}[k , :] \right\vert_{2} \left\Vert W_{j} \right\Vert_{2} \nonumber \\
        & \le \sum\nolimits_{k \in \mathcal{N}_{v_{j}^{\ast}}} \tilde{L}[v_{j}^{\ast}, k]  \Phi_{j-1} \left\Vert W_{j} \right\Vert_{2} \qquad \left( \text{since } \forall i, \left\vert H_{j-1}[i, :] \right\vert_{2} \le \Phi_{j-1} \right) \nonumber \\
        & \le d^{\frac{1}{2}} \Phi_{j-1} \left\Vert W_{j} \right\Vert_{2} \nonumber \\
        & \le d^{\frac{j}{2}} \Phi_{0} \prod_{i=1}^{j} \Vert W_{i} \Vert_{2} \qquad (\text{unroll the recursion}) \nonumber \\
        & \le d^{\frac{j}{2}} B \prod_{i=1}^{j} \Vert W_{i} \Vert_{2}, 
    \end{align}
    where in the last inequality we use the fact $\Phi_0 = \max_{i} \left\vert X[i, :] \right\vert_{2} \le B$ based on the assumption \ref{asm:bound_input}.
    $\mathcal{N}_{v_{j}^{\ast}}$ is the set of neighboring nodes (including itself) of node $v_{j}^{\ast}$.
    In the third from the last inequality, we use the Lemma \ref{lemma:symmetric_laplacian_norm} to derive the following fact that $\forall i$,
    \begin{align}\label{eq:proof_gcn_perturbation_laplacian}
        \sum_{k \in \mathcal{N}_{i}} \tilde{L}[i, k] = \sum_{k \in \mathcal{N}_{i}} \left\vert \tilde{L}[i, k] \right\vert \le \left\Vert \tilde{L} \right\Vert_{\infty} \le \sqrt{d}. 
    \end{align}
    
    \paragraph{Upper Bound on the Max Change of Node Representation.}
    For any layer $j < l$, we can derive an upper bound on the maximum (\wrt $\ell_2$ norm) change between the representations with and without the weight perturbation for any node as follows,
    {\small
    \begin{align}\label{eq:gcn_difference_output_recursion}
        \Psi_{j} & = \max\limits_{i} \left\vert H_j^{\prime}[i, :] - H_j[i, :] \right\vert_{2} = \left \vert \sigma_{j} \left( \tilde{L} H_{j-1}^{\prime} (W_{j} + U_{j}) \right) [u_{j}^{\ast}, :] - \sigma_{j} \left( \tilde{L} H_{j-1} W_{j} \right) [u_{j}^{\ast}, :] \right \vert_{2}  \qquad \nonumber \\
        & \le \left \vert \left( \tilde{L} H_{j-1}^{\prime} (W_{j} + U_{j}) \right) [u_{j}^{\ast}, :] - \left( \tilde{L} H_{j-1} W_{j} \right) [u_{j}^{\ast}, :] \right \vert_{2} \qquad (\text{Lipschitz property of ReLU}) \nonumber \\
        & = \left\vert \left( (\tilde{L} H_{j-1}^{\prime}) [u_{j}^{\ast}, :] \right) (W_{j} + U_{j}) -  \left( (\tilde{L} H_{j-1} ) [u_{j}^{\ast}, :] \right) W_{j} \right\vert_{2} \nonumber \\
        & = \left\vert \left( \left( (\tilde{L} H_{j-1}^{\prime}) [u_{j}^{\ast}, :] \right) - \left( (\tilde{L} H_{j-1} ) [u_{j}^{\ast}, :] \right) \right) (W_{j} + U_{j}) + \left( (\tilde{L} H_{j-1} ) [u_{j}^{\ast}, :] \right) U_{j}  \right\vert_{2} \nonumber \\
        & = \left\vert \left( \sum_{k \in \mathcal{N}_{u_{j}^{\ast}}} \tilde{L}[u_{j}^{\ast},k] \left(  H_{j-1}^{\prime} [k, :]
        - H_{j-1} [k, :] \right) \right) (W_{j} + U_{j}) + \left( \sum_{k \in \mathcal{N}_{u_{j}^{\ast}}} \tilde{L}[u_{j}^{\ast},k]  H_{j-1} [k, :] \right) U_{j}  \right\vert_{2} \nonumber \\
        & \le \left\vert \sum\nolimits_{k \in \mathcal{N}_{u_{j}^{\ast}}} \tilde{L}[u_{j}^{\ast},k]  \left( H_{j-1}^{\prime} [k, :] - H_{j-1} [k, :] \right) \right\vert_2 \left\Vert W_{j} + U_{j} \right\Vert_2 + \left\vert \sum\nolimits_{k \in \mathcal{N}_{u_{j}^{\ast}}} \tilde{L}[u_{j}^{\ast},k]  H_{j-1} [k, :] \right\vert_2 \left\Vert U_{j} \right\Vert_{2} \nonumber \\
        & \le \sum\nolimits_{k \in \mathcal{N}_{u_{j}^{\ast}}} \tilde{L}[u_{j}^{\ast},k]  \left\vert H_{j-1}^{\prime} [k, :] - H_{j-1} [k, :] \right\vert_2 \left\Vert W_{j} + U_{j} \right\Vert_2 + \sum\nolimits_{k \in \mathcal{N}_{u_{j}^{\ast}}} \tilde{L}[u_{j}^{\ast},k]  \left\vert H_{j-1} [k, :] \right\vert_2 \left\Vert U_{j} \right\Vert_{2} \nonumber \\
        & \le \sum\nolimits_{k \in \mathcal{N}_{u_{j}^{\ast}}} \tilde{L}[u_{j}^{\ast},k]  \Psi_{j-1} \left\Vert W_{j} + U_{j} \right\Vert_2 + \sum\nolimits_{k \in \mathcal{N}_{u_{j}^{\ast}}} \tilde{L}[u_{j}^{\ast},k]  \Phi_{j-1} \left\Vert U_{j} \right\Vert_{2} \nonumber \\
        & \le \sqrt{d} \Psi_{j-1} \left\Vert W_{j} + U_{j} \right\Vert_2 + \sqrt{d} \Phi_{j-1} \left\Vert U_{j} \right\Vert_{2},
    \end{align}}%
    where in the second from the last inequality we use the fact $\forall k$, $\left\vert H_{j-1}^{\prime}[k, :] - H_{j-1}[k, :] \right\vert_{2} \le \Psi_{j-1}$ and $\forall k$, $\left\vert H_{j-1}[k, :] \right\vert_{2} \le \Phi_{j-1}$.
    In the last inequality, we again use the fact in Eq. (\ref{eq:proof_gcn_perturbation_laplacian}).
    We can simplify the notations in Eq. (\ref{eq:gcn_difference_output_recursion}) as $\Psi_{j} \le a_{j-1} \Psi_{j-1} + b_{j-1}$ where $a_{j-1} = \sqrt{d} \Vert W_{j} + U_{j} \Vert_{2}$ and $b_{j-1} = \sqrt{d} \Phi_{j-1} \Vert U_{j} \Vert_{2}$.
    Since $\Delta_0 = X - X = \bm{0}$, we have $\Psi_0 = 0$.
    It is straightforward to work out the recursion as,
    \begin{align}\label{eq:gcn_difference_output_recursion_general}
        \Psi_{j} & \le \sum_{k=0}^{j-1} b_k \left( \prod_{i=k+1}^{j-1} a_i \right) = \sum_{k=0}^{j-1} d^{\frac{1}{2}} \Phi_{k} \Vert U_{k+1} \Vert_{2} \left( \prod_{i=k+1}^{j-1} d^{\frac{1}{2}} \Vert W_{i+1} + U_{i+1} \Vert_{2} \right) \nonumber \\
        & = \sum_{k=0}^{j-1} d^{\frac{j-k}{2}} \Phi_{k} \Vert U_{k+1} \Vert_{2} \left( \prod_{i=k+2}^{j} \Vert W_{i} + U_{i} \Vert_{2} \right).
    \end{align}
    
    Based on Eq. (\ref{eq:gcn_output_bound}), we can instantiate the bound in Eq. (\ref{eq:gcn_difference_output_recursion_general}) as
    \begin{align}\label{eq:gcn_difference_output_final}
        \Psi_{j} & \le \sum_{k=0}^{j-1} d^{\frac{j-k}{2}} \Phi_{k} \Vert U_{k+1} \Vert_{2} \left( \prod_{i=k+2}^{j} \Vert W_{i} + U_{i} \Vert_{2} \right) \nonumber \\
        & \le \sum_{k=0}^{j-1} d^{\frac{j-k}{2}} \left( d^{\frac{k}{2}} B \prod_{i=1}^{k} \Vert W_{i} \Vert_{2} \right) \Vert U_{k+1} \Vert_{2} \left( \prod_{i=k+2}^{j} \left( \Vert W_{i} \Vert_{2} + \Vert U_{i} \Vert_{2} \right) \right) \nonumber \\
        & \le B \sum_{k=0}^{j-1} d^{\frac{j}{2}} \left( \prod_{i=1}^{k} \Vert W_{i} \Vert_{2} \right) \Vert U_{k+1} \Vert_{2} \left( \prod_{i=k+2}^{j} \left( 1 + \frac{1}{l} \right) \Vert W_{i} \Vert_{2} \right)  \nonumber \\
        & = B \sum_{k=0}^{j-1} d^{\frac{j}{2}} \left( \prod_{i=1}^{k+1} \Vert W_{i} \Vert_{2} \right) \frac{\Vert U_{k+1} \Vert_{2}}{\Vert W_{k+1} \Vert_{2}} \left( \prod_{i=k+2}^{j} \left( 1 + \frac{1}{l} \right) \Vert W_{i} \Vert_{2} \right)  \nonumber \\
        & = B d^{\frac{j}{2}} \left( \prod_{i=1}^{j} \Vert W_{i} \Vert_{2} \right) \sum_{k=0}^{j-1} \frac{\Vert U_{k+1} \Vert_{2}}{\Vert W_{k+1} \Vert_{2}} \left( 1 + \frac{1}{l} \right)^{j-k-1} \nonumber \\
        & \le B d^{\frac{j}{2}} \left( \prod_{i=1}^{j} \Vert W_{i} \Vert_{2} \right) \sum_{k=1}^{j} \frac{\Vert U_{k} \Vert_{2}}{\Vert W_{k} \Vert_{2}} \left( 1 + \frac{1}{l} \right)^{j-k}
    \end{align}
    
    \paragraph{Final Bound on the Readout Layer}
    Now let us consider the average readout function in the last layer, \ie, the $l$-th layer.
    Based on Eq. (\ref{eq:gcn_output_bound}) and Eq. (\ref{eq:gcn_difference_output_final}), we can bound the change of GCN's output with and without the weight perturbation as follows,
    {\normalsize
    \begin{align}
        \vert \Delta_l \vert_{2} = & \left\vert \frac{1}{n} \bm{1}_n H_{l-1}^{\prime} (W_{l} + U_{l})  - \frac{1}{n} \bm{1}_n H_{l-1} W_{l} \right\vert_{2} \nonumber \\
        = & \left \vert \frac{1}{n} \bm{1}_n \Delta_{l-1} (W_l + U_l) + \frac{1}{n} \bm{1}_n H_{l-1} U_l \right \vert_{2} \nonumber \\
        \le & \frac{1}{n} \left\vert \bm{1}_n \Delta_{l-1} (W_l + U_l) \right\vert_{2} + \frac{1}{n} \left\vert \bm{1}_n H_{l-1} U_l \right\vert_{2} \nonumber \\
        \le & \frac{1}{n} \Vert W_l + U_l \Vert_{2} \vert \bm{1}_n \Delta_{l-1} \vert_{2} + \frac{1}{n} \Vert U_l \Vert_{2} \vert \bm{1}_n H_{l-1} \vert_{2}  \nonumber \\
        = & \frac{1}{n} \left\Vert W_l + U_l \right\Vert_{2} \left\vert \sum_{i=1}^{n} \Delta_{l-1}[i, :] \right\vert_{2} + \frac{1}{n} \left\Vert U_l \right\Vert_{2} \left\vert \sum_{i=1}^{n} H_{l-1}[i, :] \right\vert_{2}  \nonumber \\
        \le & \frac{1}{n} \left\Vert W_l + U_l \right\Vert_{2} \left( \sum_{i=1}^{n} \left\vert \Delta_{l-1}[i, :] \right\vert_{2} \right) + \frac{1}{n} \left\Vert U_l \right\Vert_{2} \left( \sum_{i=1}^{n} \left\vert H_{l-1}[i, :] \right\vert_{2} \right) \nonumber \\
        \le & \left\Vert W_l + U_l \right\Vert_{2} \Psi_{l-1} + \left\Vert U_l \right\Vert_{2} \Phi_{l-1} \nonumber \\
        \le & \left\Vert W_l + U_l \right\Vert_{2} B d^{\frac{l-1}{2}} \left( \prod_{i=1}^{l-1} \Vert W_{i} \Vert_{2} \right) \sum_{k=1}^{l-1} \frac{\Vert U_{k} \Vert_{2}}{\Vert W_{k} \Vert_{2}} \left( 1 + \frac{1}{l} \right)^{l-1-k} + \left\Vert U_l \right\Vert_{2} B d^{\frac{l-1}{2}} \prod_{i=1}^{l-1} \Vert W_{i} \Vert_{2} \nonumber \\
        = & B d^{\frac{l-1}{2}} \left[ \left\Vert W_l + U_l \right\Vert_{2} \left( \prod_{i=1}^{l-1} \Vert W_{i} \Vert_{2} \right) \sum_{k=1}^{l-1} \frac{\Vert U_{k} \Vert_{2}}{\Vert W_{k} \Vert_{2}} \left( 1 + \frac{1}{l} \right)^{l-1-k} + \left\Vert U_l \right\Vert_{2} \prod_{i=1}^{l-1} \Vert W_{i} \Vert_{2} \right] \nonumber \\
        = & B d^{\frac{l-1}{2}} \left( \prod_{i=1}^{l} \Vert W_{i} \Vert_{2} \right) \left[ \frac{\left\Vert W_l + U_l \right\Vert_{2}}{\left\Vert W_l \right\Vert_{2}} \sum_{k=1}^{l-1} \frac{\Vert U_{k} \Vert_{2}}{\Vert W_{k} \Vert_{2}} \left( 1 + \frac{1}{l} \right)^{l-1-k} + \frac{\left\Vert U_l \right\Vert_{2}}{\left\Vert W_l \right\Vert_{2}} \right] \nonumber \\
        \le & B d^{\frac{l-1}{2}} \left( \prod_{i=1}^{l} \Vert W_{i} \Vert_{2} \right) \left[ \left( 1 + \frac{1}{l} \right) \sum_{k=1}^{l-1} \frac{\Vert U_{k} \Vert_{2}}{\Vert W_{k} \Vert_{2}} \left( 1 + \frac{1}{l} \right)^{l-1-k} + \frac{\left\Vert U_l \right\Vert_{2}}{\left\Vert W_l \right\Vert_{2}} \right] \nonumber \\
        = & B d^{\frac{l-1}{2}} \left( \prod_{i=1}^{l} \Vert W_{i} \Vert_{2} \right) \left( 1 + \frac{1}{l} \right)^{l} \left[ \sum_{k=1}^{l-1} \frac{\Vert U_{k} \Vert_{2}}{\Vert W_{k} \Vert_{2}} \left( 1 + \frac{1}{l} \right)^{-k} + \frac{\left\Vert U_l \right\Vert_{2}}{\left\Vert W_l \right\Vert_{2}} \left( 1 + \frac{1}{l} \right)^{-l} \right] \nonumber \\
        \le & e B d^{\frac{l-1}{2}} \left( \prod_{i=1}^{l} \Vert W_{i} \Vert_{2} \right) \left[ \sum_{k=1}^{l} \frac{\Vert U_{k} \Vert_{2}}{\Vert W_{k} \Vert_{2}} \right] \qquad \left( \text{Use } 1 \le \left(1 + \frac{1}{l} \right)^l \le e \right)
    \end{align}
    }
    which proves the lemma.
\end{proof}

\begin{reptheorem}{thm:gcn_generalization_bound}
    (GCN Generalization Bound) For any $B > 0, l > 1$, let $f_w \in \mathcal{H}: \mathcal{X} \times \mathcal{G} \rightarrow \mathbb{R}^{K}$ be a $l$-layer GCN. Then for any $\delta, \gamma > 0$, with probability at least $1 - \delta$ over the choice of an i.i.d. size-$m$ training set $S$ according to $\mathcal{D}$, for any $w$, we have,
    \begin{align}
    	L_{\mathcal{D}, 0}(f_w) \le L_{S, \gamma}(f_w) + \mathcal{O} \left( \sqrt{\frac{ B^2 d^{l-1} l^2 h \log(lh) \prod\limits_{i=1}^{l} \Vert W_{i} \Vert_{2}^2 \sum\limits_{i=1}^{l} \frac{\Vert W_{i} \Vert_F^2}{\Vert W_{i} \Vert_{2}^2} + \log \frac{ml}{\delta} }{\gamma^2 m}} \right) \nonumber
    \end{align}
\end{reptheorem}

\begin{proof}
    Let $\beta = \left( \prod_{i=1}^{l} \Vert W_i \Vert_{2} \right)^{1/l}$.
    We normalize the weights as $\tilde{W}_i = \frac{\beta}{\Vert W_i \Vert_{2}} W_i$.
    Due to the homogeneity of ReLU, \ie, $a\phi(x) = \phi(ax)$, $\forall a \ge 0$, we have $f_w = f_{\tilde{w}}$.
    We can also verify that $\prod_{i=1}^{l} \Vert W_i \Vert_{2} = \prod_{i=1}^{l} \Vert \tilde{W}_i \Vert_{2}$ and $\Vert W_i \Vert_{{F}} / \Vert W_i \Vert_{2} = \Vert \tilde{W}_i \Vert_{{F}} / \Vert \tilde{W}_i \Vert_{2}$, \ie, the terms appear in the bound stay the same after applying the normalization.
    Therefore, w.l.o.g., we assume that the norm is equal across layers, \ie, $\forall i$, $\Vert W_i \Vert_{2} = \beta$.
    
    Consider the prior $P = \mathcal{N}(0, \sigma^2 I)$ and the random perturbation $u \sim \mathcal{N}(0, \sigma^2 I)$.
    Note that the $\sigma$ of the prior and the perturbation are the same and will be set according to $\beta$.
    More precisely, we will set the $\sigma$ based on some approximation $\tilde{\beta}$ of $\beta$ since the prior $P$ can not depend on any learned weights directly.
    The approximation $\tilde{\beta}$ is chosen to be a cover set which covers the meaningful range of $\beta$. 
    For now, let us assume that we have a fix $\tilde{\beta}$ and consider $\beta$ which satisfies $\vert \beta - \tilde{\beta} \vert \le \frac{1}{l} \beta$.
    Note that this also implies
    \begin{align}\label{eq:beta_interval}
        \vert \beta - \tilde{\beta} \vert \le \frac{1}{l} \beta & ~\Rightarrow~ \left(1 - \frac{1}{l} \right) \beta \le \tilde{\beta} \le \left(1 + \frac{1}{l} \right) \beta \nonumber \\ 
        & ~\Rightarrow~ \left(1 - \frac{1}{l} \right)^{l-1} \beta^{l-1} \le \tilde{\beta}^{l-1} \le \left(1 + \frac{1}{l} \right)^{l-1} \beta^{l-1}  \nonumber \\
        & ~\Rightarrow~ \left(1 - \frac{1}{l} \right)^{l} \beta^{l-1} \le \tilde{\beta}^{l-1} \le \left(1 + \frac{1}{l} \right)^{l} \beta^{l-1}  \nonumber \\
        & ~\Rightarrow~ \frac{1}{e} \beta^{l-1} \le \tilde{\beta}^{l-1} \le e \beta^{l-1}  
    \end{align}

    From \cite{tropp2012user}, for $U_i \in \mathbb{R}^{h \times h}$ and $U_i \sim \mathcal{N}(\bm{0}, \sigma^2 I)$, we have,
    \begin{align}
        \mathbb{P}\left( \Vert U_i \Vert_{2} \ge t \right) \le 2he^{-t^2 /2h\sigma^2}.
    \end{align}
    Taking a union bound, we have 
    \begin{align}\label{eq:gcn_generalization_bound_proof_tmp_1}
        \mathbb{P}\left( \Vert U_1 \Vert_{2} < t ~\&~ \cdots ~\&~ \Vert U_l \Vert_{2} < t \right) & = 1 - \mathbb{P}\left( \exists i, \Vert U_i \Vert_{2} \ge t \right) \nonumber \\
        & \ge 1 - \sum_{i=1}^{l} \mathbb{P}\left( \Vert U_i \Vert_{2} \ge t \right) \nonumber \\
        & \ge 1 - 2lhe^{-t^2 /2h\sigma^2}.
    \end{align}
    Setting $2lhe^{-t^2 /2h\sigma^2} = \frac{1}{2}$, we have $t = \sigma \sqrt{2h\log(4lh)}$. 
    This implies that the probability that the spectral norm of the perturbation of any layer is no larger than $\sigma \sqrt{2h\log(4lh)}$ holds with probability at least $\frac{1}{2}$.
    Plugging this bound into Lemma \ref{lemma:gcn_perturbation}, we have with probability at least $\frac{1}{2}$,
    \begin{align}
        \left\vert f_{w+u}(X, A) - f_w(X, A) \right\vert_{2} & \le eB d^{\frac{l-1}{2}} \left( \prod_{i=1}^{l} \Vert W_{i} \Vert_{2} \right) \sum_{k=1}^{l} \frac{\Vert U_{k} \Vert_{2}}{\Vert W_{k} \Vert_{2}} \nonumber \\
        & = eB d^{\frac{l-1}{2}} \beta^l \sum_{k=1}^{l} \frac{\Vert U_{k} \Vert_{2}}{\beta} \nonumber \\
        & \le eB d^{\frac{l-1}{2}} \beta^{l-1} l\sigma \sqrt{2h\log(4lh)} \nonumber \\
        & \le e^2 B d^{\frac{l-1}{2}} \tilde{\beta}^{l-1} l\sigma \sqrt{2h\log(4lh)} \le \frac{\gamma}{4},
    \end{align}
    where we can set $\sigma = \frac{\gamma}{42 B d^{\frac{l-1}{2}} \tilde{\beta}^{l-1} l \sqrt{h\log(4lh)}}$ to get the last inequality.
    Note that Lemma \ref{lemma:gcn_perturbation} also requires $\forall i \in \mathbb{N}^{+}_{l}$, $\Vert U_i \Vert_{2} \le \frac{1}{l} \Vert W_i \Vert_{2}$.
    The requirement is satisfied if $\sigma \le \frac{\beta}{l \sqrt{2h \log(4lh)}}$ which in turn can be satisfied if 
    \begin{align}\label{eq:gcn_condition_perturbation_lemma}
        \frac{\gamma}{4e B d^{\frac{l-1}{2}} \beta^{l-1} l \sqrt{2h\log(4lh)}} \le \frac{\beta}{l \sqrt{2h \log(4lh)}},
    \end{align} 
    since the chosen value of $\sigma$ satisfies $\sigma \le \frac{\gamma}{4e B d^{\frac{l-1}{2}} \beta^{l-1} l \sqrt{2h\log(4lh)}}$.
    Note that Eq. (\ref{eq:gcn_condition_perturbation_lemma}) is equivalent to $\frac{\gamma}{4eB} d^{\frac{1-l}{2}} \le \beta^{l}$.
    We will see how to satisfy this condition later.
    
    We now compute the KL term in the PAC-Bayes bound in Lemma \ref{lemma:pac_bayes_deterministic}.
    \begin{align}
        \text{KL}\left( Q \Vert P \right) & = \frac{\vert w \vert_2^2}{2 \sigma^2} 
        = \frac{42^2 B^2 d^{l-1} \tilde{\beta}^{2l-2} l^2 h \log(4lh) }{2\gamma^2} \sum_{i=1}^{l} \Vert W_i \Vert_F^2 \nonumber \\
        & \le \mathcal{O}\left( \frac{B^2 d^{l-1} \beta^{2l} l^2 h \log(lh) }{\gamma^2} \sum_{i=1}^{l} \frac{\Vert W_i \Vert_F^2}{\beta^2} \right) \nonumber \\
        & \le \mathcal{O}\left( B^2 d^{l-1} l^2 h \log(lh) \frac{\prod_{i=1}^{l} \Vert W_i \Vert_{2}^2}{\gamma^2} \sum_{i=1}^{l} \frac{\Vert W_i \Vert_F^2}{\Vert W_i \Vert_{2}^2} \right).
    \end{align}
    From Lemma \ref{lemma:pac_bayes_deterministic}, fixing any $\tilde{\beta}$, with probability $1 - \delta$ and for all $w$ such that $\vert \beta - \tilde{\beta} \vert \le \frac{1}{l}\beta$, we have,
    \begin{align}\label{eq:gcn_generalization_bound_proof_tmp_2}
    	L_{\mathcal{D}, 0}(f_w) \le L_{S, \gamma}(f_w) + \mathcal{O} \left( \sqrt{\frac{ B^2 d^{l-1} l^2 h \log(lh) \prod\limits_{i=1}^{l} \Vert W_{i} \Vert_{2}^2 \sum\limits_{i=1}^{l} \frac{\Vert W_{i} \Vert_F^2}{\Vert W_{i} \Vert_{2}^2} + \log \frac{m}{\delta} }{\gamma^2 m}} \right).
    \end{align}
    Finally, we need to consider multiple choices of $\tilde{\beta}$ so that for any $\beta$, we can bound the generalization error like Eq. (\ref{eq:gcn_generalization_bound_proof_tmp_2}).
    First, we only need to consider values of $\beta$ in the following range,
    \begin{align}\label{eq:gcn_generalization_bound_proof_tmp_3}
        \frac{1}{\sqrt{d}} \left( \frac{\gamma \sqrt{d}}{2B} \right)^{1/l} \le \beta \le \frac{1}{\sqrt{d}} \left( \frac{\gamma \sqrt{md}}{2B} \right)^{1/l},
    \end{align}
    since otherwise the bound holds trivially as $L_{\mathcal{D}, 0}(f_w) \le 1$ by definition.
    Note that the lower bound in Eq. (\ref{eq:gcn_generalization_bound_proof_tmp_3}) ensures that Eq. (\ref{eq:gcn_condition_perturbation_lemma}) holds which in turn justifies the applicability of Lemma \ref{lemma:gcn_perturbation}.
    If $\beta < \frac{1}{\sqrt{d}} \left( \frac{\gamma \sqrt{d}}{2B} \right)^{1/l}$, then for any $(X, A)$ and any $j \in \mathbb{N}^{+}_{K}$, $\vert f(X, A)[j] \vert \le \frac{\gamma}{2}$.
    To see this, we have,
    \begin{align}
        \left\vert f_{w}(X, A)[j] \right\vert & \le \left\vert f_{w}(X, A) \right\vert_{2} = \vert \frac{1}{n} \bm{1}_n H_{l-1} W_l \vert_{2} \nonumber \\
        & \le \frac{1}{n} \vert \bm{1}_n H_{l-1} \vert_{2} \Vert W_l \Vert_{2} \nonumber \\
        & \le \Vert W_l \Vert_{2} \max_{i} \vert H_{l-1}[i, :] \vert_2 \nonumber \\
        & \le B d^{\frac{l-1}{2}} \prod_{i=1}^{l} \Vert W_{i} \Vert_{2} =  d^{\frac{l-1}{2}} \beta^l B \qquad (\text{Use Eq. (\ref{eq:gcn_output_bound})}) \nonumber \\
        & = d^{\frac{l-1}{2}} B \frac{\gamma}{2B d^{\frac{l-1}{2}}} \le \frac{\gamma}{2}.
    \end{align}
    Therefore, by the definition in Eq. (\ref{eq:empirical_max_margin_loss}), we always have $L_{S,\gamma}(f_w) = 1$ when $\beta < \frac{1}{\sqrt{d}} \left( \frac{\gamma \sqrt{d}}{2B} \right)^{1/l}$.
    Alternatively, if $\beta > \frac{1}{\sqrt{d}}\left( \frac{\gamma \sqrt{md}}{2B} \right)^{1/l}$, the term inside the big-O notation in Eq. (\ref{eq:gcn_generalization_bound_proof_tmp_2}) would be,
    \begin{align}
    	\sqrt{\frac{ B^2 d^{l-1} l^2 h \log(lh) \prod\limits_{i=1}^{l} \Vert W_{i} \Vert_{2}^2 \sum\limits_{i=1}^{l} \frac{\Vert W_{i} \Vert_F^2}{\Vert W_{i} \Vert_{2}^2} + \log \frac{m}{\delta} }{\gamma^2 m}} & \ge \sqrt{ \frac{l^2 h \log(lh) }{4} \sum\limits_{i=1}^{l} \frac{\Vert W_{i} \Vert_F^2}{\Vert W_{i} \Vert_{2}^2}} \nonumber \\
    	& \ge \sqrt{ \frac{l^2 h \log(lh) }{4}} \ge 1, 
    \end{align}
    where we use the facts that $\Vert W_{i} \Vert_F \ge \Vert W_{i} \Vert_2$ and we typically choose $h \ge 2$ in practice and $l \ge 2$.
    
    Since we only need to consider $\beta$ in the range of Eq. (\ref{eq:gcn_generalization_bound_proof_tmp_3}), a sufficient condition to make $\vert \beta - \tilde{\beta} \vert \le \frac{1}{l}\beta$ hold would be $\vert \beta - \tilde{\beta} \vert \le \frac{1}{l\sqrt{d}} \left( \frac{\gamma \sqrt{d}}{2B} \right)^{1/l}$.
    Therefore, if we can find a covering of the interval in Eq. (\ref{eq:gcn_generalization_bound_proof_tmp_3}) with radius $\frac{1}{l\sqrt{d}} \left( \frac{\gamma \sqrt{d}}{2B} \right)^{1/l}$ and make sure bounds like Eq. (\ref{eq:gcn_generalization_bound_proof_tmp_2}) holds while $\tilde{\beta}$ takes all possible values from the covering, then we can get a bound which holds for all $\beta$.
    It is clear that we only need to consider a covering $C$ with size $\vert C \vert = \frac{l}{2}\left(m^{\frac{1}{2l}} - 1\right)$.
    Therefore, denoting the event of Eq. (\ref{eq:gcn_generalization_bound_proof_tmp_2}) with $\tilde{\beta}$ taking the $i$-th value of the covering as $E_i$, we have 
    \begin{align}
        \mathbb{P}\left( E_1 ~\&~ \cdots ~\&~ E_{\vert C \vert} \right) & = 1 - \mathbb{P}\left( \exists i, \bar{E}_i \right) \ge 1 - \sum_{i=1}^{\vert C \vert} \mathbb{P}\left( \bar{E}_i \right)  \ge 1 - \vert C \vert \delta.
    \end{align}
    Note $\bar{E}_i$ denotes the complement of $E_i$.
    Hence, with probability $1 - \delta$ and for all $w$, we have,
    \begin{align}
    	L_{\mathcal{D}, 0}(f_w) & \le L_{S, \gamma}(f_w) + \mathcal{O} \left( \sqrt{\frac{ B^2 d^{l-1} l^2 h \log(lh) \prod\limits_{i=1}^{l} \Vert W_{i} \Vert_{2}^2 \sum\limits_{i=1}^{l} \frac{\Vert W_{i} \Vert_F^2}{\Vert W_{i} \Vert_{2}^2} + \log \frac{m\vert C \vert}{\delta} }{\gamma^2 m}} \right) \nonumber \\
    	& = L_{S, \gamma}(f_w) + \mathcal{O} \left( \sqrt{\frac{ B^2 d^{l-1} l^2 h \log(lh) \prod\limits_{i=1}^{l} \Vert W_{i} \Vert_{2}^2 \sum\limits_{i=1}^{l} \frac{\Vert W_{i} \Vert_F^2}{\Vert W_{i} \Vert_{2}^2} + \log \frac{ml}{\delta} }{\gamma^2 m}} \right),
    \end{align}
    which proves the theorem.
\end{proof}

%% file: proof/mpgnn_proof.tex
\subsection{MPGNNs Results}

In this part, we provide the proofs of the main results regarding MPGNNs.

\begin{replemma}{lemma:mpgnn_perturbation}
    (MPGNN Perturbation Bound) For any $B > 0, l > 1$, let $f_w \in \mathcal{H}: \mathcal{X} \times \mathcal{G} \rightarrow \mathbb{R}^{K}$ be a $l$-step MPGNN. 
    % Let $\lambda = C_{\phi} B \Vert W_1 \Vert_{2}$ and $\tau = d \mathcal{C} \Vert W_2 \Vert_{2}$.
    Then for any $w$, and $x \in \mathcal{X}_{B,h_0}$, and any perturbation $u = \text{vec}(\{U_1, U_2, U_l\})$ such that $\eta = \max \left( \frac{\Vert U_1 \Vert_{2}}{\Vert W_1 \Vert_{2}}, \frac{\Vert U_2 \Vert_{2}}{\Vert W_2 \Vert_{2}}, \frac{\Vert U_l \Vert_{2}}{\Vert W_l \Vert_{2}} \right) \le \frac{1}{l}$, the change in the output of MPGNN is bounded as,
    \begin{align}
        \vert f_{w+u}(& X, A) - f_w(X, A) \vert_{2} \le 
        \begin{cases}
            eB \left( l + 1 \right)^2 \eta \Vert W_1 \Vert_{2} \Vert W_l \Vert_{2} C_{\phi}, & \text{if}\ d \mathcal{C} = 1 \\
            eBl \eta \Vert W_1 \Vert_{2} \Vert W_l \Vert_{2} C_{\phi} \frac{\left( d \mathcal{C} \right)^{l-1} - 1}{d \mathcal{C} - 1}, & \text{otherwise}
        \end{cases} \nonumber
    \end{align}
    where $\mathcal{C} = C_{\phi} C_{\rho} C_{g} \Vert W_2 \Vert_{2}$.
\end{replemma}

\begin{proof}
    We first perform the recursive perturbation analysis on node representations of all steps except the last one, i.e., the readout step. Then we derive the bound for the graph representation of the last readout step.

    \paragraph{Perturbation Analysis on Node Representations.}

    In message passing GNNs, for any step $j < l$ besides the last readout one, the node representations are,
    \begin{align}
        \bar{M}_{j} & = C_{\operatorname{in}} g \left( C_{\operatorname{out}}^{\top} H_{j-1} \right) \nonumber \\
        H_{j} & = \phi\left( X W_1 + \rho \left( \bar{M}_{j} \right) W_2 \right), 
    \end{align}
    where the incidence matrices $C_{\operatorname{in}} \in \mathbb{R}^{n \times c}$ and $C_{\operatorname{out}} \in \mathbb{R}^{n \times c}$ (recall $c$ is the number of edges).
    Moreover, since each edge only connects one incoming and one outgoing node, we have,
    \begin{align}\label{eq:mpgnn_incidence_bound}
        \sum_{k=1}^{c} C_{\operatorname{in}} [i, k] & \le \max_{i} \sum_{k=1}^{c} C_{\operatorname{in}} [i, k] = \Vert C_{\operatorname{in}} \Vert_{\infty} \le d \nonumber \\
        \sum_{t=1}^{n} C_{\operatorname{out}} [t, k] & \le \max_{k} \sum_{t=1}^{n} C_{\operatorname{out}} [t, k] \le \Vert C_{\operatorname{out}} \Vert_{1} \le 1
    \end{align}
    where $d-1$ is the maximum node degree.
    Note that one actually has $\Vert C_{\operatorname{in}} \Vert_{\infty} \le d - 1$ for simple graphs.
    Since some models in the literature pre-process the graphs by adding self-loops, we thus relax it to $\Vert C_{\operatorname{in}} \Vert_{\infty} \le d$ which holds in both cases.
    
    We add perturbation $u$ to the weights $w$, \ie, the perturbed weights are $W_{1} + U_{1}$, $W_{2} + U_{2}$ and $W_{l} + U_{l}$.
    For the ease of notation, we use the superscript of prime to denote the perturbed node representations, \eg, $H_{j}^{\prime} = f_{w + u}^{j}(X, A)$.
    Let $\Delta_j = f_{w + u}^{j}(X, A) - f_{w}^{j}(X, A) = H_{j}^{\prime} - H_{j}$. 
    Note that $\Delta_j \in \mathbb{R}^{n \times h_j}$.
    Let $\Psi_{j} = \max\limits_{i} \left\vert \Delta_j[i, :] \right\vert_{2} = \max\limits_{i} \left\vert H_j^{\prime}[i, :] - H_j[i, :] \right\vert_{2}$ and $\Phi_j = \max\limits_{i} \left\vert H_j[i, :] \right\vert_{2}$.
    We denote the $u_j^{\ast} = \argmax\limits_{i} \left\vert \Delta_j[i, :] \right\vert_{2}$ and $v_j^{\ast} = \argmax\limits_{i} \left\vert H_j[i, :] \right\vert_{2}$.    
    To simplify the derivation, we abbreviate the following statistics $\kappa = C_{\phi} B \Vert W_1 \Vert_{2}$ and $\tau = d \mathcal{C}$ throughout the proof where $\mathcal{C} = C_{\phi} C_{\rho} C_{g} \Vert W_2 \Vert_{2}$ is the \emph{percolation complexity}.
    
    \paragraph{Upper Bound on the Max Node Representation.}
    For any step $j < l$, we can derive an upper bound on the $\ell_2$ norm of the aggregated message of any node $i$ as follows,
    \begin{align}\label{eq:mpgnn_msg_bound}
        \left\vert \bar{M}_j[i, :] \right\vert_{2} & = \left\vert \sum_{k=1}^{c} C_{\operatorname{in}} [i, k] \left( g \left( C_{\operatorname{out}}^{\top} H_{j-1} \right) \right) [k, :] \right\vert_{2} = \left\vert \sum_{k=1}^{c} C_{\operatorname{in}} [i, k] g \left( C_{\operatorname{out}}^{\top} [k, :] H_{j-1} \right) \right\vert_{2} \nonumber \\
        & \le \sum_{k=1}^{c} C_{\operatorname{in}} [i, k] \left\vert g \left( C_{\operatorname{out}}^{\top} [k, :] H_{j-1} \right) \right\vert_{2} \nonumber \\
        & \le \sum_{k=1}^{c} C_{\operatorname{in}} [i, k] C_g \left\vert C_{\operatorname{out}}^{\top} [k, :] H_{j-1} \right\vert_{2} = \sum_{k=1}^{c} C_{\operatorname{in}} [i, k] C_g \left\vert \sum_{t=1}^{n} C_{\operatorname{out}}^{\top} [k, t] H_{j-1} [t, :] \right\vert_{2} \nonumber \\
        & \le \sum_{k=1}^{c} C_{\operatorname{in}} [i, k] C_g \left( \sum_{t=1}^{n} C_{\operatorname{out}} [t, k]  \left\vert H_{j-1} [t, :] \right\vert_{2} \right) \nonumber \\
        & \le \sum_{k=1}^{c} C_{\operatorname{in}} [i, k] C_g \left( \sum_{t=1}^{n} C_{\operatorname{out}} [t, k]  \Phi_{j-1} \right) \nonumber \\
        & \le d C_g \Phi_{j-1}.
    \end{align}    
    Then we can derive an upper bound on the maximum (w.r.t. $\ell_2$ norm) node representation as follows,
    \begin{align}\label{eq:mpgnn_output_bound}
        \Phi_j & = \max\limits_{i} \left\vert H_j[i, :] \right\vert_{2} = \left\vert \phi\left( X W_1 + \rho \left( \bar{M}_{j} \right) W_2 \right) [v_j^{\ast}, :] \right\vert_{2} \nonumber \\
        & = \left\vert \phi\left( \left( X W_1 + \rho \left( \bar{M}_{j} \right) W_2 \right) [v_j^{\ast}, :] \right) \right\vert_{2} \nonumber \\
        & \le C_{\phi} \left\vert \left( X W_1 + \rho \left( \bar{M}_{j} \right) W_2 \right) [v_j^{\ast}, :] \right\vert_{2} \nonumber \\
        & = C_{\phi} \left\vert \left( X W_1 \right) [v_j^{\ast}, :] + \left( \rho \left( \bar{M}_{j} \right) W_2 \right) [v_j^{\ast}, :] \right\vert_{2} \nonumber \\
        & \le C_{\phi} \left\vert \left( X W_1 \right) [v_j^{\ast}, :] \right\vert_{2} + C_{\phi} \left\vert \left( \rho \left( \bar{M}_{j} \right) W_2 \right) [v_j^{\ast}, :] \right\vert_{2} \nonumber \\
        & = C_{\phi} \left\vert X [v_j^{\ast}, :] W_1 \right\vert_{2} + C_{\phi} \left\vert \rho \left( \bar{M}_{j} \right) [v_j^{\ast}, :] W_2 \right\vert_{2} \nonumber \\
        & \le C_{\phi} \left\vert X [v_j^{\ast}, :] \right\vert_{2} \Vert W_1 \Vert_{2} + C_{\phi} \left\vert \rho \left( \bar{M}_{j} \right) [v_j^{\ast}, :] \right\vert_{2} \Vert W_2 \Vert_{2} \nonumber \\
        & \le C_{\phi} B \Vert W_1 \Vert_{2} + C_{\phi} \left\vert \rho \left( \bar{M}_{j} [v_j^{\ast}, :] \right) \right\vert_{2} \Vert W_2 \Vert_{2} \nonumber \\
        & \le C_{\phi} B \Vert W_1 \Vert_{2} + C_{\phi} C_{\rho} \left\vert \bar{M}_{j} [v_j^{\ast}, :] \right\vert_{2} \Vert W_2 \Vert_{2} \nonumber \\
        & \le C_{\phi} B \Vert W_1 \Vert_{2} + d C_{\phi} C_{\rho} C_{g} \Phi_{j-1} \Vert W_2 \Vert_{2} = \kappa + \tau \Phi_{j-1} \nonumber \\
        & \le \tau^{j} \Phi_{0} + \sum_{i=0}^{j-1} \tau^{j-1-i} \kappa \qquad (\text{Unroll recursion}) \nonumber \\
        & = \sum_{i=0}^{j-1} \tau^{j-1-i} \kappa \qquad (\text{Use } \Phi_0 = 0) \nonumber \\
        & = \begin{cases}
                j \kappa, & \text{if}\ \tau = 1 \\
                \kappa \frac{\tau^{j} - 1}{\tau - 1}, & \text{otherwise}
            \end{cases}
    \end{align}
    
    \paragraph{Upper Bound on the Max Change of Node Representation.}
    For any step $j < l$, we can derive an upper bound on the maximum (w.r.t. $\ell_2$ norm) change between the aggregated message with and without the weight perturbation for any node $i$ as follows,

    \begin{align}\label{eq:mpgnn_msg_diff_bound}
        \left\vert \bar{M}_{j}^{\prime} [i, :] - \bar{M}_{j} [i, :] \right\vert_{2} & \le \left\vert \left( C_{\operatorname{in}} g \left( C_{\operatorname{out}}^{\top} H_{j-1}^{\prime} \right) \right) [i, :] - \left( C_{\operatorname{in}} g \left( C_{\operatorname{out}}^{\top} H_{j-1} \right) \right) [i, :] \right\vert_2 \nonumber \\
        & = \left\vert \sum_{k=1}^{c} C_{\operatorname{in}}[i, k] \left(g \left( C_{\operatorname{out}}^{\top} H_{j-1}^{\prime} \right) - g \left( C_{\operatorname{out}}^{\top} H_{j-1} \right) \right) [k, :] \right\vert_2 \nonumber \\
        & \le \sum_{k=1}^{c} C_{\operatorname{in}}[i, k] \left\vert \left(g \left( C_{\operatorname{out}}^{\top} H_{j-1}^{\prime} \right) - g \left( C_{\operatorname{out}}^{\top} H_{j-1} \right) \right) [k, :] \right\vert_2 \nonumber \\
        & = \sum_{k=1}^{c} C_{\operatorname{in}}[i, k] \left\vert g \left( \left( C_{\operatorname{out}}^{\top} H_{j-1}^{\prime} \right)[k, :] \right) - g \left( \left( C_{\operatorname{out}}^{\top} H_{j-1} \right) [k, :] \right) \right\vert_2 \nonumber \\
        & \le \sum_{k=1}^{c} C_{\operatorname{in}}[i, k] C_g \left\vert \left( C_{\operatorname{out}}^{\top} H_{j-1}^{\prime} \right)[k, :] - \left( C_{\operatorname{out}}^{\top} H_{j-1} \right) [k, :] \right\vert_2 \nonumber \\
        & = \sum_{k=1}^{c} C_{\operatorname{in}}[i, k] C_g \left\vert \sum_{t=1}^{n} C_{\operatorname{out}}[t, k] H_{j-1}^{\prime}[t, :] - \sum_{t=1}^{n} C_{\operatorname{out}}[t, k] H_{j-1}[t, :] \right\vert_2 \nonumber \\
        & = \sum_{k=1}^{c} C_{\operatorname{in}}[i, k] C_g \left\vert \sum_{t=1}^{n} C_{\operatorname{out}}[t, k] \left( H_{j-1}^{\prime}[t, :] - H_{j-1}[t, :] \right) \right\vert_2 \nonumber \\
        & \le \sum_{k=1}^{c} C_{\operatorname{in}}[i, k] C_g \left( \sum_{t=1}^{n} C_{\operatorname{out}}[t, k] \left\vert H_{j-1}^{\prime}[t, :] - H_{j-1}[t, :] \right\vert_2 \right) \nonumber \\
        & \le d C_g \Psi_{j-1}
    \end{align}
    Based on Eq. (\ref{eq:mpgnn_msg_diff_bound}), we can derive an upper bound on the maximum (w.r.t. $\ell_2$ norm) change between the representations with and without the weight perturbation for any node as follows,
    \begin{align}
        \Psi_{j} = & \max\limits_{i} \left\vert H_j^{\prime}[i, :] - H_j[i, :] \right\vert_{2}
        \nonumber \\
        = & \left\vert \left( \phi\left( X (W_1 + U_1) + \rho \left( \bar{M}_{j}^{\prime} \right) (W_2 + U_2) \right) \right) [u_j^{\ast}, :] - \left( \phi\left( X W_1 + \rho \left( \bar{M}_{j} \right) W_2 \right) \right) [u_j^{\ast}, :] \right\vert_{2} \nonumber \\
        = & \left\vert \phi\left( \left( X (W_1 + U_1) + \rho \left( \bar{M}_{j}^{\prime} \right) (W_2 + U_2) \right) [u_j^{\ast}, :] \right) - \phi\left( \left( X W_1 + \rho \left( \bar{M}_{j} \right) W_2 \right) [u_j^{\ast}, :] \right) \right\vert_{2} \nonumber \\
        \le & C_{\phi} \left\vert \left( X (W_1 + U_1) + \rho \left( \bar{M}_{j}^{\prime} \right) (W_2 + U_2) \right) [u_j^{\ast}, :] - \left( X W_1 + \rho \left( \bar{M}_{j} \right) W_2 \right) [u_j^{\ast}, :] \right\vert_{2} \nonumber \\
        \le & C_{\phi} \left\vert X [u_j^{\ast}, :] U_1 + \left( \rho \left( \bar{M}_{j}^{\prime} \right) \right) [u_j^{\ast}, :] (W_2 + U_2) - \left( \rho\left( \bar{M}_{j} \right) \right) [u_j^{\ast}, :] W_2 \right\vert_{2} \nonumber \\
        = & C_{\phi} \left\vert X [u_j^{\ast}, :] U_1 + \left( \rho \left( \bar{M}_{j}^{\prime} \right) - \rho \left( \bar{M}_{j} \right) \right) [u_j^{\ast}, :] (W_2 + U_2) + \rho \left( \bar{M}_{j} \right) [u_j^{\ast}, :] U_2 \right\vert_{2} \nonumber \\
        \le & C_{\phi} B \left\Vert U_1 \right\Vert_{2} + C_{\phi} C_{\rho} \left\vert \bar{M}_{j}^{\prime} [u_j^{\ast}, :] - \bar{M}_{j} [u_j^{\ast}, :] \right\vert_{2} \left\Vert W_2 + U_2 \right\Vert_{2} + C_{\phi} C_{\rho} \left\vert \bar{M}_{j} [u_j^{\ast}, :] \right\vert_{2} \left\Vert U_2 \right\Vert_{2} \nonumber \\
        \le & \kappa \frac{\Vert U_1 \Vert_{2}}{\Vert W_1 \Vert_{2}} + d \mathcal{C} \Psi_{j-1} \frac{\Vert W_2 + U_2 \Vert_{2}}{\Vert W_2 \Vert_{2}} + d \mathcal{C} \Phi_{j-1} \frac{\Vert U_2 \Vert_{2}}{\Vert W_2 \Vert_{2}}  \qquad (\text{Use Eq. (\ref{eq:mpgnn_msg_bound}) and (\ref{eq:mpgnn_msg_diff_bound})}) \nonumber \\
        \le & \tau \left( 1 + \frac{\Vert U_2 \Vert_{2}}{\Vert W_2 \Vert_{2}} \right) \Psi_{j-1} + \kappa \frac{\Vert U_1 \Vert_{2}}{\Vert W_1 \Vert_{2}} + \tau \Phi_{j-1} \frac{\Vert U_2 \Vert_{2}}{\Vert W_2 \Vert_{2}}
    \end{align}

    If $\tau = 1$, then we have,
    \begin{align}
        \Psi_{j} \le & \tau \left( 1 + \frac{\Vert U_2 \Vert_{2}}{\Vert W_2 \Vert_{2}} \right) \Psi_{j-1} + \kappa \frac{\Vert U_1 \Vert_{2}}{\Vert W_1 \Vert_{2}} + \tau \Phi_{j-1} \frac{\Vert U_2 \Vert_{2}}{\Vert W_2 \Vert_{2}} \nonumber \\
        \le & \left( 1 + \frac{\left\Vert U_2 \right\Vert_{2}}{\left\Vert W_2 \right\Vert_{2}} \right) \Psi_{j-1} + \kappa \left( \frac{\left\Vert U_1 \right\Vert_{2}}{\left\Vert W_1 \right\Vert_{2}} + \frac{\Vert U_2 \Vert_{2}}{\Vert W_2 \Vert_{2}} (j - 1) \right) \qquad (\text{Use Eq. (\ref{eq:mpgnn_output_bound})}) \nonumber \\
        \le & \left( 1 + \eta \right) \Psi_{j-1} + \kappa \eta \left( 1 +  (j - 1) \right) \qquad \left( \text{Use } \eta = \max \left( \frac{\Vert U_1 \Vert_{2}}{\Vert W_1 \Vert_{2}}, \frac{\Vert U_2 \Vert_{2}}{\Vert W_2 \Vert_{2}}, \frac{\Vert U_l \Vert_{2}}{\Vert W_l \Vert_{2}} \right) \right) \nonumber \\
        = & \left( 1 + \eta \right) \Psi_{j-1} + \kappa \eta j.
    \end{align}
    
    If $\tau \neq 1$, then we have,
    \begin{align}
        \Psi_{j} 
        \le & \tau \left( 1 + \frac{\Vert U_2 \Vert_{2}}{\Vert W_2 \Vert_{2}} \right) \Psi_{j-1} + \kappa \frac{\Vert U_1 \Vert_{2}}{\Vert W_1 \Vert_{2}} + \tau \Phi_{j-1} \frac{\Vert U_2 \Vert_{2}}{\Vert W_2 \Vert_{2}} \nonumber \\
        \le & \tau \left( 1 + \frac{\left\Vert U_2 \right\Vert_{2}}{\left\Vert W_2 \right\Vert_{2}} \right) \Psi_{j-1} + \kappa \left( \frac{\left\Vert U_1 \right\Vert_{2}}{\left\Vert W_1 \right\Vert_{2}} + \tau \frac{\left\Vert U_2 \right\Vert_{2}}{\left\Vert W_2 \right\Vert_{2}} \frac{\tau^{j-1} - 1}{\tau - 1} \right) \qquad (\text{Use Eq. (\ref{eq:mpgnn_output_bound})}) \nonumber \\
        \le & \tau \left( 1 + \eta \right) \Psi_{j-1} + \kappa \eta \left( 1 +  \frac{\tau^{j} - \tau}{\tau - 1} \right) \qquad \left( \text{Use } \eta = \max \left( \frac{\Vert U_1 \Vert_{2}}{\Vert W_1 \Vert_{2}}, \frac{\Vert U_2 \Vert_{2}}{\Vert W_2 \Vert_{2}}, \frac{\Vert U_l \Vert_{2}}{\Vert W_l \Vert_{2}} \right) \right) \nonumber \\
        = & \tau \left( 1 + \eta \right) \Psi_{j-1} + \kappa \eta \left( \frac{\tau^{j} - 1}{\tau - 1} \right).
    \end{align}

    Recall from Eq. (\ref{eq:gcn_difference_output_recursion_general}), if $\Psi_j \le a_{j-1} \Psi_{j-1} + b_{j-1}$ and $\Psi_0 = 0$, then $\Psi_j \le \sum_{k=0}^{j-1} b_k \left( \prod_{i=k+1}^{j-1} a_i \right)$.
    If $\tau = 1$, then we have $a_{j-1} = 1 + \eta$, $b_{j-1} = \kappa \eta j$ in our case and, 
    \begin{align}\label{eq:psi_value_tau_is_1}
        \Psi_j \le & \sum\nolimits_{k=0}^{j-1} b_k \left( \prod\nolimits_{i=k+1}^{j-1} a_i \right) = \sum\nolimits_{k=0}^{j-1} \kappa \eta (k+1) \left( 1 + \eta \right)^{j-k-1} \nonumber \\
        \le & \kappa \eta \left( 1 + \frac{1}{l} \right)^{j} \sum_{k=0}^{j-1} (k+1) \left( 1 + \frac{1}{l} \right)^{-k-1} \qquad \left( \text{Use } \eta \le \frac{1}{l} \right) \nonumber \\
        = & \kappa \eta \left( 1 + \frac{1}{l} \right)^{j} \sum_{k=1}^{j} k \left( 1 + \frac{1}{l} \right)^{-k} \nonumber \\
        = & \kappa \eta \left( 1 + \frac{1}{l} \right)^{j} \left( 1 + \frac{1}{l} \right)^{-1} \frac{1 - (j+1) \left( 1 + \frac{1}{l} \right)^{-j} + j \left( 1 + \frac{1}{l} \right)^{-j-1}}{\left( 1 - \left( 1 + \frac{1}{l} \right)^{-1} \right)^2} \nonumber \\
        = & \kappa \eta \frac{\left( 1 + \frac{1}{l} \right)^{j+1} - (j + 1) \left( 1 + \frac{1}{l} \right) + j }{\left( \left( 1 + \frac{1}{l} \right)^{1} - 1 \right)^2} \nonumber \\
        = & \kappa \eta l^2 \left( \left( 1 + \frac{1}{l} \right)^{j+1} - (j + 1) \left( 1 + \frac{1}{l} \right) + j \right) \nonumber \\
        \le & \kappa \eta l^2 \left( \left( 1 + \frac{1}{l} \right)^{j+1} - 1 
        \right) \nonumber \\
        \le & \kappa \eta l \left( l + 1 \right) \left( 1 + \frac{1}{l} \right)^{j}
    \end{align}
    
    If $\tau \neq 1$, then we have $a_{j-1} = \tau \left( 1 + \eta \right)$, $b_{j-1} = \kappa \eta \left( \frac{\tau^{j} - 1}{\tau - 1} \right)$ in our case and, 
    \begin{align}\label{eq:psi_value_tau_is_not_1}
        \Psi_j \le & \sum_{k=0}^{j-1} b_k \left( \prod_{i=k+1}^{j-1} a_i \right) = \sum_{k=0}^{j-1} \kappa \eta \left( \frac{\tau^{k+1} - 1}{\tau - 1} \right) \tau^{j-k-1} \left( 1 + \eta \right)^{j-k-1} \nonumber \\
        \le & \kappa \eta \tau^j \left( 1 + \frac{1}{l} \right)^{j} \sum_{k=0}^{j-1} \left( \frac{\tau^{k+1} - 1}{\tau - 1} \right) \tau^{-k-1} \left( 1 + \frac{1}{l} \right)^{-k-1} \qquad \left( \text{Use } \eta \le \frac{1}{l} \right) \nonumber \\
        \le & \kappa \eta \tau^j \left( 1 + \frac{1}{l} \right)^{j} \sum_{k=0}^{j-1} \left( \frac{1 - \tau^{-k-1}}{\tau - 1} \right) \left( 1 + \frac{1}{l} \right)^{-k-1} \nonumber \\
        \le & \frac{\kappa \eta \tau^j}{\tau - 1} \left( 1 + \frac{1}{l} \right)^{j} \sum_{k=0}^{j-1} \left( 1 - \tau^{-k-1} \right) \left( 1 + \frac{1}{l} \right)^{-k-1} \nonumber \\
        \le & \frac{\kappa \eta \tau^j}{\tau - 1} \left( 1 + \frac{1}{l} \right)^{j} \sum_{k=1}^{j} \left( 1 - \tau^{-k} \right) 
    \end{align}

    \paragraph{Final Bound with Readout Function}
    Now let us consider the readout function. 
    Since the last readout layer produces a vector in $\mathbb{R}^{1 \times C}$, we have,
    \begin{align}
        \vert \Delta_l \vert_{2} = & \left\vert \frac{1}{n} \bm{1}_n H_{l-1}^{\prime} (W_{l} + U_{l})  - \frac{1}{n} \bm{1}_n H_{l-1} W_{l} \right\vert_{2} \nonumber \\
        = & \left \vert \frac{1}{n} \bm{1}_n \Delta_{l-1} (W_l + U_l) + \frac{1}{n} \bm{1}_n H_{l-1} U_l \right \vert_{2} \nonumber \\
        \le & \frac{1}{n} \left\vert \bm{1}_n \Delta_{l-1} (W_l + U_l) \right\vert_{2} + \frac{1}{n} \left\vert \bm{1}_n H_{l-1} U_l \right\vert_{2} \nonumber \\
        \le & \frac{1}{n} \Vert W_l + U_l \Vert_{2} \vert \bm{1}_n \Delta_{l-1} \vert_{2} + \frac{1}{n} \Vert U_l \Vert_{2} \vert \bm{1}_n H_{l-1} \vert_{2} \nonumber \\
        \le & \Vert W_l + U_l \Vert_{2} \Psi_{l-1} + \Vert U_l \Vert_{2} \Phi_{l-1}
    \end{align}
    
    If $\tau = 1$, we have, 
    \begin{align}\label{eq:delta_l_tau_eq_1_final}
        \vert \Delta_l \vert_{2} 
        \le & \Vert W_l \Vert_{2} \left( 1 + \frac{1}{l} \right) \kappa \eta l \left( l + 1 \right) \left( 1 + \frac{1}{l} \right)^{l-1}  + (l-1) \kappa \Vert U_l \Vert_{2} \qquad (\text{Use Eq. (\ref{eq:mpgnn_output_bound}), (\ref{eq:psi_value_tau_is_1})} ) \nonumber \\
        \le & \Vert W_l \Vert_{2} \left( 1 + \frac{1}{l} \right)^{l} \kappa \left( \eta l \left( l + 1 \right)  + (l-1) \frac{\Vert U_l \Vert_{2}}{\Vert W_l \Vert_{2}} \left( 1 + \frac{1}{l} \right)^{-l} \right) \nonumber \\
        \le & \Vert W_l \Vert_{2} e \kappa \eta \left( l(l + 1) + (l-1) \right) \nonumber \\
        = & \Vert W_l \Vert_{2} e \kappa \eta \left( l^2 + 2l - 1 \right) \nonumber \\
        \le & \Vert W_l \Vert_{2} e \kappa \eta \left( l + 1 \right)^2
    \end{align}    
    
    Otherwise, we have, 
    \begin{align}\label{eq:delta_l_tau_neq_1}
        \vert \Delta_l \vert_{2} \le & \Vert W_l \Vert_{2} \frac{\kappa \eta \tau^{l-1}}{\tau - 1} \left( 1 + \frac{1}{l} \right)^{l} \sum_{k=1}^{l-1} \left( 1 - \tau^{-k} \right) + \kappa \Vert U_l \Vert_{2} \frac{\tau^{l-1} - 1}{\tau - 1} \qquad (\text{Use Eq. (\ref{eq:mpgnn_output_bound}), (\ref{eq:psi_value_tau_is_not_1})} ) \nonumber \\
        \le & \Vert W_l \Vert_{2} \frac{\kappa \tau^{l-1}}{\tau - 1} \left( 1 + \frac{1}{l} \right)^{l} \left( \eta \sum_{k=1}^{l-1} \left( 1 - \tau^{-k} \right) + \frac{\Vert U_l \Vert_{2}}{\Vert W_l \Vert_{2}} (1 - \tau^{1-l}) \right) 
    \end{align}
    
    If $\tau > 1$, then $\frac{1 - \tau^{-k}}{\tau - 1} \le \frac{1 - \tau^{1-l}}{\tau - 1}$ when $1 \le k \le l-1$.
    If $\tau < 1$, we also have $\frac{1 - \tau^{-k}}{\tau - 1} \le \frac{1 - \tau^{1-l}}{\tau - 1}$ when $1 \le k \le l-1$.
    Therefore, Eq. (\ref{eq:delta_l_tau_neq_1}) can be further relaxed as,
    \begin{align}\label{eq:delta_l_tau_neq_1_final}
        \vert \Delta_l \vert_{2} \le & \Vert W_l \Vert_{2} \frac{\kappa \tau^{l-1}}{\tau - 1} \left( 1 + \frac{1}{l} \right)^{l} \left( \eta \sum_{k=1}^{l-1} \left( 1 - \tau^{-k} \right) + \frac{\Vert U_l \Vert_{2}}{\Vert W_l \Vert_{2}} (1 - \tau^{1-l}) \right) \nonumber \\
        = & \Vert W_l \Vert_{2} \kappa \tau^{l-1} \left( 1 + \frac{1}{l} \right)^{l} \left( \eta \sum_{k=1}^{l-1} \frac{1 - \tau^{-k}}{\tau - 1} + \frac{\Vert U_l \Vert_{2}}{\Vert W_l \Vert_{2}} \frac{1 - \tau^{1-l}}{\tau - 1} \right) \nonumber \\
        \le & \Vert W_l \Vert_{2} \kappa \tau^{l-1} e \left( \eta (l - 1) \frac{\left( 1 - \tau^{1-l} \right)}{\tau - 1} + \frac{\Vert U_l \Vert_{2}}{\Vert W_l \Vert_{2}} \frac{\left( 1 - \tau^{1-l} \right)}{\tau - 1} \right) \nonumber \\
        \le & \Vert W_l \Vert_{2} \kappa \tau^{l-1} e \eta l \frac{\left( 1 - \tau^{1-l} \right)}{\tau - 1} \qquad \left( \text{Use } \frac{\Vert U_l \Vert_{2}}{\Vert W_l \Vert_{2}} \le \eta \right) \nonumber \\
        = & e \eta l \kappa \Vert W_l \Vert_{2} \frac{\tau^{l-1} - 1}{\tau - 1},
    \end{align}

    Therefore, combining Eq. (\ref{eq:delta_l_tau_eq_1_final}) and Eq. (\ref{eq:delta_l_tau_neq_1_final}), we have, 
    \begin{align}
        \vert \Delta_l \vert_{2} \le 
            \begin{cases}
                e \eta \kappa \left( l + 1 \right)^2 \Vert W_l \Vert_{2}, & \text{if}\ d \mathcal{C} = 1 \\
                e \eta \kappa l \Vert W_l \Vert_{2} \frac{\tau^{l-1} - 1}{\tau - 1}, & \text{otherwise.}
            \end{cases}
    \end{align}    
    which proves the lemma.
\end{proof}

\begin{reptheorem}{thm:mpgnn_generalization_bound}
    (MPGNN Generalization Bound) For any $B > 0, l > 1$, let $f_w \in \mathcal{H}: \mathcal{X} \times \mathcal{G} \rightarrow \mathbb{R}^{K}$ be a $l$-step MPGNN.
    Then for any $\delta, \gamma > 0$, with probability at least $1 - \delta$ over the choice of an i.i.d. size-$m$ training set $S$ according to $\mathcal{D}$, for any $w$, we have,
    \begin{enumerate}
        \item If $d \mathcal{C} \neq 1$, then
        {\small
        \begin{align}
    	    L_{\mathcal{D}, 0}(f_w) \le L_{S, \gamma}(f_w) + \mathcal{O} \left( \sqrt{\frac{ B^2 \left( \max\left(\zeta^{-(l+1)}, (\lambda \xi)^{(l+1)/l} \right) \right)^{2} l^2 h \log(lh) \vert w \vert_2^2 + \log \frac{m(l+1)}{\delta} }{\gamma^2 m}} \right). \nonumber
        \end{align}}
        \item If $d \mathcal{C} = 1$, then
        {\small
        \begin{align}
    	    L_{\mathcal{D}, 0}(f_w) \le L_{S, \gamma}(f_w) + \mathcal{O} \left( \sqrt{\frac{ B^2 \max\left( \zeta^{-6}, \lambda^3 C_{\phi}^3 \right) (l+1)^4 h \log(lh) \vert w \vert_2^2 + \log \frac{m}{\delta} }{\gamma^2 m}} \right). \nonumber
        \end{align}}        
    \end{enumerate}
    where 
    $\zeta = \min \left( \Vert W_1 \Vert_{2}, \Vert W_2 \Vert_{2}, \Vert W_l \Vert_{2} \right)$,
    $\vert w \vert_2^2 = \Vert W_1 \Vert_F^2 + \Vert W_2 \Vert_F^2 + \Vert W_l \Vert_F^2$,
    $\mathcal{C} = C_{\phi} C_{\rho} C_{g} \Vert W_2 \Vert_{2}$,
    $\lambda = \Vert W_1 \Vert_{2} \Vert W_l \Vert_{2}$, 
    and $\xi = C_{\phi} \frac{\left( d \mathcal{C} \right)^{l-1} - 1}{d \mathcal{C} - 1}$.
\end{reptheorem}

\begin{proof}
    We will derive the results conditioning on the value of $d \mathcal{C}$.
    
    \paragraph{General Case $d \mathcal{C} \neq 1$}
    We first consider the general case $d \mathcal{C} \neq 1$.
    To derive the generalization bound, we construct a special statistic of the learned weights 
    $\beta = \max\left(\frac{1}{\zeta}, \left( \lambda \xi \right)^{\frac{1}{l}} \right)$. 
    It is clear that $\frac{1}{\zeta} \le \beta$, $\lambda \xi \le \beta^{l}$, and $\lambda \xi / \zeta \le \beta^{l+1}$.
    Note that $\frac{1}{\zeta} = \max \left( \frac{1}{\Vert W_1 \Vert_{2}}, \frac{1}{\Vert W_2 \Vert_{2}}, \frac{1}{\Vert W_l \Vert_{2}} \right)$.
    
    Consider the prior $P = \mathcal{N}(0, \sigma^2 I)$ and the random perturbation $u \sim \mathcal{N}(0, \sigma^2 I)$.
    Note that the $\sigma$ of the prior and the perturbation are the same and will be set according to $\beta$.
    More precisely, we will set the $\sigma$ based on some approximation $\tilde{\beta}$ of $\beta$ since the prior $P$ can not depend on any learned weights directly.
    The approximation $\tilde{\beta}$ is chosen to be a cover set which covers the meaningful range of $\beta$. 
    For now, let us fix any $\tilde{\beta}$ and consider $\beta$ which satisfies $\vert \beta - \tilde{\beta} \vert \le \frac{1}{l+1} \beta$.
    This also implies,
    \begin{align}
        \vert \beta - \tilde{\beta} \vert \le \frac{1}{l+1} \beta ~\Rightarrow~ & \left(1 - \frac{1}{l+1} \right) \beta \le \tilde{\beta} \le \left(1 + \frac{1}{l+1} \right) \beta \nonumber \\ 
        ~\Rightarrow~ & \left(1 - \frac{1}{l+1} \right)^{l+1} \beta^{l+1} \le \tilde{\beta}^{l+1} \le \left(1 + \frac{1}{l+1} \right)^{l+1} \beta^{l+1}  \nonumber \\
        ~\Rightarrow~ & \frac{1}{e} \beta^{l+1} \le \tilde{\beta}^{l+1} \le e \beta^{l+1}  
    \end{align}
    
    From \cite{tropp2012user}, for $U_i \in \mathbb{R}^{h \times h}$ and $U_i \sim \mathcal{N}(\bm{0}, \sigma^2 I)$, we have,
    \begin{align}
        \mathbb{P}\left( \Vert U_i \Vert_{2} \ge t \right) \le 2he^{-t^2 /2h\sigma^2}.
    \end{align}
    Taking a union bound, we have 
    \begin{align}\label{eq:mpgnn_generalization_bound_proof_tmp_1}
        \mathbb{P}\left( \Vert U_1 \Vert_{2} < t ~\&~ \cdots ~\&~ \Vert U_l \Vert_{2} < t \right) & = 1 - \mathbb{P}\left( \exists i, \Vert U_i \Vert_{2} \ge t \right) \nonumber \\
        & \ge 1 - \sum_{i=1}^{l} \mathbb{P}\left( \Vert U_i \Vert_{2} \ge t \right) \nonumber \\
        & \ge 1 - 2lhe^{-t^2 /2h\sigma^2}.
    \end{align}
    Setting $2lhe^{-t^2 /2h\sigma^2} = \frac{1}{2}$, we have $t = \sigma \sqrt{2h\log(4lh)}$. 
    This implies that the probability that the spectral norm of the perturbation of any layer is no larger than $\sigma \sqrt{2h\log(4lh)}$ holds with probability at least $\frac{1}{2}$.
    Plugging this bound into Lemma \ref{lemma:mpgnn_perturbation}, we have with probability at least $\frac{1}{2}$,
    \begin{align}
        \vert f_{w+u}(X, A) - f_w(X, A) \vert_{2} & \le e \frac{t}{\zeta} l C_{\phi} B \Vert W_1 \Vert_{2} \Vert W_l \Vert_{2} \frac{(d\mathcal{C})^{l-1} - 1}{d\mathcal{C} - 1} \nonumber \\
        & = e t l B \frac{\lambda \xi}{\zeta} \nonumber \\
        & = e B l \beta^{l+1} t \le e^2 B l \tilde{\beta}^{l+1} \sigma \sqrt{2h\log(4lh)} \le \frac{\gamma}{4},
    \end{align}    
    
    where we can set $\sigma = \frac{\gamma}{42 B l \tilde{\beta}^{l+1} \sqrt{h\log(4lh)}}$ to get the last inequality.
    Note that Lemma \ref{lemma:mpgnn_perturbation} also requires $\max \left( \frac{\Vert U_1 \Vert_{2}}{\Vert W_1 \Vert_{2}}, \frac{\Vert U_2 \Vert_{2}}{\Vert W_2 \Vert_{2}}, \frac{\Vert U_l \Vert_{2}}{\Vert W_l \Vert_{2}} \right) \le \frac{1}{l}$.
    The requirement is satisfied if $\sigma \le \frac{\zeta}{l \sqrt{2h\log(4lh)}}$ which in turn can be satisfied if 
    \begin{align}\label{eq:gnn_condition_perturbation_lemma}
        \frac{\gamma}{4 e B l \beta^{l+1} \sqrt{2h\log(4lh)}} \le \frac{1}{\beta l \sqrt{2h\log(4lh)}},
    \end{align}
    since the chosen value of $\sigma$ satisfies $\sigma \le \frac{\gamma}{4eB l \beta^{l+1} \sqrt{2h\log(4lh)}}$ and $\frac{1}{\beta} \le \zeta$.
    Therefore, one sufficient condition to make Eq. (\ref{eq:gnn_condition_perturbation_lemma}) hold is $\frac{\gamma}{4eB} \le \beta^{l}$. 
    We will see how to satisfy this condition later.
    
    We now compute the KL term in the PAC-Bayes bound in Lemma \ref{lemma:pac_bayes_deterministic}.
    \begin{align}
        \text{KL}\left( Q \Vert P \right) & = \frac{\vert w \vert_2^2}{2 \sigma^2} \nonumber \\
        & = \frac{42^2 B^2 \tilde{\beta}^{2l+2} l^2 h \log(4lh) }{2\gamma^2} \left(\Vert W_1 \Vert_F^2 + \Vert W_2 \Vert_F^2 + \Vert W_l \Vert_F^2\right) \nonumber \\
        & \le \mathcal{O}\left( \frac{ B^2 \beta^{2l+2} l^2 h \log(lh) }{\gamma^2} \left(\Vert W_1 \Vert_F^2 + \Vert W_2 \Vert_F^2 + \Vert W_l \Vert_F^2\right) \right) 
    \end{align}
    From Lemma \ref{lemma:pac_bayes_deterministic}, fixing any $\tilde{\beta}$, with probability $1 - \delta$ and for all $w$ such that $\vert \beta - \tilde{\beta} \vert \le \frac{1}{l+1}\beta$, we have,
    \begin{align}\label{eq:mpgnn_generalization_bound_proof_tmp_2}
    	L_{\mathcal{D}, 0}(f_w) \le L_{S, \gamma}(f_w) + \mathcal{O} \left( \sqrt{\frac{ B^2 \beta^{2l+2} l^2 h \log(lh) \vert w \vert_2^2 + \log \frac{m}{\delta} }{\gamma^2 m}} \right).
    \end{align}
    Finally, we need to consider multiple choices of $\tilde{\beta}$ so that for any $\beta$, we can bound the generalization error like Eq. (\ref{eq:mpgnn_generalization_bound_proof_tmp_2}).
    In particular, we only need to consider values of $\beta$ in the following range,
    \begin{align}\label{eq:mpgnn_generalization_bound_proof_tmp_3}
        \left( \frac{\gamma}{2B} \right)^{\frac{1}{l}} \le \beta \le \left( \frac{\gamma \sqrt{m}}{2B} \right)^{\frac{1}{l}},
    \end{align}
    since otherwise the bound holds trivially as $L_{\mathcal{D}, 0}(f_w) \le 1$ by definition.
    To see this, if $\beta^{l} < \frac{\gamma}{2B}$, then for any $(X, A)$ and any $j \in \mathbb{N}^{+}_{K}$, we have,
    \begin{align}
        \left\vert f_{w}(X, A)[j] \right\vert & \le \left\vert f_{w}(X, A) \right\vert_{2} = \vert \frac{1}{n} \bm{1}_n H_{l-1} W_l \vert_{2} \nonumber \\
        & \le \frac{1}{n} \vert \bm{1}_n H_{l-1} \vert_{2} \Vert W_l \Vert_{2} \nonumber \\
        & \le \Vert W_l \Vert_{2} \max_{i} \vert H_{l-1}[i, :] \vert_2 \nonumber \\
        & \le B C_{\phi} \Vert W_1 \Vert_{2} \Vert W_l \Vert_{2} \frac{(d\mathcal{C})^{l-1} - 1}{d\mathcal{C} - 1} \qquad (\text{Use Eq. (\ref{eq:mpgnn_output_bound})}) \nonumber \\
        & \le B \lambda \xi \qquad (\text{Use definition of } \lambda \text{ and } \xi) \nonumber \\
        & \le B \beta^{l} \qquad (\text{Use definition of } \beta) \nonumber \\
        & < \frac{\gamma}{2}. 
    \end{align}
    Therefore, based on the definition in Eq. (\ref{eq:empirical_max_margin_loss}), we always have $L_{S,\gamma}(f_w) = 1$ when $\beta^{l} < \frac{\gamma}{2B}$.
    It is hence sufficient to consider $\beta^{l} \ge \frac{\gamma}{2B} > \frac{\gamma}{4eB}$ which also makes Eq. (\ref{eq:gnn_condition_perturbation_lemma}) hold.
    % Therefore, we can satisfy both conditions by considering $\max\left( \left( \frac{\gamma}{2B} \right)^{\frac{1}{l}}, \left( \frac{\gamma}{4eB} \right)^{\frac{1}{l-1}} \right) \le \beta$.
    % It it hence sufficient to consider $\beta^{l} \ge \frac{\gamma}{2B} \ge \frac{\gamma}{4 e B }$ which means the condition in Eq. (\ref{eq:gnn_condition_perturbation_lemma}) is indeed satisfied.
    Alternatively, if $\beta^{l} > \frac{ \gamma \sqrt{m} }{2B}$, 
    the term inside the big-O notation in Eq. (\ref{eq:mpgnn_generalization_bound_proof_tmp_2}) would be,
    \begin{align}
    	\sqrt{\frac{ B^2 \beta^{2l} l^2 h \log(lh) (\beta^2 \vert w \vert_2^2) + \log \frac{m}{\delta} }{\gamma^2 m}} \ge \sqrt{ \frac{l^2 h \log(lh) (\vert w \vert_2^2 / \zeta^2) }{4}} \ge 1,  
    \end{align}
    % where the second inequality hold since $\beta \ge \frac{1}{\zeta}$.
    The last inequality uses the fact that we typically choose $h \ge 2$ in practice, $l \ge 2$ and $\vert w \vert_2^2 \ge \min\left( \Vert W_1 \Vert_F^2, \Vert W_2 \Vert_F^2, \Vert W_l \Vert_F^2 \right) \ge \zeta^2$.
    Since we only need to consider $\beta$ in the range of Eq. (\ref{eq:mpgnn_generalization_bound_proof_tmp_3}), one sufficient condition to ensure $\vert \beta - \tilde{\beta} \vert \le \frac{1}{l+1}\beta$ holds would be $\vert \beta - \tilde{\beta} \vert \le \frac{1}{l+1} \left( \frac{\gamma}{2B} \right)^{\frac{1}{l}}$.
    Therefore, if we can find a covering of the interval in Eq. (\ref{eq:mpgnn_generalization_bound_proof_tmp_3}) with radius $\frac{1}{l+1} \left( \frac{\gamma}{2B} \right)^{\frac{1}{l}}$ and make sure bounds like Eq. (\ref{eq:mpgnn_generalization_bound_proof_tmp_2}) holds while $\tilde{\beta}$ takes all possible values from the covering, then we can get a bound which holds for all $\beta$.
    It is clear that we only need to consider a covering $C$ with size $\vert C \vert = \frac{(l+1)}{2}\left( m^{1/2l} - 1\right)$.
    Therefore, denoting the event of Eq. (\ref{eq:mpgnn_generalization_bound_proof_tmp_2}) with $\tilde{\beta}$ taking the $i$-th value of the covering as $E_i$, we have 
    \begin{align}
        \mathbb{P}\left( E_1 ~\&~ \cdots ~\&~ E_{\vert C \vert} \right) & = 1 - \mathbb{P}\left( \exists i, \bar{E}_i \right) \ge 1 - \sum_{i=1}^{\vert C \vert} \mathbb{P}\left( \bar{E}_i \right) \ge 1 - \vert C \vert \delta,
    \end{align}%
    where $\bar{E}_i$ denotes the complement of $E_i$.
    Hence, with probability $1 - \delta$ and for all $w$, we have,
    {
    \begin{align}
    	L_{\mathcal{D}, 0}(f_w) & \le L_{S, \gamma}(f_w) + \mathcal{O} \left( \sqrt{\frac{ B^2 \beta^{2l+2} l^2 h \log(lh) \vert w \vert_2^2 + \log \frac{m\vert C \vert}{\delta} }{\gamma^2 m}} \right) \nonumber \\
    	& = L_{S, \gamma}(f_w) + \mathcal{O} \left( \sqrt{\frac{ B^2 \beta^{2l+2} l^2 h \log(lh) \vert w \vert_2^2 + \log \frac{m(l+1)}{\delta} + \frac{1}{2l} \log m }{\gamma^2 m}} \right) \nonumber \\
    	& = L_{S, \gamma}(f_w) + \mathcal{O} \left( \sqrt{\frac{ B^2 \max\left(\zeta^{-1}, \left( \lambda \xi \right)^{\frac{1}{l}} \right)^{2(l+1)} l^2 h \log(lh) \vert w \vert_2^2 + \log \frac{m(l+1)}{\delta} }{\gamma^2 m}} \right)
    \end{align}}%
    which proves the theorem for the case of $d \mathcal{C} \neq 1$.
    
    \paragraph{Special Case $d \mathcal{C} = 1$}
    Now we consider $d \mathcal{C} = 1$ of which the proof follows the logic of the one for $d \mathcal{C} \neq 1$. 
    Note that this case happens rarely in practice. 
    We only include it for the completeness of the analysis.
    We again construct a statistic $\beta = \max\left(\frac{1}{\zeta}, \sqrt{\lambda C_{\phi}} \right)$.
    % Fixing any $\tilde{\beta}$, we consider the values of $\beta$ which satisfies $\frac{1}{\sqrt{e}} \beta \le \tilde{\beta} \le \sqrt{e} \beta$.
    For now, let us fix any $\tilde{\beta}$ and consider $\beta$ which satisfies $\vert \beta - \tilde{\beta} \vert \le \frac{1}{3} \beta$. This also implies $\frac{1}{e} \beta^{3} \le \tilde{\beta}^{3} \le e \beta^{3}$.
    Based on Lemma \ref{lemma:mpgnn_perturbation}, we have,
    \begin{align}
        \vert f_{w+u}(X, A) - f_w(X, A) \vert_{2} & \le e \frac{t}{\zeta} (l + 1)^2 C_{\phi} B \Vert W_1 \Vert_{2} \Vert W_l \Vert_{2} \nonumber \\
        & = e t (l + 1)^2 B \frac{\lambda C_{\phi}}{\zeta} \le e B (l + 1)^2 \beta^{3} t \nonumber \\
        & \le e^2 B (l + 1)^2 \tilde{\beta}^{3} \sigma \sqrt{2h\log(4lh)} \le \frac{\gamma}{4},
    \end{align}
    where we can set $\sigma = \frac{\gamma}{42 B (l+1)^2 \tilde{\beta}^3 \sqrt{h\log(4lh)}}$ to get the last inequality.
    Note that Lemma \ref{lemma:mpgnn_perturbation} also requires $\max \left( \frac{\Vert U_1 \Vert_{2}}{\Vert W_1 \Vert_{2}}, \frac{\Vert U_2 \Vert_{2}}{\Vert W_2 \Vert_{2}}, \frac{\Vert U_l \Vert_{2}}{\Vert W_l \Vert_{2}} \right) \le \frac{1}{l}$.
    The requirement is satisfied if $\sigma \le \frac{\zeta}{l \sqrt{2h\log(4lh)}}$ which in turn can be satisfied if 
    \begin{align}\label{eq:gnn_condition_perturbation_lemma_tau_1}
        \frac{\gamma}{4 e B (l+1)^2 \beta^{3} \sqrt{2h\log(4lh)}} \le \frac{1}{\beta l \sqrt{2h\log(4lh)}},
    \end{align}
    since the chosen value of $\sigma$ satisfies $\sigma \le \frac{\gamma}{4eB (l+1)^2 \beta^3 \sqrt{2h\log(4lh)}}$ and $\frac{1}{\beta} \le \zeta$.
    As shown later, we only need to consider a certain range of values of $\beta$ which naturally satisfy the condition $\frac{\gamma l}{4 e B (l+1)^2 } \le \beta^2$, \ie, the equivalent form of Eq. (\ref{eq:gnn_condition_perturbation_lemma_tau_1}).
    This assures the applicability of Lemma \ref{lemma:mpgnn_perturbation}.
    Now we compute the KL divergence,
    \begin{align}\label{eq:mpgnn_generalization_bound_proof_tmp_4}
        \text{KL}\left( Q \Vert P \right) & = \frac{\vert w \vert_2^2}{2 \sigma^2} \nonumber \\
        & = \frac{42^2 B^2 \tilde{\beta}^{6} (l+1)^4 h \log(4lh) }{2\gamma^2} \left( \Vert W_1 \Vert_F^2 + \Vert W_2 \Vert_F^2 + \Vert W_l \Vert_F^2 \right) \nonumber \\
        & \le \mathcal{O}\left( \frac{ B^2 \beta^{6} (l+1)^4 h \log(lh) }{\gamma^2} \left( \Vert W_1 \Vert_F^2 + \Vert W_2 \Vert_F^2 + \Vert W_l \Vert_F^2 \right) \right) 
    \end{align}    
    
    In particular, we only need to consider values of $\beta$ in the following range,
    \begin{align}\label{eq:mpgnn_generalization_bound_proof_tmp_5}
        \sqrt{\frac{\gamma}{2Bl}} \le \beta \le \sqrt{\frac{\gamma \sqrt{m}}{2Bl}},
    \end{align}
    since otherwise the bound holds trivially as $L_{\mathcal{D}, 0}(f_w) \le 1$ by definition.
    To see this, if $\beta < \frac{\gamma}{2Bl}$, then for any $(X, A)$ and any $j \in \mathbb{N}^{+}_{K}$, we have,
    \begin{align}
        \left\vert f_{w}(X, A)[j] \right\vert & \le \left\vert f_{w}(X, A) \right\vert_{2} = \vert \frac{1}{n} \bm{1}_n H_{l-1} W_l \vert_{2} \nonumber \\
        & \le \frac{1}{n} \vert \bm{1}_n H_{l-1} \vert_{2} \Vert W_l \Vert_{2} \nonumber \\
        & \le \Vert W_l \Vert_{2} \max_{i} \vert H_{l-1}[i, :] \vert_2 \nonumber \\
        & \le B (l - 1) C_{\phi} \Vert W_1 \Vert_{2} \Vert W_l \Vert_{2} \qquad (\text{Use Eq. (\ref{eq:mpgnn_output_bound})}) \nonumber \\
        & \le B (l - 1) \lambda C_{\phi} \qquad (\text{Use definition of } \lambda) \nonumber \\
        & \le B l \beta^2 \qquad (\text{Use definition of } \beta) \nonumber \\
        & < \frac{\gamma}{2}. 
    \end{align}
    Therefore, based on the definition in Eq. (\ref{eq:empirical_max_margin_loss}), we always have $L_{S,\gamma}(f_w) = 1$ when $\beta < \frac{\gamma}{2Bl}$.
    It it hence sufficient to consider $\beta^2 \ge \frac{\gamma}{2Bl} \ge \frac{\gamma l}{4eB l^2} \ge \frac{\gamma l}{4eB (l+1)^2}$ which means the condition in Eq. (\ref{eq:gnn_condition_perturbation_lemma_tau_1}) is indeed satisfied.
    Alternatively, if $\beta > \sqrt{\frac{ \gamma \sqrt{m} }{2Bl}}$, 
    the term inside the big-O notation in Eq. (\ref{eq:mpgnn_generalization_bound_proof_tmp_4}) would be,
    \begin{align}
    	\sqrt{\frac{ B^2 \beta^4 (l+1)^4 h \log(lh) \beta^{2} \vert w \vert_2^2 + \log \frac{m}{\delta} }{\gamma^2 m}} \ge \sqrt{ \frac{(l+1)^4 h \log(lh) \frac{\vert w \vert_2^2}{\zeta^{2}} }{4 l^2}} \ge 1,  
    \end{align}
    where the first inequality hold since $\beta \ge \frac{1}{\zeta}$.
    The last inequality uses the fact that we typically choose $h \ge 2$ in practice, $l \ge 2$, and $\vert w \vert_2^2 \ge \min\left( \Vert W_1 \Vert_F^2, \Vert W_2 \Vert_F^2, \Vert W_l \Vert_F^2 \right) \ge \zeta^2$.    
    Since we only need to consider $\beta$ in the range of Eq. (\ref{eq:mpgnn_generalization_bound_proof_tmp_5}), one sufficient condition to ensure $\vert \beta - \tilde{\beta} \vert \le \frac{1}{3} \beta$ always holds would be $\vert \beta - \tilde{\beta} \vert \le \frac{1}{3} \sqrt{\frac{\gamma}{2Bl}}$.
    Therefore, if we can find a covering of the interval in Eq. (\ref{eq:mpgnn_generalization_bound_proof_tmp_5}) with radius $\frac{1}{3} \sqrt{\frac{\gamma}{2Bl}}$ and make sure bounds like Eq. (\ref{eq:mpgnn_generalization_bound_proof_tmp_2}) holds while $\tilde{\beta}$ takes all possible values from the covering, then we can get a bound which holds for all $\beta$.
    It is clear that we only need to consider a covering $C$ with size $\vert C \vert = \frac{3}{2}\left(m^{\frac{1}{4}} - 1\right)$.    
    
    Hence, with probability $1 - \delta$ and for all $w$, we have,
    {
    \begin{align}
    	L_{\mathcal{D}, 0}(f_w) & \le L_{S, \gamma}(f_w) + \mathcal{O} \left( \sqrt{\frac{ B^2 \beta^{6} (l+1)^4 h \log(lh) \vert w \vert_2^2 + \log \frac{m\vert C \vert}{\delta} }{\gamma^2 m}} \right) \nonumber \\
    	& = L_{S, \gamma}(f_w) + \mathcal{O} \left( \sqrt{\frac{ B^2 \beta^{6} (l+1)^4 h \log(lh) \vert w \vert_2^2 + \log \frac{m}{\delta} + \frac{1}{4} \log{m} }{\gamma^2 m}} \right) \nonumber \\
    	& = L_{S, \gamma}(f_w) + \mathcal{O} \left( \sqrt{\frac{ B^2 \max\left( \zeta^{-6}, \lambda^{3} C_{\phi}^3 \right) (l+1)^4 h \log(lh) \vert w \vert_2^2 + \log \frac{m}{\delta} }{\gamma^2 m}} \right)
    \end{align}}%
    which proves the theorem for the case of $d \mathcal{C} = 1$.    
    
\end{proof}

\begin{remark}
    Note that our proof applies to both homogeneous and non-homogeneous GNNs.
\end{remark}

%% file: proof/bound_comparison.tex
\subsection{Bound Comparison}\label{sect:appendix_bound_comparison}

\begin{table}[t]
\begin{center}
\resizebox{\textwidth}{!}
{
\begin{tabular}{c|cccc}
    \hline
    \toprule
    Statistics & \begin{tabular}{@{}c@{}}Max Node Degree \\ $d-1$ \end{tabular} & \begin{tabular}{@{}c@{}}Max Hidden Dim \\ $h$ \end{tabular} & \begin{tabular}{@{}c@{}}Spectral Norm of \\ Learned Weights \end{tabular} \\
    \midrule
    \midrule
    \begin{tabular}{@{}c@{}}VC-Dimension \\ \citep{scarselli2018vapnik}\end{tabular}
      & - & $\mathcal{O}\left( h^4 \right)$ & - \\ 
    \begin{tabular}{@{}c@{}}Rademacher Complexity \\ \citep{garg2020generalization} Case A \end{tabular} & $\mathcal{O}\left( d^{l-1} \sqrt{ \log(d^{l-1}) } \right)$ & $\mathcal{O}\left( h \right)$ & $\mathcal{O}\left( \lambda \mathcal{C} \xi \sqrt{ \log \left( \lambda \mathcal{C} \xi \right)} \right)$ \\ 
    \begin{tabular}{@{}c@{}}Rademacher Complexity \\ \citep{garg2020generalization} Case B \end{tabular} & $\mathcal{O}\left( d^{l-1} \sqrt{\log(d^{l-2})} \right)$ & $\mathcal{O}\left( h \sqrt{\log\sqrt{h}} \right)$ & $\mathcal{O}\left( \lambda \mathcal{C} \xi \sqrt{ \log \left( \lambda \xi \right)} \right)$ \\ 
    \begin{tabular}{@{}c@{}}Rademacher Complexity \\ \citep{garg2020generalization} Case C \end{tabular} & $\mathcal{O}\left( d^{l-1} \sqrt{\log(d^{2l-3})} \right)$ & $\mathcal{O}\left( h \sqrt{\log\sqrt{h}} \right)$ & $\mathcal{O}\left( \lambda \mathcal{C} \xi \sqrt{ \log \left( \Vert W_{2} \Vert_2 \lambda \xi^{2} \right)} \right)$ \\
    Ours Case A & - & $\mathcal{O}\left( \sqrt{h \log h} \right)$ & $\mathcal{O}\left( \zeta^{-(l+1)} \sqrt{ \Vert W_1 \Vert_F^2 + \Vert W_2 \Vert_F^2 + \Vert W_l \Vert_F^2 } \right)$ \\
    Ours Case B & $\mathcal{O}\left( d^{\frac{(l+1)(l-2)}{l}} \right)$ & $\mathcal{O}\left( \sqrt{h \log h} \right)$ & $\mathcal{O}\left( \lambda^{1 + \frac{1}{l}} \xi^{1 + \frac{1}{l}} \sqrt{ \Vert W_1 \Vert_F^2 + \Vert W_2 \Vert_F^2 + \Vert W_l \Vert_F^2 } \right)$ \\
    \bottomrule
\end{tabular}
}
\end{center}
\caption{Detailed comparison of Generalization Bounds for GNNs. ``-" means inapplicable. 
We only consider the general case $d \mathcal{C} \Vert W_2 \Vert_2 \neq 1$ and simplify the Rademacher complexity based bounds (\wrt spectral norm of weights) based on the assumption that $C_{\phi} \ll d \mathcal{C} \xi$ which generally holds in practice.
Here $\mathcal{C} = C_{\phi} C_{\rho} C_{g} \Vert W_2 \Vert_2$, 
$\xi = C_{\phi} \frac{\left( d \mathcal{C} \right)^{l-1} - 1}{d \mathcal{C} - 1}$, $\zeta = \min \left( \Vert W_1 \Vert_{2}, \Vert W_2 \Vert_{2}, \Vert W_l \Vert_{2} \right)$, 
and $\lambda = \Vert W_1 \Vert_{2} \Vert W_l \Vert_{2}$.
Note that $d^{\frac{(l+1)(l-2)}{l}} \le d^{\frac{l^2-l}{l}} = d^{l-1}$.} 
\label{table:comparison_bound_detail}
\end{table}

In this section, we explain the details of the comparison with Rademacher complexity based generalization bounds of GNNs.

\subsubsection{Rademacher Complexity based Bound}
We first restate the Rademacher complexity bound from \citep{garg2020generalization} as below:

{\small
\begin{align}\label{eq:rademacher_complexity_bound}
	L_{D, 0}(f_w) & \le L_{S, \gamma}(f_w) \nonumber \\ 
	& + \mathcal{O} \left( \frac{1}{\gamma m} + h B_{l} Z \sqrt{ \frac{ \log \left( B_{l} \sqrt{m} \max\left(Z, M\sqrt{h} \max \left( B B_1, \bar{R} B_2 \right) \right) \right) }{\gamma^2 m} } + \sqrt{\frac{\log \frac{1}{\delta} }{m}} \right)
\end{align}
}%
where $M = C_{\phi} \frac{ (C_{\phi} C_{\rho} C_{g}dB_2)^{l-1} - 1 }{C_{\phi} C_{\rho} C_{g}dB_2 - 1}$, $Z = C_{\phi} \left( B B_1 + \bar{R} B_2 \right)$, $\bar{R} \le C_{\rho} C_{g}d \min \left(b\sqrt{h}, B B_1 M \right) $, $b$ is the uniform upper bound of $\phi$ (\ie, $\forall x \in \mathbb{R}^h$, $\phi(x) \le b$), and $B_1, B_2, B_l$ are the spectral norms of the weight matrices $W_1, W_2, W_l$.
Note that the numerator of $M$ has the exponent $l-1$ since we count the readout function in the number of layers/steps, \ie, there are $l-1$ message passing steps in total.

\subsubsection{Comparison in Our Context}

For typical message passing GNNs presented in the literature, node state update function $\phi$ could be a neural network like MLP or GRU, a ReLU unit, etc.
% MLP~\cite{hamilton2017inductive}, GRU~\cite{li2015gated}, LSTM~\cite{qi20173d}.
This makes the assumption of the uniform upper bound on $\phi$ impractical, \eg, $b = \infty$ when $\phi$ is ReLU.
Therefore, we dot not adopt this assumption in our analysis\footnote{If we introduce the uniform upper bound on $\phi$ in our analysis, we can also obtain a similar functional dependency in our bound like $\min(b\sqrt{h}, \cdot)$. But as aforementioned, it is somewhat impractical and leads to a more cumbersome bound.}.

\paragraph{Rademacher Complexity Based Bound}
Based on the above consideration, we have $\bar{R} \le C_{\rho} C_{g}d B B_1 M$.
We further convert some notations in the original bound to the ones in our context.
\begin{align}
    M & = C_{\phi} \frac{ (C_{\phi} C_{\rho} C_{g}dB_2)^{l-1} - 1 }{C_{\phi} C_{\rho} C_{g}dB_2 - 1} = \xi \\
    \bar{R} & \le C_{\rho} C_{g} d B B_1 M  = C_{\rho} C_{g} d B \Vert W_1 \Vert_2 \xi\\
    Z & = C_{\phi} \left( B B_1 + \bar{R} B_2 \right) = B \Vert W_1 \Vert_2 \left( C_{\phi} + d \mathcal{C} \xi \right),
\end{align}
where we use the same abbreviations as in Theorem \ref{thm:mpgnn_generalization_bound},
$\xi = C_{\phi} \frac{\left( d \mathcal{C} \right)^{l-1} - 1}{d \mathcal{C} - 1}$,
$\lambda = \Vert W_1 \Vert_{2} \Vert W_l \Vert_{2}$,
$\mathcal{C} = C_{\phi} C_{\rho} C_{g} \Vert W_2 \Vert_{2}$.

We need to consider three cases for the big-O term of the original bound in Eq. (\ref{eq:rademacher_complexity_bound}) depending on the outcomes of the two point-wise maximum functions. 
% We will simplify the bounds by focusing on the dependencies on the spectral norms of weights, maximum hidden dimension $h$, maximum node degree $d-1$ and sample size $m$.

\paragraph{Case A}
If $\max \left(Z, M\sqrt{h} \max \left( B B_1, \bar{R} B_2 \right) \right) = Z$, then the generalization bound is,
{
\begin{align}
    \mathcal{O} & \left( h B_{l} Z \sqrt{ \frac{ \log \left( B_{l} \sqrt{m} Z \right) }{m} } \right) \nonumber \\
    & = \mathcal{O}\left( h \Vert W_l \Vert B \Vert W_1 \Vert_2 \left( C_{\phi} + d \mathcal{C} \xi \right) \sqrt{ \frac{ \log \left( \Vert W_{l} \Vert_2 \sqrt{m} B \Vert W_1 \Vert_2 \left( C_{\phi} + d \mathcal{C} \xi \right) \right) }{m} } \right) \nonumber \\
    & = \mathcal{O}\left( h B \lambda \left( C_{\phi} + d \mathcal{C} \xi \right) \sqrt{ \frac{ \log \left(  \sqrt{m} B \lambda \left( C_{\phi} + d \mathcal{C} \xi \right) \right) }{m} } \right).
\end{align}
}

\paragraph{Case B}
If $\max \left(Z, M\sqrt{h} \max \left( B B_1, \bar{R} B_2 \right) \right) = M\sqrt{h} \max \left( B B_1, \bar{R} B_2 \right)$ and $BB_1 = \max \left( B B_1, \bar{R} B_2 \right)$, then the generalization bound is,
{
\begin{align}
    \mathcal{O} & \left( h B_{l} Z \sqrt{ \frac{ \log \left( B_{l} \sqrt{m} M\sqrt{h} B B_1 \right) }{m} } \right) \nonumber \\
    & = \mathcal{O} \left( h \Vert W_l \Vert B \Vert W_1 \Vert_2 \left( C_{\phi} + d \mathcal{C} \xi \right) \sqrt{ \frac{ \log \left( \Vert W_l \Vert_2 \sqrt{m} \xi \sqrt{h} B \Vert W_1 \Vert_2 \right) }{m} } \right) \nonumber \\
    & = \mathcal{O} \left( h B \lambda \left( C_{\phi} + d \mathcal{C} \xi \right) \sqrt{ \frac{ \log \left( \sqrt{m} \lambda \xi \sqrt{h} B \right) }{m} } \right)
\end{align}
}

\paragraph{Case C}

If $\max \left(Z, M\sqrt{h} \max \left( B B_1, \bar{R} B_2 \right) \right) = M\sqrt{h} \max \left( B B_1, \bar{R} B_2 \right)$ and $\bar{R} B_2 = \max \left( B B_1, \bar{R} B_2 \right)$, then the generalization bound is,
{
\begin{align}
    \mathcal{O} & \left( h B_{l} Z \sqrt{ \frac{ \log \left( B_{l} \sqrt{m} M\sqrt{h} \bar{R} B_2 \right) }{m} } \right) \nonumber \\
    & = \mathcal{O} \left( h \Vert W_l \Vert B \Vert W_1 \Vert_2 \left( C_{\phi} + d \mathcal{C} \xi \right) \sqrt{ \frac{ \log \left( \Vert W_l \Vert_2 \sqrt{m} \xi \sqrt{h} C_{\rho} C_{g} d B \Vert W_1 \Vert_2 \xi \Vert W_2 \Vert_2 \right) }{m} } \right) \nonumber \\
    & = \mathcal{O} \left( h B \lambda \left( C_{\phi} + d \mathcal{C} \xi \right) \sqrt{ \frac{ \log \left( \lambda \sqrt{m} \sqrt{h} C_{\rho} C_{g} d B \xi^{2} \Vert W_2 \Vert_2 \right) }{m} } \right) 
\end{align}
}

We show the detailed dependencies of the Rademarcher complexity based bound under three cases in Table \ref{table:comparison_bound_detail}.
In practice, we found message passing GNNs typically do not behave like a contraction mapping. In other words, we have $d \mathcal{C} > 1$ and $\xi \gg 1$ hold for many datasets.
Therefore, the case C happens more often in practice, \ie, $\max \left(Z, M\sqrt{h} \max \left( B B_1, \bar{R} B_2 \right) \right) = M \sqrt{h} \bar{R} B_2$.

\paragraph{PAC Bayes Bound}
For our PAC-Bayes bound in Theorem \ref{thm:mpgnn_generalization_bound}, we also need to consider two cases which correspond to $\max\left({\zeta}^{-1}, (\lambda \xi)^{\frac{1}{l}} \right) = {\zeta}^{-1}$ (case A) and $\max\left({\zeta}^{-1}, (\lambda \xi)^{\frac{1}{l}} \right) = (\lambda \xi)^{\frac{1}{l}}$ (case B) respectively.
Here $\zeta = \min \left( \Vert W_1 \Vert_{2}, \Vert W_2 \Vert_{2}, \Vert W_l \Vert_{2} \right)$.
We show the detailed dependencies of our bound under three cases in Table \ref{table:comparison_bound_detail}.
Again, in practice, we found $d \mathcal{C} > 1$, $\xi \gg 1$ and $\zeta \le 1$.
Therefore, case B occurs more often.

\paragraph{VC-dim Bound}
\citep{scarselli2018vapnik} show that the upper bound of the VC-dimension of general GNNs with Sigmoid or Tanh activations is $\mathcal{O}(p^4 N^2)$ where $p$ is the total number of parameters and $N$ is the maximum number of nodes. 
Since $p = \mathcal{O}(h^2)$ in our case, the VC-dim bound is $\mathcal{O}(h^8 N^2)$.
Therefore, the corresponding generalization bound scales as $\mathcal{O}(\frac{h^4 N}{\sqrt{m}})$.
Note that $N$ is at least $d$ and could be much larger than $d$ for some datasets.

%% file: proof/mlp_cnn_gnn_connection.tex
\subsection{Connections with Existing Bounds of MLPs/CNNs}\label{sect:appendix_connections}

\paragraph{ReLU Networks are Special GCNs}
Since regular feedforward neural networks could be viewed as a special case of GNNs by treating each sample as the node feature of a single-node graph, it is natural to investigate the connections between these two classes of models.
In particular, we consider the class of ReLU networks studied in \cite{neyshabur2017pac},
\begin{align}
    H_0 & = X \quad && (\text{Input Node Feature}) \nonumber \\
    H_{k} & = \sigma_{k} \left( H_{k-1} W_{k} \right) \quad && (k \text{-th Layer}) \nonumber \\
    H_{l} & = H_{l-1} W_{l} \quad && (\text{Readout Layer}), \label{eq:relu_networks}
\end{align}
where $\sigma_{k} = \text{ReLU}$.
It includes two commonly-seen types of deep neural networks, \ie, fully connected networks (or MLPs) and convolutional neural networks (CNNs), as special cases. 
Comparing Eq. (\ref{eq:relu_networks}) against Eq. (\ref{eq:gcn}), it is clear that these ReLU networks can be further viewed as special cases of GCNs which operate on single-node graphs, \ie, $\tilde{L} = I$.

\paragraph{Connections of Generalization Bounds}

Let us restate the PAC-Bayes bound of ReLU networks in \cite{neyshabur2017pac} as below,
\begin{align}
	L_{D, 0}(f_w) \le L_{S, \gamma}(f_w) + \mathcal{O} \left( \sqrt{\frac{ B^2 l^2 h \log(lh) \prod\limits_{i=1}^{l} \Vert W_{i} \Vert_{2}^2 \sum\limits_{i=1}^{l} \frac{\Vert W_{i} \Vert_F^2}{\Vert W_{i} \Vert_{2}^2} + \log \frac{ml}{\delta} }{\gamma^2 m}} \right). \label{eq:pac_bayes_relu_networks}
\end{align}
Comparing it with the bound in Theorem \ref{thm:gcn_generalization_bound}, we can find that our bound only adds a factor $d^{l-1}$ to the first term inside the square root of the big-O notation which is brought by the underlying graph structure of the data.
If we consider GCNs operating on single-node graphs, \ie, the case where GCNs degenerate to ReLU networks, two bounds coincide since $d = 1$.
Therefore, our Theorem \ref{thm:gcn_generalization_bound} directly generalizes the result in \cite{neyshabur2017pac} to GCNs which is a strictly larger class of models than ReLU networks.

%% file: proof/exp_detail.tex
\subsection{Experimental Details}\label{sect:appendix_exp}

\paragraph{Datasets}

We create 6 synthetic datasets by generating random graphs from different random graph models.
In particular, the first 4 synthetic datasets correspond to the Erdős–Rényi models with different edge probabilities: 1) Erdős–Rényi-1 (ER-1), edge probability = 0.1; 2) Erdős–Rényi-2 (ER-2), edge probability = 0.3; 3) Erdős–Rényi-3 (ER-3), edge probability = 0.5; 4) Erdős–Rényi-4 (ER-4), edge probability = 0.7.
The remaining 2 synthetic datasets correspond to the stochastic block model with the following settings: 1) Stochastic-Block-Model-1 (SBM-1), two blocks, sizes = [40, 60], edge probability = [[0.25, 0.13], [0.13, 0.37]]; 2) Stochastic-Block-Model-2 (SBM-2), three blocks, sizes = [25, 25, 50], edge probability = [[0.25, 0.05, 0.02], [0.05, 0.35, 0.07], [0.02, 0.07, 0.40]].
Each synthetic dataset has 200 graphs where the number of nodes of individual graph is 100, the number of classes is 2, and the random train-test split ratio is $90\% / 10\%$.
For each random graph of individual synthetic dataset, we generate the 16-dimension random Gaussian node feature (normalized to have unit $\ell_2$ norm) and a binary class label following a uniform distribution. 
We summarize the statistics of the real-world and synthetic datasets in Table \ref{table:real_dataset} and Table \ref{table:syn_dataset} respectively.

\begin{table}[t]
\begin{center}
% \resizebox{\textwidth}{!}
{
\begin{tabular}{c|cccc}
    \hline
    \toprule
    Statistics & COLLAB & IMDB-BINARY & IMDB-MULTI & PROTEINS \\
    \midrule
    \midrule
    max \# nodes & 492 & 136 & 89 & 620 \\
    max \# edges & 80727 & 2634 & 3023 & 2718 \\
    \# classes & 3 & 2 & 3 & 2 \\
    \# graphs & 5000 & 1000 & 1500 & 1113 \\
    train/test & 4500/500 & 900/100 & 1350/150 & 1002/111 \\
    feature dimension & 367 & 65 & 59 & 3 \\
    max node degree & 491 & 135 & 88 & 25 \\
    \bottomrule
\end{tabular}
}
\end{center}
\caption{Statistics of real-world datasets.} 
\label{table:real_dataset}
\end{table}

\begin{table}[t]
\begin{center}
% \resizebox{\textwidth}{!}
{
\begin{tabular}{c|cccccc}
    \hline
    \toprule
    Statistics & ER-1 & ER-2 & ER-3 & ER-4 & SBM-1 & SBM-2 \\
    \midrule
    \midrule
    max \# nodes & 100 & 100 & 100 & 100 & 100 & 100 \\
    max \# edges & 1228 & 3266 & 5272 & 7172 & 2562 & 1870 \\
    \# classes & 2 & 2 & 2 & 2 & 2 & 2 \\
    \# graphs & 200 & 200 & 200 & 200 & 200 & 200 \\
    train/test & 180/20 & 180/20 & 180/20  & 180/20 & 180/20 & 180/20 \\
    feature dimension & 16 & 16 & 16 & 16 & 16 & 16 \\
    max node degree & 25 & 48 & 69 & 87 & 25 & 36 \\
    \bottomrule
\end{tabular}
}
\end{center}
\caption{Statistics of synthetic datasets.} 
\label{table:syn_dataset}
\end{table}

\paragraph{Experimental Setup}

For all MPGNNs used in the experiments, we specify $\phi = \text{ReLU}$, $\rho = \text{Tanh}$, and $g = \text{Tanh}$ which imply $C_{\phi} = C_{\rho} = C_g = 1$.
For experiments on real-world datasets, we set $h = 128$, the number of training epochs to $50$, and try $2$ values of network depth, \ie, $l = 2$ and $l = 4$.
The batch size is set to $20$ (due to the GPU memory constraint) on COLLAB and $128$ for others.
For experiments on synthetic datasets, we set $h = 128$ and try $4$ values of network depth, \ie, $l = 2$, $l = 4$, $l = 6$ and $l = 8$.
Since these generated datasets essentially require GNNs to fit to random labels which is arguably hard, we extend the number of training epochs to $200$.
For all above experiments, we use Adam as the optimizer with learning rate set to $1.0e^{-2}$. 
The batch size is $128$ for all synthetic datasets.

\paragraph{Bound Computations}
For all datasets, we compute the bound values for the learned model saved in the end of the training.
% To make the comparison fair, we remove all constant factors in both bounds.
% In particular, for our bound, we compute the following quantity
% \begin{align}
%     \sqrt{\frac{ B^2 \left( \max\left(\zeta^{-(l+1)}, (\lambda \xi)^{\frac{l+1}{l}} \right) \right)^{2} l^2 h \log(lh) \vert w \vert_2^2}{\gamma^2 m}}.
% \end{align}
% For the Rademacher complexity based bound, we compute the following quantity
% \begin{align}
% 	h B_{l} Z \sqrt{ \frac{ \log \left( B_{l} \sqrt{m} \max\left(Z, M\sqrt{h} \max \left( B B_1, \bar{R} B_2 \right) \right) \right) }{\gamma^2 m} }, 
% \end{align}
% where the variables are the same as Eq. (\ref{eq:rademacher_complexity_bound}).
We also consider the constants of both bounds in the computation.
In particular, for our bound, we compute the following quantity
\begin{align}
    \sqrt{\frac{ 42^2 B^2 \left( \max\left(\zeta^{-(l+1)}, (\lambda \xi)^{\frac{l+1}{l}} \right) \right)^{2} l^2 h \log(4lh) \vert w \vert_2^2}{\gamma^2 m}}.
\end{align}

For the Rademacher complexity based bound, we compute the following quantity
\begin{align}
	2 \times 24 h B_{l} Z \sqrt{ \frac{ 3 \log \left( 24 B_{l} \sqrt{m} \max\left(Z, M\sqrt{h} \max \left( B B_1, \bar{R} B_2 \right) \right) \right) }{\gamma^2 m} }, 
\end{align}
where the variables are the same as Eq. (\ref{eq:rademacher_complexity_bound}).

\paragraph{Experimental Results}

\begin{table}[t]
\begin{center}
% \resizebox{\textwidth}{!}
{
\begin{tabular}{c|cccc}
    \hline
    \toprule
    $l=2$ & PROTEINS & IMDB-MULTI & IMDB-BINARY & COLLAB \\
    \midrule
    \midrule
    Rademacher & $11.80 \pm 0.18$ & $16.66 \pm 0.04$ & $17.37 \pm 0.02$ & $21.26 \pm 0.07$ \\
    PAC-Bayes & $\bm{8.45} \pm 0.28$ & $\bm{15.26} \pm 0.07$ & $\bm{15.44} \pm 0.03$ & $\bm{19.37} \pm 0.17$ \\
    \midrule
    \midrule    
    $l=4$ &  &  &  &  \\
    \midrule
    \midrule
    Rademacher & $24.04 \pm 0.23$ & $29.94 \pm 0.10$ & $31.38 \pm 0.09$ & $41.03 \pm 0.33$ \\
    PAC-Bayes & $\bm{22.10} \pm 0.23$ & $\bm{28.35} \pm 0.11$ & $\bm{29.53} \pm 0.08$ & $\bm{40.31} \pm 0.36$ \\    
    \bottomrule
\end{tabular}
}
\end{center}
\caption{Bound (log value) comparisons on real-world datasets.} 
\label{table:exp_real_dataset}
\end{table}

\begin{table}[t]
\begin{center}
\resizebox{\textwidth}{!}
{
\begin{tabular}{c|cccccc}
    \hline
    \toprule
    $l=2$ & ER-1 & ER-2 & ER-3 & ER-4 & SBM-1 & SBM-2 \\
    \midrule
    \midrule
    Rademacher & $17.37 \pm 0.16$ & $17.98 \pm 0.13$ & $18.15 \pm 0.15$	& $18.35 \pm 0.10$ & $17.88 \pm 0.11$ & $17.71 \pm 0.09$ \\
    PAC-Bayes & $\bm{15.38} \pm 0.12$ & $\bm{15.13} \pm 0.13$ & $\bm{14.86} \pm 0.25$ & $\bm{14.69} \pm 0.24$ & $\bm{15.23} \pm 0.12$ & $\bm{15.35} \pm 0.10$ \\
    \midrule
    \midrule
    $l=4$ &  &  &  &  &  &  \\
    \midrule
    \midrule
    Rademacher & $27.92 \pm 0.02$ & $29.57 \pm 0.12$ & $30.64 \pm 0.18$ & $31.34 \pm 0.20$ & $29.35 \pm 0.14$ & $28.87 \pm 0.07$ \\
    PAC-Bayes & $\bm{27.00} \pm 0.04$ & $\bm{28.32} \pm 0.07$ & $\bm{29.18} \pm 0.12$ & $\bm{29.70} \pm 0.14$ & $\bm{28.14} \pm 0.05$ & $\bm{27.74} \pm 0.04$ \\
    \midrule
    \midrule
    $l=6$ &  &  &  &  &  &  \\    
    \midrule
    \midrule
    Rademacher & $37.10 \pm 0.29$ & $40.22 \pm 0.19$ & $42.00 \pm 0.26$ & $43.08 \pm 0.39$ & $40.04 \pm 0.25$ & $39.02 \pm 0.19$ \\
    PAC-Bayes & $\bm{36.85} \pm 0.25$ & $\bm{39.65} \pm 0.14$ & $\bm{41.30} \pm 0.22$ & $\bm{42.24} \pm 0.34$ & $\bm{39.50} \pm 0.17$ & $\bm{38.63} \pm 0.17$ \\
    \midrule
    \midrule    
    $l=8$ &  &  &  &  &  &  \\
    \midrule
    \midrule
    Rademacher & $\bm{46.72} \pm 0.51$ & $51.16 \pm 0.21$ & $53.44 \pm 0.39$ & $55.06 \pm 0.38$ & $50.60 \pm 0.17$ & $49.29 \pm 0.34$ \\
    PAC-Bayes & ${46.79} \pm 0.48$ & $\bm{51.02} \pm 0.21$ & $\bm{53.10} \pm 0.36$ & $\bm{54.67} \pm 0.38$ & $\bm{50.44} \pm 0.16$ & $\bm{49.22} \pm 0.36$ \\    
    \bottomrule
\end{tabular}
}
\end{center}
\caption{Bound (log value) comparisons on synthetic datasets.} 
\label{table:exp_syn_dataset}
\end{table}

In addition to the figures shown in the main paper, we also provide the numerical values of the bound evaluations in Table \ref{table:exp_real_dataset} (real-world datasets) and Table \ref{table:exp_syn_dataset} (synthetic datasets).
As you can see, our bound is tighter than the Rademacher complexity based one under all settings except for one synthetic setting which falls in the scenario ``small $d$ (max-node-degree) and large $l$ (number-of-steps)''.
This makes sense since we have a square term on the number of steps $l$ and it will play a role when the term involved with $d$ is comparable (\ie, when $d$ is small).
Again, all quantities are in the log domain.